\documentclass{article}

    \PassOptionsToPackage{numbers, compress}{natbib}

\usepackage[preprint]{neurips_2026}
\usepackage[utf8]{inputenc} 
\usepackage[T1]{fontenc}    
\usepackage{hyperref}       
\usepackage{url}            
\usepackage{booktabs}       
\usepackage{amsfonts}       
\usepackage{nicefrac}       
\usepackage{microtype}      
\usepackage{xcolor}         

\usepackage{amsmath}
\usepackage{amssymb}
\usepackage{mathtools}
\usepackage{amsthm}

\usepackage{amsmath,amsfonts,bm}

\def\eqref#1{equation~\ref{#1}}

\def\1{\bm{1}}

\def\rvc{{\mathbf{c}}}

\def\rvg{{\mathbf{g}}}

\def\rvm{{\mathbf{m}}}

\def\rvv{{\mathbf{v}}}
\def\rvw{{\mathbf{w}}}
\def\rvx{{\mathbf{x}}}

\def\mI{{\bm{I}}}

\DeclareMathAlphabet{\mathsfit}{\encodingdefault}{\sfdefault}{m}{sl}
\SetMathAlphabet{\mathsfit}{bold}{\encodingdefault}{\sfdefault}{bx}{n}

\def\gB{{\mathcal{B}}}

\def\gD{{\mathcal{D}}}

\def\gL{{\mathcal{L}}}
\def\gM{{\mathcal{M}}}
\def\gN{{\mathcal{N}}}
\def\gO{{\mathcal{O}}}

\def\gT{{\mathcal{T}}}

\def\gW{{\mathcal{W}}}

\def\sR{{\mathbb{R}}}

\newcommand{\meanp}[2]{\mathbb{E}_{#1} \left\lbrack #2 \right\rbrack}

\newcommand{\cov}[2]{\mathrm{cov}_{#1} \left\lbrack #2 \right\rbrack}

\newcommand{\norm}[1]{\lVert#1\rVert}

\usepackage{subcaption}
\usepackage{algorithm}
\usepackage{algpseudocode}

\def\method{IMCD}
\def\sg{\mathrm{sg}}

\theoremstyle{plain}
\newtheorem{theorem}{Theorem}[section]
\newtheorem{proposition}[theorem]{Proposition}

\newtheorem{corollary}[theorem]{Corollary}
\theoremstyle{definition}
\newtheorem{definition}[theorem]{Definition}

\theoremstyle{remark}
\newtheorem{remark}[theorem]{Remark}
\usepackage[capitalize,noabbrev]{cleveref}
\crefname{corollary}{corollary}{corollaries}
\Crefname{corollary}{Corollary}{Corollaries}
\crefname{definition}{definition}{definitions}
\Crefname{definition}{Definition}{Definitions}
\crefname{proposition}{proposition}{proposition}
\Crefname{proposition}{Proposition}{Proposition}

\title{Integration Matters: Rollout-Based Training for Constrained Diffusion Models}

\author{
Xiaoxuan Liang$^{1\,2}$\quad
Saeid Naderiparizi$^{1 \, 2}$ \quad
Berend Zwartsenberg$^{2}$ \quad
Frank Wood$^{1\,2\,3}$\\[0.5em]
$^1$University of British Columbia \quad
$^2$Inverted AI \quad
$^3$Alberta Machine Intelligence Institute\\
\texttt{liang51@cs.ubc.ca}
}

\begin{document}

\maketitle

\begin{abstract}
    Constrained generative models aim to produce samples that satisfy complex feasibility constraints while remaining faithful to the data distribution. Existing constrained generation methods typically enforce constraints either through training-time optimization or sampling-time correction. Training-time optimization approaches optimize on states induced by the training distribution, which can differ substantially from those encountered during sampling. Sampling-time correction methods instead modify the sampling process at inference, introducing distribution shift and requiring expensive tuning, particularly for few-step sampling. We propose a fine-tuning framework that incorporates constraint guidance obtained through online rollout into the training process, which aligns training with sampling by differentiating through the fixed noise schedule used to numerically integrate the denoising process. This exposes the model to violations that arise along the denoising trajectory and aligns diffusion learning with the sampling process. Experiments across multiple tasks show that our method improves constraint satisfaction while maintaining competitive sampling quality compared to prior methods.
\end{abstract}

\section{Introduction}
In many real-world applications, generative tasks are defined not only by the underlying data distribution but also by complex feasibility constraints that govern the validity of the sample space. In robotics systems, generated states should remain geometrically feasible and avoid collisions with the environment and other agents~\citep{schulman2014motion, carvalho2023motion}. In autonomous driving scenarios, predicted trajectories should remain on the road~\citep{yang2024diffusion, zheng2025diffusion}. Diffusion models~\citep{ho2020denoising, songscore, karras2022elucidating} provide a powerful framework for high-quality generation~\citep{rombach2022high, ho2022video, harvey2022flexible}, but standard denoising objectives do not explicitly enforce feasibility constraints. As a result, generated samples may violate constraints that are critical for deployment.

Specifically, standard diffusion models are trained with denoising objectives evaluated at independently sampled noise levels, without directly enforcing constraints on the final sample. Existing constrained diffusion models address this gap through constraint-aware fine-tuning~\citep{bastek2024physics, naderiparizi2025constrained, liang2026improved} or post-training guidance~\citep{dhariwal2021diffusion, chung2022diffusion, hemanifold, yu2023freedom, lou2023reflected, fishman2023diffusion, ye2024tfg, christopher2024projected, liang2025simultaneous}. Fine-tuning methods incorporate constraint guidance during training, but often optimize constraint losses on forward noising states, which can differ from the denoising states encountered during sampling. Sampling-time correction methods instead apply guidance only at inference, forcing corrections on states the model was not trained to handle~\citep{chung2024cfg++, yang2024guidance}.



We tackle constrained generation with a fine-tuning approach that aligns training with sampling. Instead of enforcing constraints only through sampling-time corrections, we optimize the model through online denoising rollouts that match the inference procedure. Constraint violations are measured on the terminal generated samples, and the resulting objective trains the model on the states it encounters during sampling. To preserve distributional fidelity, we retain the EDM denoising loss as a regularizer.

Our contributions are: 
\begin{itemize}
    \item We introduce a fine-tuning framework that optimizes constraint satisfaction through the reverse denoising trajectory.
    \item We provide formal analysis of the training-time alignment induced by our training objective and show that the terminal denoising state converges towards the feasible region.
    \item We validate our method on two constrained generation tasks: bouncing ball and traffic scene trajectory prediction experiments.
\end{itemize}

\section{Background}
\paragraph{Diffusion model generation}
Diffusion models~\citep{sohl2015deep, song2019generative, ho2020denoising, songscore} generate samples by gradually corrupting data with noise and then learning to reverse this process through denoising. Formally, the forward process perturbs data $\rvx_0\sim p_0$ into a sequence of increasingly noisy states $\rvx_t$, while the reverse process reconstructs clean samples from noise.
The forward process is defined by the Ito stochastic differential equation (SDE):
\begin{align}
    d\rvx_t = f(\rvx_t; t)\,dt + g(t)\,d\rvw_t,
\end{align}
where $f(\cdot; \cdot)$ and $g(\cdot)$ are drift and diffusion coefficients, and $\rvw_t$ is a standard Wiener process. The corresponding reverse-time SDE is given by~\citep{anderson1982reverse}:
\begin{align}
    d\rvx_t = [f(\rvx_t; t) - g^2(t)\nabla \log p_t(\rvx_t)]\,dt + g(t)\,d\bar\rvw_t.
\end{align}
where $\tilde\rvw_t$ denotes a Wiener process evolving backward in time. Since the score function $\nabla_{\rvx_t} \log p_t(\rvx_t)$ is intractable, diffusion models learn a neural approximation $s_\theta(\rvx_t; t)$ trained via denoising score matching (DSM) objective~\citep{vincent2011connection}, which minimizes
\begin{align}
    \gL_{\mathrm{DSM}}(\theta) = \meanp{t, \rvx_0, \rvx_t}{\lambda(t)\norm{s_\theta(\rvx_t; t) - \nabla_{\rvx_t} \log p_t(\rvx_t)}^2},
\end{align}
where $\rvx_t\sim p(\rvx_t\mid \rvx_0)$ is sampled from the forward process.

For the Gaussian perturbation process used in EDM~\citep{karras2022elucidating}, consider the diffusion-only SDE 
\begin{align}
    d\rvx_t = g(t)\,d\rvw_t.
\end{align} 
It has marginals $p(\rvx_t\mid \rvx_0)\sim \gN(\rvx_0, \sigma^2(t)\mI)$, $\sigma^2(t) = \int_0^t g^2(s)\,ds$, where $\sigma(t)$ denotes noise level at time $t$. Choosing $g(t) = \sqrt{2t}$ leads to $\rvx_t \sim \gN(\rvx_0, t^2\mI)$. Thus the noise level satisfies $\sigma(t) = t$, and we use $t$ to denote the noise scale throughout. This equivalence is standard~\citep{songscore, karras2022elucidating}. Through Tweedie's formula~\citep{efron2011tweedie}, score estimation can be recast as a denoising problem. Given a denoiser network $D_\theta(\rvx_t; t)$ that approximates $\meanp{}{\rvx_0\mid \rvx_t}$, the score can be estimated as
\begin{align}
    \nabla_{\rvx_t} \log p_t(\rvx_t) = \frac{\meanp{}{\rvx_0\mid \rvx_t} - \rvx_t}{t^2} \approx \frac{D_\theta(\rvx_t; t) - \rvx_t}{t^2}.
\end{align}
\paragraph{Constrained diffusion generation}
In many applications, generated samples are required to satisfy problem-specific constraints that define a feasible set $\Omega\subseteq \sR^d$. These constraints often do not admit an explicit analytic boundary and are only implicitly specified through task-specified functions. Let $\ell^\Omega(\rvx)$ denote a constraint violation function such that, for a sample $\rvx\in\sR^d$, $\ell^\Omega(\rvx) = 0$ if $\rvx\in\Omega$ and $\ell^\Omega(\rvx) > 0$ otherwise. We assume that the data distribution is supported on the feasible set, $\mathrm{supp} (p_0)\subseteq\Omega$. Constrained diffusion generation seeks to produce samples through a reverse process that remains faithful to $p_0$ while satisfying the feasibility condition of lying in $\Omega$.

However, learning such constraints from data alone is challenging, as training data only consists of feasible examples and lacks explicit information about infeasible regions~\citep{regenwetter2024constraining}. Prior work~\citep{naderiparizidon} further shows that increasing dataset size alone does not eliminate constraint violations, leaving the model's behavior outside the data support and instead determined by inductive biases. Consequently, maximum-likelihood training alone cannot achieve high constraint satisfaction, motivating the need for additional signals that explicitly penalize infeasible generations. 

\section{Limitations of existing approaches to constrained diffusion}
\subsection{Training-time misalignment in constraint fine-tuning}
\label{sec: training-sampling mismatch}
Prior constrained diffusion methods~\citep{bastek2024physics, liang2026improved} derive their training signals from constraint loss evaluated at $D_\theta(\rvx_t; t)$, where $\rvx_t\sim q(\rvx_t\mid \rvx_0)$. At low noise levels, $\rvx_t$ remains close to the data manifold and is typically near-feasible. Since the pretrained denoiser trained on feasible data tends to recover specific training samples~\citep{biroli2024dynamical}, $D_\theta(\rvx_t; t)$ also tends to satisfy the constraints. This makes $\ell^\Omega(D_\theta(\rvx_t; t)) \approx 0$ and causes the correction signal to vanish. In contrast, sampling starts from Gaussian noise and follows discrete denoising steps. High-noise samples often introduce violations under uninformative guidance. These violations can persist until low noise, where correction becomes reliable. However, low-noise infeasible states produced by reverse rollouts are rarely encountered during training. This leads to a misalignment: the model is trained mostly on near-feasible states, yet must correct infeasible rollout states at inference.

This is illustrated in~\cref{fig: infraction-loss}. We evaluate the constraint loss of the one-step lookahead $D_\theta(\rvx_t; t)$ at matched noise levels, where $\rvx_t$ is obtained either from the forward process or reverse sampling using the EDM model later evaluated in~\cref{sec: bouncing-ball}. While both exhibit similar behavior at high noise levels, they diverge significantly at low noise. Forward states remain near-feasible and yield near-zero constraint loss, whereas reverse states incur persistent violations. This indicates that training on forward states underestimates the violations encountered during sampling.


Prior work has studied the discrepancy between forward sampling states $\rvx_t$ and reverse sampling states $\tilde\rvx_t$~\citep{ning2023input, li2023alleviating}. We examine its effect by comparing constraint losses on the denoised lookahead $D_\theta(\rvx_t; t)$ and $D_\theta(\tilde\rvx_t; t)$. This shows that constraint signals remain weak on forward sampled states, but become nontrivial along the reverse denoising trajectory.

\begin{figure}[t]
    \centering
    \begin{minipage}{0.52\textwidth}
        \centering
        \includegraphics[width=\linewidth]{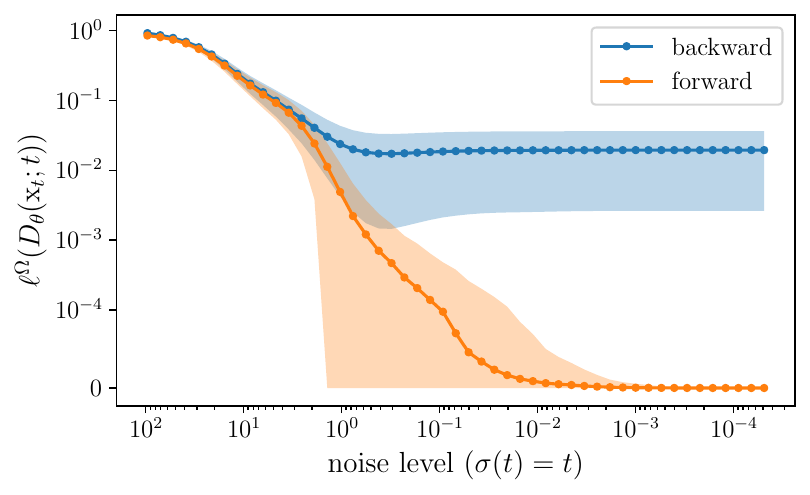}
        \captionof{figure}{Infraction loss of the one-step lookahead prediction as a function of intermediate state $\rvx_t$ and noise level $t$, comparing forward process samples and reverse-time sampling states. Shaded regions denote $\pm 0.2$ standard deviations.}
        \label{fig: infraction-loss}
    \end{minipage}
    \hfill
    \begin{minipage}{0.44\textwidth}
        \centering
          \begin{subfigure}[b]{\textwidth}
            \centering
            \includegraphics[width=\textwidth]{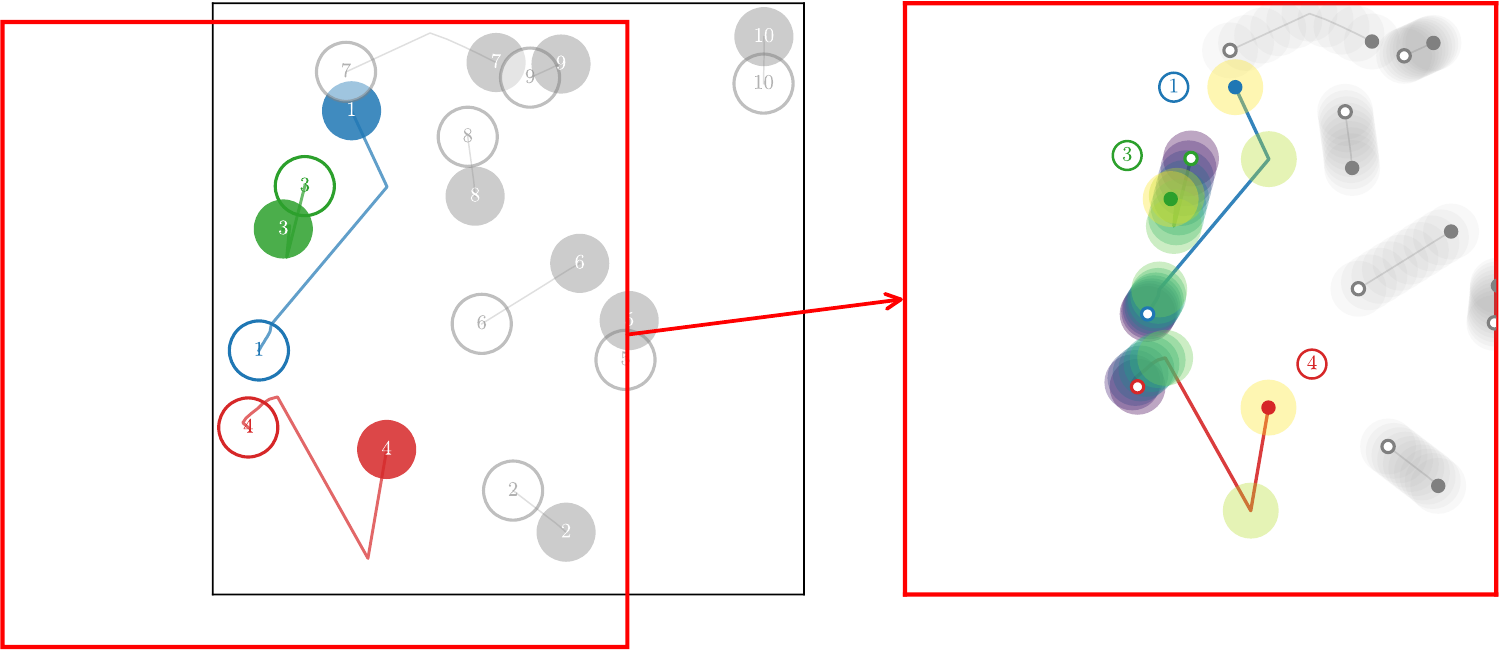}
            \caption{Generated sample}
            \label{fig:poly-inf-rate}
          \end{subfigure}
          \hfill
          \begin{subfigure}[b]{\textwidth}
            \centering
            \includegraphics[width=\textwidth]{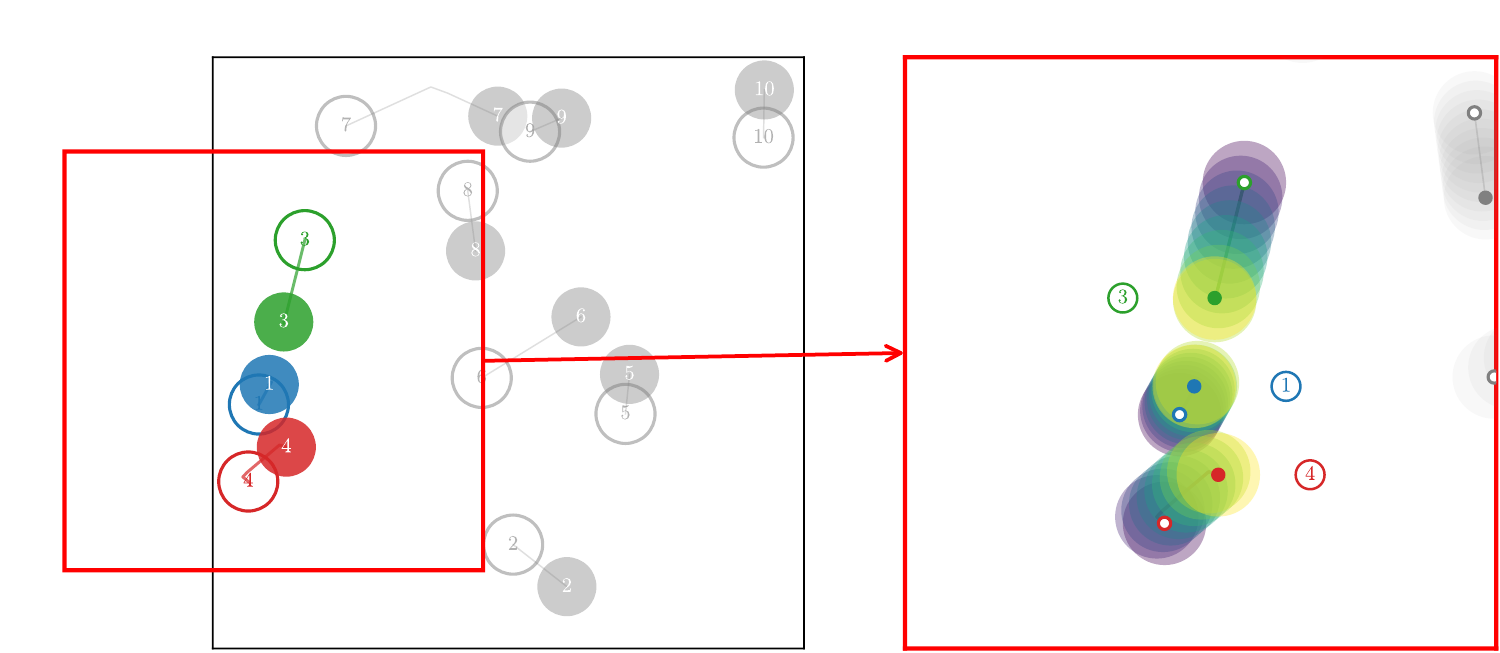}
            \caption{Ground truth sample}
            \label{fig:poly-mse}
          \end{subfigure}
        \captionof{figure}{Generated versus ground-truth bouncing ball trajectories.}
        \label{fig: bouncing-balls}
    \end{minipage}
\end{figure}

\subsection{Sampling-time correction: distribution shift and few-step degradation}
\label{sec: limitation-of-tf}
Another common approach enforces constraints only at sampling time, using either projection~\citep{christopher2024projected, liang2025simultaneous} or gradient-based updates~\citep{hemanifold}
\begin{align}
    d\rvx_t = \left[f(\rvx_t; t) - g^2(t) (s_\theta(\rvx_t; t) -\gamma(t) \nabla \ell^\Omega(\rvx_t; t))\right]\,dt + g(t)\,d\bar\rvw_t,
\end{align}
where $\gamma(t)$ controls the guidance strength. These methods correct violations online during denoising and can reduce the infeasibility rate in practice, but they are highly sensitive to the guidance schedule~\citep{chung2022diffusion}. Weak guidance may fail to remove violations, while strong guidance can push intermediate state $\rvx_t$ away from the data manifold~\citep{dhariwal2021diffusion}, which places samples in regions where the learned score is unreliable. This further widens the misalignment between forward sampled and reverse sampling states. Moreover, because sampling-time correction relies on incremental corrections across the denoising trajectory, it becomes sensitive to sampling budget: many steps allow small corrections to accumulate, while few steps require stronger updates that can distort or destabilize sampling~\citep{shen2024understanding}.


\cref{fig: bouncing-balls} illustrates this behavior using trajectories generated by a sampling-time correction method called MPGD without projection~\citep{hemanifold} under 50-step sampling setting later used in~\cref{sec: bouncing-ball}. The full-scene and zoomed views use hollow and solid circles denoting start and end positions. With gray balls shown for context, balls 1, 3, and 4 are highlighted because of their jittering motion, which produces sharp trajectory changes that avoid collision but distort the motion. In contrast, the ground truth rollout remains smooth and physically consistent. Thus, this method enforces feasibility without preserving realistic motion.

These observations reveal two limitations. Training-time optimization suffers from misalignment between the states encountered from training and sampling distributions. Sampling-time correction introduces distribution shift and becomes fragile under few-step sampling. Together, these limitations motivate training on reverse-time rollout states and learning constraint-aware updates that enforce feasibility.

\section{Methodology}
We present~\method{} (Integration Matters for Constrained Diffusion), a rollout-based fine-tuning framework for constrained generative modeling that addresses the training-time misalignment. \method{} evaluates constraint losses at terminal rollout samples and backpropagates through the full denoising trajectory for model updates, aligning optimization with inference. During each rollout step, it computes constraint-aware corrections using denoised lookahead estimates and applies them through a learnable scaling mechanism, enabling adaptive guidance across noise levels. To preserve data fidelity, \method{} retains the EDM loss and updates only lightweight trainable components.


We assume access to a pretrained denoiser denoted by $D$, where the pretrained model is parameterized by fixed model weights $\theta$. Since $\theta$ remains frozen throughout fine-tuning, we omit it from the notation. The fine-tuned denoiser network is initialized from $\theta$ and denoted by $D_\phi$, where $\phi$ collects the lightweight trainable components introduced during fine-tuning. The score functions induced by these denoisers are denoted by $s(\rvx_t; t)$ and $s_\phi(\rvx_t; t)$, respectively.


\def\guidance{\ensuremath{G^\Omega}}
\def\lossgrad{\ensuremath{\partial_1\ell^\Omega}}
\subsection{Constraint Guidance Mechanism}
A key design choice in constrained diffusion is how to construct a constraint-aware score estimator. Generalizing the definitions in \citep{naderiparizi2025constrained}, we consider estimators of the form
\begin{align}
\label{eq: trainable-score}
    s_\phi^\Omega(\rvx_t; t) = s_\phi(\rvx_t; t) - \gamma(\rvx_t; t) \guidance(\rvx_t; t),
\end{align}
where $\guidance$ produces a direction that reduces violations of the constraint $\Omega$ and $\gamma$ controls its strength. We call $\guidance$ the guidance function and $\gamma$ the scaling function. \cref{tab: scaling-table} summarizes the definition of these functions in our method and in prior work.

\paragraph{Guidance function} We follow \citet{liang2026improved} for the choice of guidance function $\guidance(\rvx_t; t) := \lossgrad(D(\rvx_t; t); t)$ where $\partial_1$ denotes the partial derivative with respect to the first argument. Effectively, this computes the gradient of the constraint loss function at the one-step denoising prediction by a pretrained denoiser $D$. \citet{liang2026improved} show that this leads to a more stable model (both during training and sampling) than computing the gradients at the noisy state $\rvx_t$.

It is important to note that $\lossgrad(D(\rvx_t; t); t)$ differs from $\nabla_{\rvx_t}\ell^\Omega(D(\rvx_t; t); t)$. The latter corresponds more directly to the diffusion formulation, since it fully accounts for the dependence of the denoised prediction on the noisy latent $\rvx_t$ by backpropagating through the denoiser $D$. In contrast, $\lossgrad(D(\rvx_t; t); t)$ ignores this dependency and does not backpropagate gradients through $D$, making it significantly more computationally efficient in practice.

\paragraph{Scaling function} 
The scaling $\gamma$ used in prior work~\citep{naderiparizi2022amortized,liang2026improved} is independent of the sample $\rvx_t$ and is typically selected via grid search. In practice, this tuning procedure can be computationally expensive, as the optimal value of $\gamma$ depends on the sampling regime, noise level, and constraint strength.
We instead propose to learn a scaling function that adapts the correction dynamically to both the current state $\rvx_t$ and diffusion time $t$. This eliminates the need for costly hyperparameter sweeps. This parameterization is also motivated by the exact gradient $\nabla_{\rvx_t}\ell^\Omega(D(\rvx_t; t); t)$. In particular, the discrepancy between $\lossgrad(D(\rvx_t; t); t)$ and the exact gradient induced by backpropagation through the denoiser can be interpreted as a state- and time-dependent rescaling of the guidance direction. More precisely,
\begin{align}
    \nabla_{\rvx_t} \ell^\Omega(D(\rvx_t; t)) = \left[\frac{\partial D(\rvx_t; t)}{\partial \rvx_t}\right]^\top \lossgrad(D(\rvx_t; t); t).
\end{align}
Prior work~\citep{efron2011tweedie, boys2023tweedie} shows that the Jacobian scales with a posterior covariance term of the form $\partial D(\rvx_t; t)/\partial \rvx_t \approx \cov{}{\rvx_0 \mid \rvx_t}/t^2$. The covariance term captures state-dependent uncertainty, while the denominator amplifies corrections as the sample approaches the data manifold.

\begin{definition}
    \label[definition]{definition: learned-scaling}
    (Learned scaling network) Given a trainable score $s_\phi$, a constraint violation function $\ell^\Omega(\rvx_t; t)$, and a learned scaling network $\gamma_\phi: \sR^d\times [0, T]\to \sR^+$, we define:
    \begin{align}
        \label{eq: scaling-network}
        \gamma_\phi(\rvx_t; t) = \alpha_\phi(\rvx_t; t)\cdot t^{\beta_\phi(\rvx_t; t)},
    \end{align}
    where $\alpha_\phi(\rvx_t; t) \geq 0$, $\beta_\phi < -1$. The constrained score is then
    \begin{align}
    \label{eq: constrained-score}
    s_\phi^\Omega = s_\phi(\rvx_t; t) - \gamma_\phi(\rvx_t; t)\lossgrad(D(\rvx_t; t); t),
    \end{align}
\end{definition}

This parameterization keeps the simple power-law guidance form but makes it adaptive to both the current state $\rvx_t$ and noise level $t$.

\begin{proposition}
\label[proposition]{proposition: convergence}
    (Convergence property) The constrained score~\ref{eq: constrained-score} corresponds to the score of 
    \begin{align}
        p^\Omega(\rvx_t; t) \propto p(\rvx_t; t) \exp(-\gamma_\phi(\rvx_t; t)\ell^\Omega(D(\rvx_t; t))).
    \end{align}
    Furthermore, $\gamma_\phi(\rvx_t; t)\to\infty$ as $t\to 0$ for any $\rvx_t$ with $\alpha_\phi(\rvx_t; t) > 0$, such that
    \begin{align}
        \lim_{t\to 0} p^\Omega(\rvx_t; t) \propto \lim_{t\to 0} p(\rvx_t; t) \1_\Omega(\rvx_t),
    \end{align}
    preserving stronger corrections near the data manifold while allowing the magnitude to adapt to the current state and noise level. 
\end{proposition}

The discussion above considers a single constraint gradient. For tasks with multiple constraints, we combine their gradients into a single compatible update direction inspired by PCGrad~\citep{yu2020gradient}, then apply the scaling network to this direction to learn a global correction magnitude. We provide the combination procedure and scaling-network architecture in Appendix~\cref{app: scaling-network}.

\color{black}

\subsection{Rollout-based training}
Now we address the training-time misalignment discussed in~\cref{sec: training-sampling mismatch}. We construct training states by rolling out the reverse-time process from pure Gaussian noise. The training-time rollout shares the same noise sampling schedule as inference such that the intermediate states are generated by the fine-tuned network $D_\phi$.

We use $(\rvx_t; t)$ for continuous-time notation and adopt step-indexed notation $(\rvx_i, t_i)$ for the discretized reverse-time rollout. Let the noise schedule be $T = t_N > t_{N-1} > \dots > t_0 = t_{\min} > 0$. Starting from $\tilde\rvx_N\sim \gN(0, t_N^2 \mI)$, we simulate denoising steps backward along this schedule and denote the resulting reverse rollout trajectory by $\{\tilde\rvx_i\}_{i=1}^N$. At each rollout step $i$, we compute the denoised state guidance $\lossgrad(D(\tilde\rvx_i; t_i); t_i)$ using the pretrained denoiser. The fine-tuned denoiser $D_\phi$ is conditioned on this guidance through a lightweight embedding, and the scaling network $\gamma_\phi(\tilde\rvx_i; t_i)$ defined in~\cref{eq: scaling-network} determines its strength at the current rollout state. We also introduce LoRA adapters~\citep{hu2022lora} within $D_\phi$ to provide adaptation capacity while keeping the pretrained backbone weights fixed. This gives the constraint-aware denoiser:
\begin{align}
\label{eq: constraint-aware fine-tune denoiser}
    D_\phi^\Omega(\tilde\rvx_i; t_i) = D_\phi\left(\tilde\rvx_i; t_i, \lossgrad(D(\tilde\rvx_i; t_i); t_i)\right) + \gamma_\phi(\tilde\rvx_i; t_i) \lossgrad(D(\tilde\rvx_i; t_i); t_i).
\end{align}


We use this constraint-aware denoiser $D_\phi^\Omega$ at each rollout step. Thus the model receives the constraint signal at its input while also applying an output-level correction with a learned, state-dependent strength. Training on the resulting rollout states $\{\tilde\rvx_i\}$ under the same discretization schedule as inference exposes the model to off-manifold states produced by its own sampling process. This alignment directly targets the training-time misalignment identified in~\cref{sec: training-sampling mismatch}, enabling the model to correct constraint violations during generation.

\subsection{Training objective and sampling procedure}
We optimize the fine-tuning parameters $\phi$ with an EDM denoising term and a terminal rollout constraint term, where $\kappa$ balances constraint optimization against data fidelity:
\begin{align}
\label{eq: training-objective}
    \gL(\phi) = \gL_{\mathrm{EDM}}(\phi) + \kappa\gL_{\mathrm{rollout}}(\phi),
\end{align}

\paragraph{EDM denoising loss} To preserve data-consistent behavior, we retain the EDM denoising objective for the constraint-augmented denoiser $D_{\phi}^\Omega$:
\begin{align}
    \gL_{\mathrm{EDM}}(\phi) = \meanp{t, \rvx_0, \rvx_t}{\lambda(t)\norm{D_{\phi}^\Omega(\rvx_t; t) - \rvx_0}^2},
\end{align}
where $\rvx_t\sim q(\rvx_t\mid \rvx_0)$ is sampled from the forward process independently of the rollout, and $\lambda(t)$ is the EDM loss weighting. This objective regularizes the model toward data-consistent denoising while allowing deviations required to satisfy the constraints.

\paragraph{Rollout constraint loss} For constraint optimization, constraint loss is applied at the terminal rollout:
\begin{align}
\label{eq: unscaled-rollout-loss}
    \gL_{\mathrm{rollout}}(\phi) = \meanp{p(\tilde\rvx_T), p_{\phi}(\tilde\rvx_0\mid \tilde\rvx_T)}{\ell^\Omega(\tilde\rvx_0)}.
\end{align}
During optimization, the sampled states $\{\tilde\rvx_i\}$ are not updated by gradients, and optimization is only performed with respect to $\phi$ through the fine-tuned denoiser evaluations. For efficiency, we adopt the EDM discretization~\citep{karras2022elucidating}, which achieves high sampling quality with fewer denoising steps. In addition, we employ gradient checkpointing~\citep{chen2016training} to control memory when propagating through rollout trajectories.

Since supervision is applied only at the terminal state, the constraint signal is sparse and becomes sparser as the infraction rate decreases. We therefore adaptively weight the rollout loss:
\begin{align}
    \kappa =  \sg(\overline{\gL_{EDM}}) / \sg(\overline{\gL_{\mathrm{rollout}}} + \epsilon),\quad \epsilon=10^{-5}
\end{align}
where the overline denotes the batch mean, $\sg(\cdot)$ stops gradients through the scaling factors. The adaptive stop-gradient weight keeps the rollout loss on the same scale as the EDM loss, maintaining effective constraint supervision without introducing additional gradient pathways.

We summarize the full training procedure in~\cref{alg: training}. At inference time, sampling follows the same reverse-time rollout, with the same noise discretization schedule used during training. The rollout starts from a Gaussian initialization at the highest noise level and iteratively denoises to obtain the final sample. The constraint-aware denoiser $D_\phi^\Omega$ is used for sampling, while the pretrained denoiser $D$ provides the denoised estimate for constraint gradient computation. Terminal loss evaluation and parameter updates are omitted. 

\subsection{Theoretical analysis}
We formalize the training-time alignment by viewing sampling as a composition of differentiable denoising maps, showing that rollout backpropagation optimizes constraint violations under the same sampling distribution at inference.
\begin{theorem} (Training-time alignment)
\label{theom: training-sampling-alignment}
    Let $\{t_i\}_{i=0}^N$ be the fixed discretization schedule. Define each denoising step as a deterministic function:
    \begin{align}
        \rvx_{i-1} = f_\phi^{(i)}(\rvx_i; t_i, t_{i-1}, \epsilon_i), \quad \epsilon_i\sim \gN(0, \mI),
    \end{align}
    and $f_\phi$ is differentiable in $\rvx_i$ and $\phi$. For an initial noise sample $\rvx_N$, the terminal sample is obtained by composing the update maps: 
    \begin{align}
        F_\phi = f_\phi^{(1)}\circ f_\phi^{(2)} \circ \cdots \circ f_\phi^{(N)}, \quad s.t. \quad\rvx_0^\phi = F_\phi(\rvx_N)\ \text{is the terminal state}.
    \end{align}
    Then the gradient of unscaled rollout objective~\ref{eq: unscaled-rollout-loss} $\gL_{\mathrm{rollout}}(\phi)$  with respect to $\phi$ is:
    \begin{align}
         \nabla_\phi \gL_\mathrm{rollout}(\phi) = \meanp{p(\rvx_N)}{\sum_{i=1}^N \nabla_{\rvx_0^\phi} \ell^\Omega(\rvx_0^\phi) \frac{\partial \rvx_0^\phi}{\partial \rvx_{i-1}^\phi} \frac{\partial f_\phi^{(i)} (\rvx_i)}{\partial \phi}}.
    \end{align}
    This concludes that backpropagating $\ell^\Omega$ through the rollout computes the exact gradient of $\gL_{\mathrm{rollout}}(\phi)$, directly minimizing constraint violations under the model's own sampling distribution.
\end{theorem}


\begin{remark}
    (Full training objective) The full objective in~\cref{eq: training-objective} augments the rollout loss with $\gL_{\mathrm{EDM}}(\phi)$. Since $\gL_{\mathrm{EDM}}(\phi)$ is evaluated on clean data samples independently of the rollout trajectory, its gradient is simply an additive term outside the rollout computation graph. Thus,~\cref{theom: training-sampling-alignment} still applies to $\gL(\phi)$, with $\nabla_\phi \gL_{\mathrm{EDM}}(\phi)$ added independently.
\end{remark}

\section{Experiments}
We evaluate whether~\method{} improves constraint enforcement while preserving generative fidelity under the same experimental conditions as prior work. To enable direct comparison, we adopt the experimental setup of~\citep{liang2026improved} and focus on isolating the effect of~\method{}. Specifically, we assess whether the proposed design reduces infeasible generation through rollout-based training and improves stability across noise levels.

\subsection{Bouncing balls}
\label{sec: bouncing-ball}
We begin with a synthetic bouncing ball experiment, where multiple balls move inside a closed box with elastic ball-ball and ball-wall collisions. This setting enables controlled evaluation of both physical consistency and constraint satisfaction. Following~\citep{liang2026improved}, we generate simulated trajectories with randomized initial positions and velocities using a physics-based simulator~\citep{gan2015deep}. The dataset contains 100,000 scenarios, each with 10 balls over 100 timesteps. The goal is to model the trajectory distribution while enforcing two constraints: (i) boundary constraints, keeping all balls inside the box, (ii) overlap constraints, preventing overlap between balls.

We compare against a diverse set of baselines spanning unconstrained, sampling-time correction and fine-tuning approaches. EDM~\citep{karras2022elucidating} serves as the unconstrained pretrained diffusion model. MBM~\citep{naderiparizi2025constrained} and MBM++~\citep{liang2026improved} incorporate constraint gradients into the score function with scaling and guidance functions in~\cref{tab: scaling-table} to enforce constraint satisfaction at the clean sample. MPGD~\citep{hemanifold} is a sampling-time correction method that we adapt with the same diverging scaling schedule as MBM++, omitting decoder projection. PIDM~\citep{bastek2024physics} adds denoised-state constraint losses to the DSM objective. Since the original formulation does not specify whether pretraining is used, we evaluate both training-from-scratch and fine-tuning variants, denoted PIDM and PIDM (FT). We also include rollout-based fine-tuning baselines, Adjoint Matching (AM)~\citep{domingo2024adjoint} and DPOK~\citep{fan2023dpok}, which optimize rollout-level objectives while regularizing deviation from the unconstrained pretrained model. All methods use the same sampling configuration with 50 reverse-time steps.

We evaluate constraint satisfaction using overlap and boundary violation rates, and assess distributional fidelity using reweighted ELBO (r-ELBO). To capture physical motion consistency, we introduce three additional metrics. We measure maximum frame-to-frame displacement (F2F) to quantify temporal inconsistency and detect jittering artifacts, maximum contact distance (MCD) to capture worst-case proximity during ball-ball and ball-wall interactions, and maximum energy deviation (MED) to measure violations of kinetic energy conservation over time. We also compute these metrics on ground-truth trajectories for reference. Full definitions of these metrics are provided in the appendix~\cref{app: def-metrics}.


\cref{tab: bb-table} reports results on the bouncing ball experiment. Baselines reveal a trade-off between feasibility, physical realism, and fidelity. MPGD enforces constraints perfectly, but its high F2F and MED indicate abrupt frame-to-frame corrections and jittery motion, as visualized in~\cref{fig: bouncing-balls}. MBM and MBM++ reduce violations but remain unstable or lose physical plausibility under few-step sampling. PIDM and PIDM (FT) preserve plausible dynamics but fail to enforce hard constraints, while DPOK and AM either degrade fidelity or only partially satisfy constraints. Our method avoids these trade-offs: it achieves near-zero boundary and overlap violations while keeping F2F and MED close to the EDM baseline, indicating smooth constraint-aware rollouts rather than abrupt corrections. Its competitive fidelity shows that retaining EDM denoising preserves the training data distribution, yielding the strongest overall balance across metrics. To better understand these gains, we further provide ablation studies on the effect of rollout-based training and the trainable components in~\cref{app: bb-ablation}.

\begin{table}[!t]
    \centering
    \small
    \setlength{\tabcolsep}{4pt}
    \caption{Bouncing ball experiment results. We report constraint violation rates, distributional fidelity (r-ELBO), and physical consistency metrics (F2F, MCD, MED). Ground truth values are provided for reference.}
    \begin{tabular}{lcccccc}
    \toprule
     & \multicolumn{2}{c}{Constraint (\%) $\downarrow$} & Fidelity $\uparrow$ & \multicolumn{3}{c}{Physical plausibility ($\times 10^{-1}$) $\downarrow$} \\
    \cmidrule(lr){2-3} \cmidrule(lr){5-7}
    Method & Boundary rate & Overlap rate & r-ELBO ($\times 10^{-2}$) & F2F & MCD & MED \\
    \midrule
    EDM~\citep{karras2022elucidating} & $6.01\pm 0.02$  & $39.82\pm 0.21$  & $-21.7\pm 0.1$  & $3.1\pm 0.0$ & $1.0\pm 0.0$ & $0.4\pm 0.0$ \\
    MPGD~\citep{hemanifold} & $0.00\pm 0.00$  & $0.00\pm 0.00$ & $-27.0\pm 0.1$ & $8.6\pm 0.4$  & $0.8\pm 0.0$ & $29.7\pm 5.4$ \\
    MBM~\citep{naderiparizi2025constrained} & $0.36\pm 0.00$  & $0.16\pm 0.00$  & $-22.2\pm 0.1$  & $59.7\pm 0.4$  & $1.7\pm 0.0$  & $510.9\pm 11.5$  \\
    MBM++~\citep{liang2026improved} & $0.00\pm 0.00$  & $0.00\pm 0.00$ & $-22.8\pm 0.1$ & $10.3\pm 0.6$ & $0.8\pm 0.0$  & $41.1\pm 5.9$  \\
    PIDM~\citep{bastek2024physics} & $11.07\pm 0.15$  & $40.91\pm 0.19$  & $-26.9\pm 0.1$ & $2.9\pm 0.0$ & $1.1\pm 0.0$ &  $0.3\pm 0.0$\\
    PIDM (FT)~\citep{bastek2024physics} & $6.40\pm 0.04$ & $29.15\pm 0.10$  & $-21.6\pm 0.1$ & $3.0\pm 0.0$ & $0.8\pm 0.0$ &  $0.3\pm 0.0$\\
    AM~\citep{domingo2024adjoint} & $0.49\pm 0.01$ & $4.64\pm 0.06$ & $-39.4\pm 0.2$  & $3.4\pm 0.0$ & $3.8\pm 0.0$ & $0.5\pm 0.0$ \\
    DPOK~\citep{fan2023dpok} & $2.04\pm 0.05$ & $1.69\pm 0.03$ & $-53.5\pm 0.1$ & $2.8\pm 0.0$ & $4.3\pm 0.0$ & $0.3\pm 0.0$ \\
    \midrule
    \method{} (\textbf{Ours}) & $0.01\pm 0.00$  & $0.01\pm 0.00$  & $-22.9\pm 0.0$  & $3.4\pm 0.1$  & $1.2\pm 0.0$ & $0.6\pm 0.0$\\
    \midrule
    Ground truth & $0.00\pm 0.00$  & $0.00\pm 0.00$ & --  & $2.8\pm 0.0$  & $0.4\pm 0.0$ &  $0.3\pm 0.0$\\
    \bottomrule
    \end{tabular}
    \label{tab:placeholder}
\end{table}

\subsection{Traffic scene trajectory prediction}
We next evaluate real-world traffic scene trajectory prediction on the Interaction dataset~\citep{zhan2019interaction}, which contains vehicle trajectories from 11 traffic scenarios. Given a short history of observed states, the model predicts future positions and headings for all agents in the scene. This setting tests multi-agent driving behavior modeling under road geometry and vehicle interactions, including staying within drivable areas and avoiding collisions. Following prior work~\citep{lioutascritic, liang2026improved}, we observe one second of motion and predict a three-second future horizon. 

We compare against the same classes of baselines as in the bouncing ball experiment, excluding MBM and AM, and additionally include CriticSMC~\citep{lioutascritic}. Our bouncing ball results show that MBM underperforms MBM++, and we found that AM converges slowly and incurs high computational cost in this large-scale trajectory prediction setting. 
We use DJINN~\citep{niedoba2024diffusion}, an EDM-based diffusion model for joint trajectory prediction, as the pretrained backbone, fine-tuning it with the corresponding objective for baselines that require training and for our method. We report our main variant in the main text and defer additional rollout/LoRA variants to~\cref{app: method-variants}.


For evaluation, we generate six trajectory samples per scene and report metrics covering constraint satisfaction, trajectory match, physical plausibility and diversity. Constraint satisfaction is measured by offroad and collision rates. Trajectory match is evaluated using ego minADE$_6$/minFDE$_6$, where scene metrics minSADE$_6$/minSFDE$_6$ are computed over all vehicles. Physical plausibility is measured by maximum scene frame-to-frame displacement (SF2F) and maximum scene lateral velocity (SMLV), which capture abrupt motion and unrealistic lateral shift. Diversity is measured by maximum ego final distance (MFD$_6$) across six samples.


\cref{tab: scene-table} reports results on the INTERACTION DR\_DEU\_Merging\_MT scenario. Existing methods remain limited in different ways: EDM and PIDM (FT) better preserve diversity and physical plausibility but retain substantial violations, while MPGD and MBM++ reduce violations but still incur non-negligible offroad errors. In contrast, our~\method{} achieves the strongest feasibility, substantially reducing offroad violations and eliminating collisions while maintaining good trajectory match. Its physical plausibility metrics remain close to EDM, indicating that constraint satisfaction is achieved without severe motion distortion. Overall,~\method{} gives the best balance between feasibility and quality under few-step generation. Additional experiments and visualizations on different locations are reported in~\cref{app: dm-more-results}.

\begin{table}[!t]
    \centering
    \small
    \setlength{\tabcolsep}{4.5pt}
    \caption{
    Traffic scene trajectory prediction results on the INTERACTION DR\_DEU\_Merging\_MT scenario. 
    Ground truth values are provided for reference.
    }
    \label{tab: scene-table}
    \begin{tabular}{lccccccccc}
        \toprule
        & \multicolumn{2}{c}{Constraint (\%) $\downarrow$}
        & \multicolumn{4}{c}{Min traj. match $\downarrow$}
        & \multicolumn{2}{c}{Phys. plaus. $\downarrow$}
        & Diversity $\uparrow$ \\
        \cmidrule(lr){2-3}
        \cmidrule(lr){4-7}
        \cmidrule(lr){8-9}
        \cmidrule(lr){10-10}
        Method
        & Offroad & Collision & ADE$_6$ & FDE$_6$ & SADE$_6$ & SFDE$_6$ & SF2F$_6$ & SMLV$_6$ & MFD$_6$ \\
        \midrule
        EDM~\citep{karras2022elucidating} & $ 9.03$ & $0.42$ & $0.20$ & $0.50$ & $0.27$ & $0.77$ & $1.11$ & $0.41$ & $2.75$ \\
        MPGD~\citep{hemanifold} & $3.78$ & $0.29$ & $0.19$ & $0.46$ & $0.23$ & $0.67$ & $1.04$ & $0.34$ & $1.92$ \\
        MBM++~\citep{liang2026improved} & $3.73$ & $0.29$ & $0.17$& $0.45$ & $0.22$ & $0.66$ & $1.01$ & $0.31$ & $1.88$ \\
        PIDM (FT)~\citep{bastek2024physics} & $8.52$ & $0.34$ & $0.18$ & $0.50$ & $0.24$ & $0.76$ & $0.96$ & $0.26$ & $2.67$ \\
        DPOK~\citep{fan2023dpok} & $5.45$ & $0.03$ & $3.45$ & $9.94$ & $3.34$ & $9.73$ & $0.88$ & $0.40$  & $1.90$ \\
        CriticSMC~\citep{lioutascritic} & -- & $1.03$ & $0.35$ & -- & -- & -- & -- & -- & $2.37$ \\
        \midrule
        \method{} (\textbf{Ours}) & $0.30$ & $0.00$ & $0.18$ & $0.45$ & $0.23$ & $0.66$ & $1.07$ & $0.43$ & $1.91$ \\
        \midrule
        Ground truth & $0.00$ & $0.00$ & -- & -- & -- & -- & $0.83$ & $0.02$ & -- \\
        \bottomrule
    \end{tabular}
\end{table}

\section{Related Work}
\subsection{Diffusion models for constrained generation}
Constrained diffusion has been approached by modifying either the forward or reverse process.~\citet{liu2023learning} employ Doob's h-transform~\citep{rogers2000diffusions} to construct a drift that confines the diffusion process to the constraint domain at training time. Reflected diffusion model~\citep{lou2023reflected} and constrained diffusion framework proposed by~\citet{fishman2023diffusion} alter the forward noising process by reflected Brownian motion or log-barrier metrics to confine samples to feasible regions. Projected diffusion~\citep{christopher2024projected} instead enforces constraints at inference time by projecting noisy samples onto the constraint set at each denoising step, with extensions to multi-robot motion planning~\citep{liang2025simultaneous}. However, repeated projection of noisy intermediate states can distort the denoising trajectory. More importantly, these methods all require constraints to be analytically formulated upfront, which makes them less applicable when feasibility is observed through violations but not available in closed form, as in our setting.

\subsection{Reward-driven fine-tuning and controlled diffusion}
Fine-tuning diffusion models with reward or energy objectives is closely related to our training formulation. We defer a detailed comparison with AM and DPOK to~\cref{para: stoc-control}. A related line of work views diffusion sampling as stochastic optimal control, where methods such as AM, Adjoint sampling (AS)~\citep{havens2025adjoint} and Adjoint Schrodinger bridge Sampler (ASBS)~\citep{liu2025adjoint} learn reverse-time controls that drive samples toward reward- or energy-defined targets. ASBS further extends AS to arbitrary source distributions. However, both methods are designed to optimize reward or energy objective directly, rather than leverage a pretrained model as the base distribution. In contrast, our goal is constrained generation, which fine-tunes a pretrained denoiser through its sampling rollout to improve feasibility while preserving the learned data distribution.

\subsection{Diffusion models for autonomous driving}
Several recent works have explored diffusion models for autonomous driving planning. DJINN~\citep{niedoba2024diffusion} jointly diffuses trajectories of all agents for realistic scenario generation, but is primarily designed for simulation rather than planning, and does not focus on constraint satisfaction. Diffusion-ES~\citep{yang2024diffusion} combines gradient-free evolutionary search with a diffusion trajectory model to optimize non-differentiable test-time reward objectives, but operates on the ego vehicle only without modeling interactions with other agents. Diffusion Planner~\citep{zheng2025diffusion} jointly models ego planning and agent prediction with a unified diffusion architecture, and applies classifier guidance at inference time to achieve constraint satisfaction, but may encounter the sampling-time correction limitation discussed in~\cref{sec: limitation-of-tf}.

\section{Discussion and future work}
Our method~\method{} provides a rollout-based fine-tuning framework for constrained diffusion generation that addresses training-time misalignment. By minimizing constraint violations on terminal rollout samples, the model learns corrections under its own sampling dynamics rather than relying on post hoc guidance. Both theory and experiments show that our method improves constraint satisfaction while preserving distributional fidelity across bouncing ball simulation and autonomous driving tasks. While effective, the current inference procedure still involves constraint-gradient computation during the reverse rollout, adding modest computational overhead. A promising direction is to improve efficiency through distillation, where a standalone denoiser learns to approximate the constraint-aware updates directly. Another important direction is to extend rollout-based fine-tuning to higher-dimensional visual domains and broader application settings where feasibility is specified by task-specific violation signals.

\section*{Acknowledgment}
We acknowledge the support of the Natural Sciences and Engineering Research Council of Canada
(NSERC), the Alberta Machine Intelligence Institute (Amii) through the Canada CIFAR AI Chairs
Program, Inverted AI, Mitacs, and Google. This research was enabled in part by technical support
and computational resources provided by the Digital Research Alliance of Canada (alliancecan.ca),
the Advanced Research Computing at the University of British Columbia (arc.ubc.ca), and Amazon
Web Services.


\bibliography{refs}
\bibliographystyle{plainnat}

\newpage
\appendix

\section{Algorithm}
\begin{algorithm}
    \caption{Rollout-Based Constrained Fine-Tuning}
    \label{alg: training}
    \begin{algorithmic}[1]
    \Require Pretrained denoiser $D$, fine-tuned denoiser $D_\phi$ initialized from $D$, scaling network $\gamma_\phi$, constraint loss $\ell^\Omega$, noise schedule $\{t_i\}_{i=0}^N$, training data $\gD$, division stabilizer $\epsilon=10^{-5}$.
    \While{not converged}
        \State \textbf{// Rollout}
        \State Sample $\rvx_T \sim \gN(0, t_N^2\mI)$
        \For{$i = T,\ldots,1$}
            \State Compute constraint-aware fine-tuned denoiser $D_\phi^\Omega(\tilde\rvx_i; t_i)$ with~\cref{eq: constraint-aware fine-tune denoiser}.
            \State Update the rollout state to $\tilde\rvx_{i-1}$.
        \EndFor

        \State \textbf{// Losses}
        \State $\gL_{\mathrm{rollout}} \leftarrow \ell^\Omega(\tilde\rvx_0)$ \Comment{Terminal state only}
        \State Sample $\rvx_0 \sim \gD$, $\rvx_t \sim q(\rvx_t\mid \rvx_0)$
        \State $\gL_{\mathrm{EDM}} \leftarrow
        \lambda(t)\norm{D_{\phi}^\Omega(\rvx_t;t)-\rvx_0}^2$

        \State \textbf{// Dynamic Scaling and Update}
        \State $\kappa \leftarrow
        \frac{\sg(\overline{\gL_{\mathrm{EDM}}})}
        {\sg(\overline{\gL_{\mathrm{rollout}}})+\epsilon}$
        \State $\gL \leftarrow \gL_{\mathrm{EDM}} + \kappa \gL_{\mathrm{rollout}}$
        \State Update $\phi \leftarrow \phi - \eta\nabla_\phi \gL$ \Comment{Gradients flow only through $\phi$}
    \EndWhile
    \end{algorithmic}
\end{algorithm}

\paragraph{Rollout-based fine-tuning and stochastic optimal control} 
\label{para: stoc-control}
Our rollout-based training admits a natural stochastic optimal control (SOC) interpretation: the reverse denoising process induced by the pretrained denoiser $D(\rvx_t; t)$ and constraint-aware fine-tuned denoiser $D_{\phi}^\Omega$ correspond to the base and controlled stochastic process respectively, and the rollout constraint loss acts as a terminal cost. This connects our approach to Adjoint Matching (AM)~\citep{domingo2024adjoint}, which provides a formal SOC grounding for reward fine-tuning of generative models via rollout-based optimization. It is also related to DPOK~\citep{fan2023dpok}, which frames the denoising chain as a Markov decision process (MDP) and applies policy gradient with KL regularization. Both AM and DPOK approaches regularize the fine-tuned model by minimizing its KL divergence from the pretrained model, treating the pretrained model as the reference. AM additionally incurs significant computational overhead beyond rollout simulation, as each gradient update requires integrating the lean adjoint ODE backward through the trajectory. However, in our setting, the pretrained denoiser can produce samples that violate the constraints, so enforcing proximity to it can perpetuate these errors. Instead, we rely on initialization from the pretrained model, anchor to the data distribution directly via the denoising objective, and combine it with rollout-time constraint losses to drive the sampler toward the feasible manifold, preserving consistency with the data rather than with the pretrained model's potentially infeasible outputs.

\paragraph{Gradient embeddings}
Following MBM++, the constraint gradient is embedded with a lightweight MLP and added once to $\rvx_t$ as input to $D_{\phi}(\rvx_t; t)$. Since this modifies the model input distribution seen by the pretrained backbone, we introduce LoRA adapters~\citep{hu2022lora} in the attention layers of $D_{\phi}$ to provide the necessary adaptation capacity. 

\begin{table}[h]
\centering
\caption{A summary table of scaling and guidance functions for various methods.}
\begin{tabular}{c|ccc}
\toprule
 & MBM \citep{naderiparizi2025constrained} & MBM++ \citep{liang2026improved} & \method{} (ours)\\
\midrule
$\gamma(\rvx_t; t)$ & $c / t^2$ & $c / t^2$ & $\alpha_\phi(\rvx_t; t) \cdot t^{\beta_\phi(\rvx_t; t)}$\\
\midrule
$\guidance(\rvx_t; t)$ & $\lossgrad(\rvx_t; t)$  & $\lossgrad(D(\rvx_t; t); t)$ & $\lossgrad(D(\rvx_t; t); t)$ \\
\bottomrule
\end{tabular}
\label{tab: scaling-table}
\end{table}

\section{Derivation of proofs}
\subsection{Proof of~\Cref{proposition: convergence}}
\begin{proof}
    Given the formulation of $p^\Omega(\rvx_t; t)$, take gradient $\nabla_{\rvx_t} \log p^\Omega(\rvx_t; t)$ and treat $\gamma_t(\rvx_t; t)$ as constant with respect to $\rvx_t$. This yields the score decomposition in~\Cref{definition: learned-scaling}. For the convergence of denoising terminal state to the constrained region, since $\beta_\phi < -1$,  $t^{\beta_\phi} \to \infty$ as $t\to 0$. Given $\alpha_\phi > 0$, we have $\gamma_\phi(\rvx_t; t)\to \infty$ as $t\to 0$. Then
    \begin{align}
        \lim_{t\to 0} \exp(-\gamma_\phi(\rvx_t; t) \ell^\Omega(\rvx_t; t)) = \left\{
        \begin{array}{cc}
            1 & \mathrm{if}\ \ell^\Omega(\rvx_t; t) = 0 \\
            0 & \mathrm{if}\ \ell^\Omega(\rvx_t; t) > 0 
        \end{array}
        \right.
    \end{align}
    Therefore, 
    \begin{align}
        \lim_{t\to 0} p^\Omega(\rvx_t; t) \propto \lim_{t\to 0} p(\rvx_t; t) \1_\Omega(\rvx_t),
    \end{align}
    meaning that as $t\to 0$, the distribution concentrates on the constrained region $\Omega$, ensuring that the terminal denoised sample $\rvx_0\sim p(\rvx_0, 0)$ lies within $\Omega$.
\end{proof}
\subsection{Proof of~\cref{theom: training-sampling-alignment}}
\begin{proof}
    Since each $\epsilon_i$ is independent of $\phi$, the trajectory $\{\rvx_i\}_{i=0}^N$ is a deterministic differentiable function of $\phi$. Because each $f_\phi^{(i)}$ is differentiable, then the composed map $F_\phi$ is differentiable. By chain rule,
    \begin{align}
        \nabla_\phi \ell^\Omega(F_\phi(\rvx_N)) = \nabla_{\rvx_0^\phi} \ell^\Omega(\rvx_0^\phi) \frac{\partial F_\phi(\rvx_N)}{\partial \phi}.
    \end{align}
    Then the gradient of $\gL_{\mathrm{rollout}}(\phi)$ with respect to $\phi$ is
    \begin{align}
        \nabla_\phi \gL_{\mathrm{rollout}}(\phi) &= \nabla_\phi \meanp{p(\rvx_N)}{\ell^\Omega(F_\phi(\rvx_N))} = \meanp{p(\rvx_N)}{\nabla_\phi \ell^\Omega(F_\phi(\rvx_N))} \\
        &= \meanp{p(\rvx_N)}{\nabla_{\rvx_0^\phi} \ell^\Omega(\rvx_0^\phi) \frac{\partial F_\phi(\rvx_N)}{\partial \phi}}
    \end{align}
    by interchanging differentiation with expectation.

    Now we expand $\frac{\partial F_\phi(\rvx_N)}{\partial \phi}$ by applying the chain rule through the individual step $f_\phi^{(i)}$:
    \begin{align}
        \frac{\partial F_\phi(\rvx_N)}{\partial \phi} = \sum_{i=1}^N \frac{\partial \rvx_0^\phi}{\partial \rvx_i^\phi}\cdot \frac{\partial \rvx_i^\phi}{\partial \phi},
    \end{align}
    where $\frac{\partial \rvx_0^\phi}{\partial \rvx_i^\phi}$ is the Jacobian of the terminal state $\rvx_0^\phi$ with respect to the intermediate state $\rvx_i^\phi$. The second term $\frac{\partial \rvx_i^\phi}{\partial_\phi}$, where $\rvx_i^\phi = f_\phi^{(i+1)}(\rvx_{i+1})$, can be further decomposed to:
    \begin{align}
        \frac{\partial \rvx_i^\phi}{\partial \phi} = \frac{\partial f_\phi^{(i+1)}(\rvx_{i+1})}{\partial \rvx_{i+1}^\phi}\cdot \frac{\partial \rvx_{i+1}^\phi}{\partial \phi} + \frac{\partial f_\phi^{(i+1)}(\rvx_{i+1})}{\partial\phi}.
    \end{align}
    Given this recursive pattern with base case $\frac{\partial \rvx_N}{\partial \phi} = 0$ as $\rvx_N$ is independent of $\phi$, we have
    \begin{align}
            \frac{\partial \rvx_{N-1}^\phi}{\partial \phi} &= \frac{\partial f_\phi^{(N)}(\rvx_N)}{\partial \rvx_N^\phi}\cdot \frac{\partial \rvx_N^\phi}{\partial \phi} + \frac{\partial f_\phi^{(N)}(\rvx_N)}{\partial \phi} = \frac{\partial f_\phi^{(N)}(\rvx_N)}{\partial \phi}, \\
            \frac{\partial \rvx_{N-2}^\phi}{\partial \phi} 
            &= \frac{\partial f_\phi^{(N-1)}(\rvx_{N-1})}{\partial \rvx_{N-1}^\phi}\cdot \frac{\partial \rvx_{N-1}^\phi}{\partial\phi} + \frac{\partial f_\phi^{(N-1)}(\rvx_{N-1})}{\partial \phi} \\
            &=  \frac{\partial f_\phi^{(N-1)}(\rvx_{N-1})}{\partial \rvx_{N-1}^\phi} \cdot \frac{\partial f_\phi^{(N)}(\rvx_N)}{\partial \phi} + \frac{\partial f_\phi^{(N-1)}(\rvx_{N-1})}{\partial \phi} \\
            &= \frac{\partial \rvx_{N-2}^\phi}{\partial \rvx_{N-1}^\phi} \cdot \frac{\partial f_\phi^{(N)}(\rvx_N)}{\partial \phi} + \frac{\partial f_\phi^{(N-1)}(\rvx_{N-1})}{\partial \phi}.
    \end{align}
    This gives a general form for $i \in \{0, 1, \cdots, N-1\}$: 
    \begin{align}
        \frac{\partial \rvx_i^\phi}{\partial \phi} &= \sum_{j=i+1}^N \frac{\partial \rvx_i^\phi}{\partial \rvx_{j-1}^\phi}\cdot \frac{\partial f_\phi^{(j)}(\rvx_j)}{\partial \phi}.
    \end{align}
    Then
    \begin{align}
        \frac{\partial F_\phi(\rvx_N)}{\partial \phi} = \frac{\partial \rvx_0^\phi}{\partial \phi} =  \sum_{i=1}^N \frac{\partial \rvx_0^\phi}{\partial \rvx_{i-1}^\phi} \cdot \frac{\partial f_\phi^{(i)}(\rvx_{i})}{\partial \phi}.
    \end{align}
    Therefore,
    \begin{align}
        \nabla_\phi \gL_\mathrm{rollout}(\phi) = \meanp{p(\rvx_N)}{\sum_{i=1}^N \nabla_{\rvx_0^\phi} \ell^\Omega(\rvx_0^\phi) \frac{\partial \rvx_0^\phi}{\partial \rvx_{i-1}^\phi} \frac{\partial f_\phi^{(i)} (\rvx_i)}{\partial \phi}}.
    \end{align}
\end{proof}

\begin{corollary} 
\label[corollary]{corollary: unbiasedness}
    (Unbiased gradient estimator)
    Let $\rvx_N$ be a single initial noise sample and let $\{\rvx_i^\phi\}_{i=0}^{N-1}$ be the trajectory by rolling out $\{f_\phi^{(i)}\}_{i=1}^N$. Then the single sample estimator
    \begin{align}
        h(\phi) = \sum_{i=1}^N \nabla_{\rvx_0^\phi} \ell^\Omega(\rvx_0^\phi) \frac{\partial \rvx_0^\phi}{\partial \rvx_{i-1}^\phi} \frac{\partial f_\phi^{(i)} (\rvx_i)}{\partial \phi}
    \end{align}
    is an unbiased estimator of $\nabla_\phi\gL_{\mathrm{rollout}}(\phi)$.
\end{corollary}

\begin{proof}
This is an immediate conclusion from~\cref{theom: training-sampling-alignment}. Since $h(\phi)$ is a fixed realization, this result is obtained by taking expectation over $p(\rvx_N)$ and interchanging expectation and differentiation, which is,
\begin{align}
        \nabla_\phi \gL_{\mathrm{rollout}}(\phi) = \meanp{p(\rvx_N)}{h(\phi)}.
\end{align}
\end{proof}

\section{Multiple constraint gradients and scaling network architecture}
\label{app: scaling-network}

To handle potential conflicts between multiple gradients, we apply a symmetric gradient projection inspired by PCGrad~\citep{yu2020gradient}. Let $\{\rvg_j\}_{j=1}^J$ denote the gradients associated with $J$ constraint losses. For each pair of constraints $(\rvg_j, \rvg_k)$, if $\rvg_j^\top \rvg_k < 0$, we remove the conflicting component from both gradients:
\begin{align}
    \rvg_k' &\leftarrow \rvg_k - \frac{\rvg_k^\top \rvg_j}{\norm{\rvg_j}^2 + \epsilon}\, \rvg_j, \\
    \rvg_j' &\leftarrow \rvg_j - \frac{\rvg_j^\top \rvg_k}{\norm{\rvg_k}^2 + \epsilon}\, \rvg_k,
\end{align}
where $\epsilon = 10^{-5}$. If $\rvg_j^\top \rvg_k \geq 0$, the gradients are left unchanged. The corrected gradients are then combined to form the final constraint direction. Therefore, the scaling network learns a global correction magnitude rather than per-constraint weights while mitigating destructive interference between constraints. In both our experiments, $J = 2$, so this reduces to the symmetric two gradient updates described above. 

The scaling network $\gamma_\phi$ is implemented as three MLPs with linear layers and SiLU activations. The first MLP processes the noisy state $\rvx_t$ to a feature embedding. The second MLP maps the noise level $t$ to a time embedding, which is split into scale and shift components for FiLM-style~\citep{perez2018film} conditioning of the feature embedding. The third MLP maps the conditioned features to $\alpha_\phi$ and $\beta_\phi$. Together, these outputs determine the correction magnitude through~\cref{eq: scaling-network}. Experiment-specific parameterizations of $\alpha_\phi$ and $\beta_\phi$ are provided in the corresponding experiment detail sections.

\section{Bouncing balls experiment details}

\subsection{Definition of constraint losses and evaluation metrics}
\label{app: def-metrics}
We provide the full definitions of the constraint losses and evaluation metrics used in the bouncing balls experiment. Each scene contains $B=10$ balls in a $10\times 10$ square box, with ball radius $r = 0.5$. Let $\rvx_{\tau}^{(b)}\in\sR^2$ denote the center position of ball $b$ at frame $\tau$, for $\tau = 1, \cdots, \gT$ and $\gT=100$.

\paragraph{Constraint losses} We first define the boundary and overlap losses used to measure constraint violations. The boundary loss measures ball-wall penetration. Since the center of each ball remains in $[r, 10-r]^2$, we define
\begin{align}
    \ell_{\mathrm{boundary}}(\rvx_{\tau}^{(b)}) = \max\{r - \rvx_{\tau, 1}^{(b)}, \rvx_{\tau, 1}^{(b)} - (10-r), r - \rvx_{\tau, 2}^{(b)}, \rvx_{\tau, 2}^{(b)} - (10-r), 0\}.
\end{align}
The overlap loss measures overlap between pairs of balls:
\begin{align}
    \ell_{\mathrm{overlap}}(\rvx_{\tau}^{(b)}, \rvx_{\tau}^{(b')} = \max(2r - \norm{\rvx_{\tau}^{(b)} - \rvx_{\tau}^{(b')}}, 0).
\end{align}

\paragraph{Reported metrics} The boundary rate is the fraction of frames in which at least one ball has positive boundary loss:
\begin{align}
    \text{Boundary rate} = \frac 1\gT \sum_{t=1}^\gT \1\left[\max_b \ell_{\mathrm{boundary}}\left(\rvx_{\tau}^{(b)}\right) > 0\right].
\end{align}
The overlap rate is the fraction of frames in which at least one pair of balls has positive overlap loss:
\begin{align}
    \text{Overlap rate} = \frac 1\gT \sum_{t=1}^\gT \1\left[\max_{b<b'} \ell_{\mathrm{overlap}}\left(\rvx_{\tau}^{(b)}, \rvx_{\tau}^{(b')}\right) > 0\right].
\end{align}
The maximum frame-to-frame displacement (F2F) measures abrupt motion or jitter:
\begin{align}
    \mathrm{F2F} = \max_{b, t<\gT} \norm{\rvx_{\tau+1}^{(b)} -\rvx_{\tau}^{(b)}}_2.
\end{align}
The maximum contact distance (MCD) measures whether sharp direction changes are supported by nearby contacts. We estimate velocity directly from consecutive positions, which is more reliable than using predicted velocity in generated samples. Specifically, we define
\begin{align}
    \rvv_\tau^{(b)} = \rvx_{\tau+1}^{(b)} - \rvx_{\tau}^{(b)}.
\end{align}
We then mark $(b, \tau)$ as a candidate bounce if the angle between $\rvv_{\tau+1}^{(b)}$ and $\rvv_{\tau}^{(b)}$ exceeds $\frac \pi 6$. Formally,
\begin{align}
    \gB = \left\{(b, \tau): \arccos\left(\frac{{\rvv_{\tau-1}^{(b)}}^\top\rvv_{\tau}^{(b)}}{\norm{\rvv_{\tau-1}^{(b)}}\norm{\rvv_{\tau}^{(b)}}+\epsilon}\right) > \frac{\pi}{6}\right\},
\end{align}
where $\tau=2, \cdots, \gT-1$. For each candidate bounce $(b, \tau)$, we compute the signed distance to the nearest possible contact object over a local temporal window
\begin{align}
    \gW_\tau = \{\tau-1, \tau, \tau+1\}.
\end{align}
This accounts for small discretization offsets between the detected direction change and the actual contact frame. The signed wall distance is
\begin{align}
    d_{\mathrm{wall}}(b, s) = \min\{\rvx_{s, 1}^{(b)} - r, 10 - r - \rvx_{s, 1}^{(b)}, \rvx_{s, 2}^{(b)} - r, 10 - r - \rvx_{s, 2}^{(b)}\},
\end{align}
and the signed ball-ball distance is
\begin{align}
    d_{\mathrm{ball}}(b, s) = \min_{b'\neq b} \left(\norm{\rvx_{s}^{(b)} - \rvx_{s}^{(b')}} - 2r\right).
\end{align}
We collect all candidate signed contact distances in the local window:
\begin{align}
    \gD(b, \tau) = \{d_{\mathrm{wall}}(b, s): s\in\gW_\tau\} \cup \{d_{\mathrm{ball}}(b, s): s\in \gW_\tau\}.
\end{align}
MCD is then defined as
\begin{align}
    \mathrm{MCD} = \max_{(b, \tau)\in \gB}\min_{d\in\gD(b, \tau)} |d|.
\end{align}

Finally, the maximum energy deviation (MED) measures frame-wise inconsistency in system kinetic energy. Assuming unit mass and constant time step, we define the kinetic energy proxy
\begin{align}
    E_\tau = \sum_{b=1}^B \frac 12\norm{\rvx_{\tau+1}^{(b)} - \rvx_{\tau}^{(b)}}_2^2.
\end{align}
The maximum energy deviation is:
\begin{align}
    \mathrm{MED} = \max_{t<\gT-1} |E_{\tau+1} - E_\tau|.
\end{align}
All maximum-based metrics are computed within each generated trajectory and then averaged over evaluation batches. We report the mean and standard deviation over four independent evaluation runs.

\subsection{Ablation study}
\label{app: bb-ablation}
We further conduct ablations to isolate the contribution of each component in our framework. FP training replaces rollout-based training with standard forward process training, which introduces a mismatch between training and sampling dynamics. It keeps the learned scaling network and gradient embeddings trainable. Rollout-only keeps rollout-based training and the same trainable scaling network and gradient embeddings, but removes LoRA adaptations. Finally, \method{} combines rollout-based training approach, the learned scaling network, gradient embeddings, and LoRA adaptation.

\begin{table}[t]
    \centering
    \small
    \setlength{\tabcolsep}{4pt}
    \caption{Ablation study of bouncing ball experiment}
    \begin{tabular}{lcccccc}
    \toprule
    & \multicolumn{2}{c}{Constraint (\%) $\downarrow$} & Fidelity $\uparrow$ & \multicolumn{3}{c}{Physical plausibility ($\times 10^{-1}$) $\downarrow$} \\
    \cmidrule(lr){2-3} \cmidrule(lr){5-7}
    Method & Boundary rate & Overlap rate & r-ELBO ($\times 10^{-2}$) & F2F & MCD & MED \\
    \midrule
     FP training  & $0.78\pm 0.03$ & $2.83\pm 0.06$ & $-21.7\pm 0.1$ & $4.1\pm 0.0$ & $0.5\pm 0.0$ & $1.3\pm 0.1$  \\
     Rollout-only & $0.07\pm 0.01 $ & $0.01\pm 0.00$ & $-23.4\pm 0.0$ & $5.4\pm 0.0$ & $0.9\pm 0.0$ & $1.9\pm 0.1$ \\
     \method{} & $0.01\pm 0.00$ & $0.01\pm 0.00$ & $-22.9\pm 0.0$ & $3.4\pm 0.0$ & $1.2\pm 0.0$ & $0.6\pm 0.0$ \\
    \bottomrule
    \end{tabular}
    \label{tab: bb-table}
\end{table}

\begin{figure}[t]
    \centering
    \includegraphics[width=\linewidth]{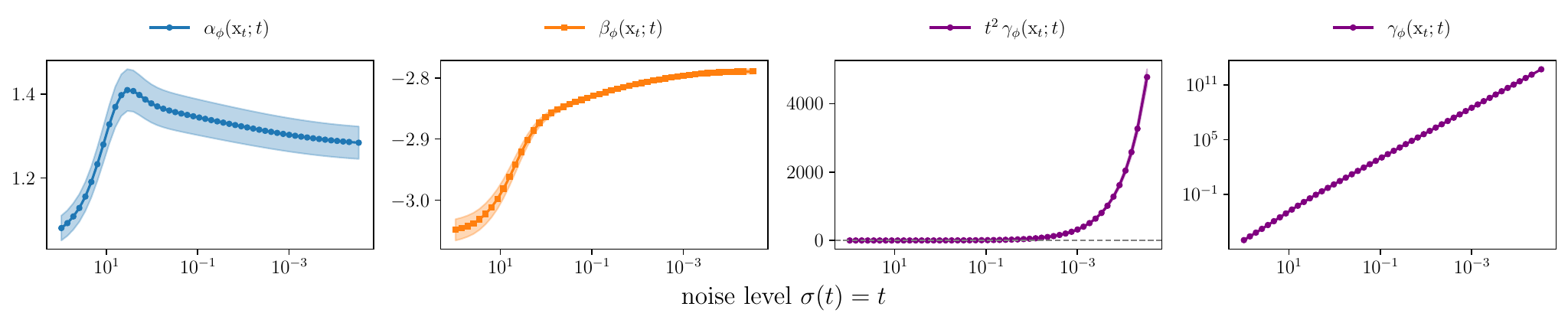}
    \caption{Scaling network output versus noise level $\sigma$ along denoising trajectory for bouncing ball experiment, Means are taken over batches; shaded bands show $\pm 1$ standard error.}
    \label{fig: bb-scalings}
\end{figure}

Rather than performing a discrete ablation on the scaling network, we analyze its learned behavior across noise levels. As shown in~\cref{fig: bb-scalings}, both learned components are noise dependent: $\alpha_\phi$ is not a constant multiplier, and $\beta_\phi$ varies smoothly along the sampling trajectory. To compare with prior schedules of the form $ct^{-2}$ in~\cref{tab: scaling-table}, we plot $t^2\gamma_\phi(\rvx_t; t)$. This quantity would be constant under such a $ct^{-2}$ schedule, but instead changes substantially across noise levels. This suggests that the scaling network learns a nontrivial modulation beyond a fixed power law. Moreover, $\gamma_\phi(\rvx_t; t)$ grows rapidly as the noise level decreases, consistent with stronger constraint enforcement near the data manifold.

\subsection{Effect of sampling steps on sampling-time correction}
MPGD without projection~\citep{hemanifold} performs constraint correction only during sampling, without updating the pretrained model. We evaluate it on the bouncing ball task with $50, 100, 150, 200$ denoising steps, where $50$ matches the main method setting. This ablation tests how sensitive sampling-time correction is to the number of steps for incremental updates along the denoising trajectory.

\begin{table}[t]
    \centering
    \small
    \setlength{\tabcolsep}{4pt}
    \caption{Effect of sampling steps on MPGD without projection in the bouncing ball task}
    \begin{tabular}{lccccc}
    \toprule
    & \multicolumn{2}{c}{Constraint (\%) $\downarrow$} & \multicolumn{3}{c}{Physical plausibility ($\times 10^{-1}$) $\downarrow$} \\
    \cmidrule(lr){2-3} \cmidrule(lr){4-6}
    Steps & Boundary rate & Overlap rate & F2F & MCD & MED \\
    \midrule
     50  & $0.0\pm 0.0$ & $0.0\pm 0.0$ & $8.6\pm 0.4$ & $0.8\pm 0.0$ & $29.7\pm 5.4$  \\
     100 & $0.0\pm 0.0$ & $0.0\pm 0.0$ & $3.8\pm 0.2$ & $0.6\pm 0.0$ & $3.2\pm 0.9$ \\
     150 & $0.0\pm 0.0$ & $0.0\pm 0.0$ & $3.2\pm 0.0$ & $0.6\pm 0.0$ & $0.6\pm 0.1$ \\
     200 & $0.0\pm 0.0$ & $0.0\pm 0.0$ & $3.1\pm 0.0$ & $0.6\pm 0.0$ & $0.4\pm 0.1$ \\
    \bottomrule
    \end{tabular}
    \label{tab:mpgd-sampling-steps}
\end{table}

\cref{tab:mpgd-sampling-steps} shows that MPGD achieves zero boundary and collision violations across all various sampling steps, but its physical plausibility depends strongly on the number of denoising steps. With $50$ steps, matching the main evaluation setting, MPGD exhibits large frame-to-frame changes (F2F) and high energy deviation (MED), indicating that feasibility is obtained through many incremental corrections. This supports the claim that sampling-time correction methods can enforce constraints, but may rely on many incremental updates to avoid physically infeasible trajectories.

\subsection{Scaling network parameterization}
The third MLP outputs $\alpha_\phi$ and a raw exponent parameter $\beta_\phi^{\mathrm{raw}}$. We use a softplus output for $\alpha_\phi$ to ensure a nonnegative magnitude, and map the raw exponent to the effective component as $2\mathrm{sigmoid}(\beta_\phi^{\mathrm{raw}})-4$. This offset incorporates the known scaling of the denoised-state guidance term: the denoiser Jacobian contributes a factor $\cov{}{\rvx_0\mid \rvx_t}/t^2$, and the score parameterization $(\meanp{}{\rvx_0\mid \rvx_t} - \rvx_t)/t^2$ contributes an additional $t^{-2}$ factor. Thus the $t^{-4}$ component is fixed analytically, while the network learns the remaining state- and noise-dependent covariance modulation. This parameterization preserves the desired divergence property as noise level approaches zero without requiring a manually tuned schedule.

\subsection{Training details}
We follow the prior work~\citep{liang2026improved} for the experimental setup, using the same transformer-based diffusion backbone~\citep{vaswani2017attention, zhang2022motiondiffuse, beltagy2020longformer} and gradient embedding MLP. We use the same log-linear noise schedule for sampling with $\sigma_{\min}=3\times 10^{-5}$ and $\sigma_{\max} = 80$ but use 50 sampling steps. The standard EDM model is trained from scratch for $950$k iterations with learning rate $3\times 10^{-4}$ on a Tesla V100 GPU, taking approximately $70$ hours. Our proposed fine-tuned model additionally uses the scaling network described in~\cref{app: scaling-network}, and includes LoRA adapters. Our model and other baseline models are either initialized from scratch or from the pretrained EDM checkpoint, as specified in their method descriptions, and are trained with learning rate $3\times 10^{-5}$ for approximately $72$ hours on an NVIDIA L40S GPU. The number of training iterations varies across methods due to differences in computational cost. All models use the Adam optimizer; our method and DPOK use batch size $16$, while other fine-tuning models use batch size $32$.

\section{Traffic scene trajectory prediction experiment details}
\subsection{Definition of constraint losses and evaluation metrics}
We provide the definition of the constraint losses and evaluation metrics used in the traffic scene prediction experiment. Each scene contains a varying number of agents. Given the first 10 frames observed, the model jointly predicts all agents over a future prediction horizon $\gT=30$ frames. We denote the predicted center position of the oriented bounding box for agent $a$ at frame $\tau$ by $\rvx_{\tau}^{(a)}\in\sR^2$, and denote the corresponding bounding box by $\gO_\tau^{(a)}$. In our implementation, the first agent in the predicted scene tensor is the ego agent.

\paragraph{Constraint losses}
The offroad loss measures how far a vehicle extends outside the drivable region. Given the drivable mesh $\gM$, we compute the four corners of each vehicle box. For a corner $\rvx_{\tau, j}^a$, where $j = \{1, \cdots, 4\}$, we define its distance to drivable mesh as
\begin{align}
    d_{\mathrm{mesh}}\left(\rvc_{\tau, j}^{(a)}\right) = \min_{\rvm\in\gM}\norm{\rvc_{\tau, j}^{(a)} - \rvm}_2^2.
\end{align}
The offroad loss is the maximum distance among corners that lie outside the drivable region:
\begin{align}
    \ell_{\mathrm{offroad}}\left(\gO_\tau^{(a)}\right) = \max_j \1\left[\rvc_{\tau, j}^{(a)}\in\gM\right]d_{\mathrm{mesh}}\left(\rvc_{\tau, j}^{(a)}\right).
\end{align}
The collision loss measures the overlap area between pairs of vehicles. For two agents $a$ and $a'$, we define
\begin{align}
    \ell_{\mathrm{collision}}\left(\gO_{\tau}^{(a)}, \gO_{\tau}^{(a')}\right) = \mathrm{Area}\left(\gO_{\tau}^{(a)}, \gO_{\tau}^{(a')}\right).
\end{align}
\paragraph{Reported metrics}
For each scene, we generate $K=6$ predicted futures. We use $k\in\{1, \cdots K\}$ to index generated samples and $a=0$ to denote the ego agent. The offroad rate is the fraction of scenes in which the ego vehicle goes offroad at least once during the future prediction horizon
\begin{align}
    \text{Offroad rate} = \frac 1K \sum_{k=1}^K \1\left[\max_{\tau=1, \cdots, \gT} \ell_{\mathrm{offroad}}(\gO_{\tau}^{k, (0)})>0\right].
\end{align}

The collision rate is the fraction of scenes in which the ego vehicle collides at least once with another vehicle during the future prediction
\begin{align}
    \text{Collision rate} = \frac 1K \sum_{k=1}^K \1\left[\max_{\tau=1, \cdots, \gT}\ell_{\mathrm{collision}}(\gO_{\tau}^{k, (0)}) > 0\right].
\end{align}
The scene maximum frame-to-frame displacement (SF2F) measures abrupt motion or jitter, same as in the bouncing ball experiment:
\begin{align}
    \mathrm{SF2F}_6 = \max_{k, a, \tau<\gT} \norm{\rvx_{\tau+1}^{k, (a)} - \rvx_{\tau}^{k, (a)}}_2.
\end{align}
Both offroad and collision rates are then averaged over evaluation batches. 

The scene maximum lateral velocity (SMLV) measures sideways motion relative to each vehicle heading. We estimate velocity from consecutive predicted positions:
\begin{align}
    \rvv_{\tau}^{k, (a)} = \rvx_{\tau+1}^{k, (a)} - \rvx_{\tau}^{k, (a)}.
\end{align}
Given heading $h_\tau^{k, (a)}$, the lateral velocity is
\begin{align}
    \rvv_{\mathrm{lat}, \tau}^{k, (a)} = - \rvv_{\tau, x}^{k, (a)}\sin h_\tau^{k, (a)} + \rvv_{\tau}^{k, (a)}\cos h_\tau^{k, (a)}.
\end{align}
We define maximum lateral velocity as
\begin{align}
    \text{SMLV}_6 = \max_{k, a, \tau<\gT} \norm{\rvv_{\mathrm{lat}, \tau}^{k, (a)}}.
\end{align}

The maximum final distance (MFD) measures endpoint diversity among generated ego trajectories. We compute the maximum pairwise distance between final predicted ego position across $K$ generated samples,
\begin{align}
    \text{MFD}_6 = \max_{k, k'} \norm{\rvx_{\gT}^{k, (0)} - \rvx_{\gT}^{k', (0)}}.
\end{align}

SF2F$_6$, SMLV$_6$, MFD$_6$ are computed within $K=6$ generated samples, and then average over evaluation batches.

\begin{figure}[t]
    \centering
    \includegraphics[width=\linewidth]{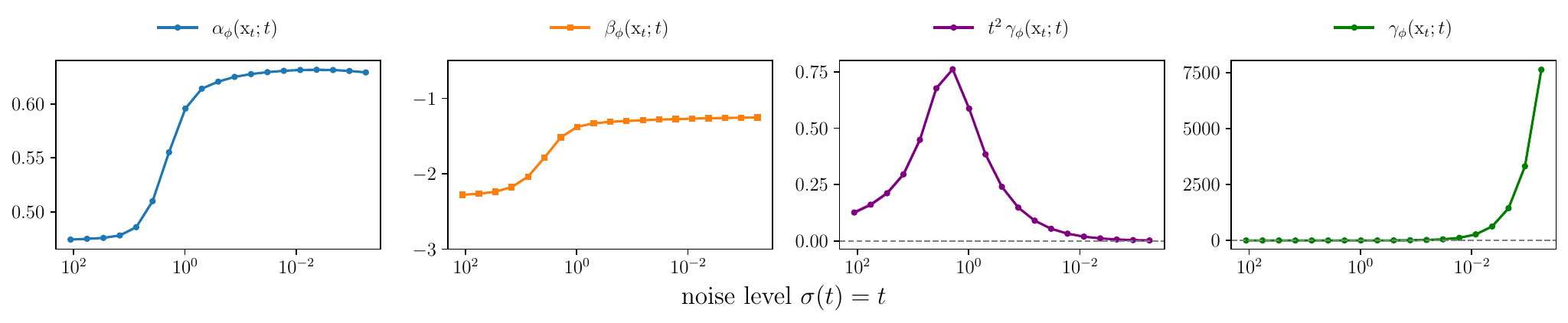}
    \caption{Scaling network output versus noise level $\sigma$ along denoising trajectory for traffic scene trajectory prediction experiment. Means are taken over batches; shaded bands show $\pm 1$ standard error.}
    \label{fig: dm-scalings}
\end{figure}

\subsection{Method variants}
\label{app: method-variants}
We report two additional variants to our~\method{} in~\cref{tab: scene-variant-table}: Rollout-only trains the scaling network and gradient embedding without LoRA adaptation; and~\method{} (staged) uses the same trainable components and training objective as~\method{}, but initializes from the trained Rollout-only checkpoint before enabling LoRA adapters for a second fine-tuning stage.~\method{} is the variant reported in~\cref{tab: scene-table}. As shown in~\cref{tab: scene-variant-table}, \method{} (staged) further reduces the offroad rate, but requires additional training through the initial rollout-only stage. 

\subsection{Effect of scaling network and its parameterization}
As shown in~\cref{fig: dm-scalings}, $\alpha_\phi(\rvx_t; t)$ and $\beta_\phi(\rvx_t; t)$ both change with the noise level, so the learned scaling does not collapse to a constant coefficient. Unlike the bouncing balls case, $t^2\gamma_\phi(\rvx_t; t)$ peaks at an intermediate noise level, suggesting that the model places the strongest relative correction in the middle of the trajectory. The final scale $\gamma_\phi(\rvx_t; t)$ still grows rapidly as $t$ decreases, preserving strong constraint enforcement near the data manifold.

For parameterization, the third MLP outputs $\alpha_\phi$ and a raw exponent parameter $\beta_\phi^{\mathrm{raw}}$ as before. We apply a softplus to $\alpha_\phi$ to keep the correction magnitude nonnegative, and map the raw exponent to the effective exponent as $2\mathrm{sigmoid}(\beta_\phi^{\mathrm{raw}})-3$. This effective exponent, rather than the raw MLP output, is the $\beta_\phi$ used in the scaling function. This retains the $t^{-2}$ factor induced by the score parameterization, while using the remaining $t^{-1}$ factor as a milder base scaling. Compared with the $-4$ offset used in bouncing balls, this reduces the strength of low-noise guidance. We conjecture that this milder scaling is better suited to trajectory prediction because it uses few sampling steps, where overly aggressive low-noise guidance can lead to abrupt late-stage correction. The bounded learned exponent keeps the divergence property as $t$ goes to zero.

\subsection{Training details}
We follow prior work~\citep{liang2026improved} for the experimental setup, using the same transformer-based diffusion backbone~\citep{niedoba2024diffusion} and gradient embedding MLP as the bouncing ball experiment. Our fine-tuned models additionally use the scaling network described in~\cref{app: scaling-network} and LoRA adapters. We use the same log-linear sampling noise schedule with $\sigma_{\max}=80$ and $\sigma_{\min} = 2\times 10^{-4}$, and sample with $20$ steps. The standard diffusion model is trained from scratch with learning rate $3\times 10^{-4}$ on a Tesla V100 GPU, taking approximately 168 hours for about 140 epochs. Our fine-tuned models, and other fine-tuned baselines are trained with learning rate $3\times 10^{-4}$ for approximately $72$ hours on two NVIDIA L40S GPUs, except for PIDM, which uses learning rate $3\times 10^{-5}$, and \method{} (staged). The \method{} (staged) variant follows the staged procedure described above: the Rollout-only model is trained for three days, and the LoRA adapters are then trained for another three days. The number of training iterations varies across methods due to differences in computational cost. All models use the Adam optimizer with batch size $16$.

\subsection{More experimental results}
\label{app: dm-more-results}
\begin{table}[!h]
    \centering
    \small
    \setlength{\tabcolsep}{4.5pt}
    \caption{
    Traffic scene trajectory prediction results on the INTERACTION DR\_DEU\_Merging\_MT scenario. 
    }
    \label{tab: scene-variant-table}
    \begin{tabular}{lccccccccc}
        \toprule
        & \multicolumn{2}{c}{Constraint (\%) $\downarrow$}
        & \multicolumn{4}{c}{Min traj. match $\downarrow$}
        & \multicolumn{2}{c}{Phys. plaus. $\downarrow$}
        & Diversity $\uparrow$ \\
        \cmidrule(lr){2-3}
        \cmidrule(lr){4-7}
        \cmidrule(lr){8-9}
        \cmidrule(lr){10-10}
        Method
        & Offroad & Collision & ADE$_6$ & FDE$_6$ & SADE$_6$ & SFDE$_6$ & SF2F$_6$ & SMLV$_6$ & MFD$_6$ \\
        \midrule
        Rollout-only & $0.49$ & $0.02$ & $0.18$ & $0.45$ & $0.23$ & $0.66$ & $1.10$ & $0.49$ & $1.92$ \\
        \method{} & $0.30$ & $0.00$ & $0.18$ & $0.45$ & $0.23$ & $0.66$ & $1.07$ & $0.43$ & $1.91$ \\
        \method{} (staged) & $0.04$ & $0.00$ & $0.19$ & $0.47$ & $0.24$ & $0.69$ & $1.02$ & $0.37$ & $1.96$ \\
        \bottomrule
    \end{tabular}
\end{table}

\begin{table}[t]
    \centering
    \small
    \setlength{\tabcolsep}{4.5pt}
    \caption{
    Traffic scene trajectory prediction results on the INTERACTION DR\_DEU\_Roundabout\_OF scenario. 
    }
    \begin{tabular}{lccccccccc}
        \toprule
        & \multicolumn{2}{c}{Constraint (\%) $\downarrow$}
        & \multicolumn{4}{c}{Min traj. match $\downarrow$}
        & \multicolumn{2}{c}{Phys. plaus. $\downarrow$}
        & Diversity $\uparrow$ \\
        \cmidrule(lr){2-3}
        \cmidrule(lr){4-7}
        \cmidrule(lr){8-9}
        \cmidrule(lr){10-10}
        Method
        & Offroad & Collision & ADE$_6$ & FDE$_6$ & SADE$_6$ & SFDE$_6$ & SF2F$_6$ & SMLV$_6$ & MFD$_6$ \\
        \midrule
        EDM~\citep{karras2022elucidating} & $23.92$ & $1.70$ & $0.35$ & $1.11$ & $0.50$ & $1.80$ & $1.15$ & $0.49$ & $6.00$  \\
        MPGD~\citep{hemanifold} & $11.96$ & $1.42$ & $0.33$ & $1.04$ & $0.45$ & $1.57$ & $1.24$ & $0.65$ & $3.98$ \\
        MBM++~\citep{liang2026improved} & $12.42$ & $1.40$ & $0.32$ & $1.05$ & $0.44$ & $1.59$ & $1.21$ & $0.61$ & $3.91$\\
        PIDM (FT)~\citep{bastek2024physics} & $14.12$ & $1.35$ & $0.34$ & $1.24$ & $0.49$ & $1.93$ & $0.96$ & $0.33$ & $5.00$ \\
        DPOK~\citep{fan2023dpok} & $3.73$ & $0.11$ & $2.99$ & $8.88$ & $2.96$ & $9.94$ & $0.97$ & $0.56$ & $2.42$ \\
        CriticSMC~\citep{lioutascritic} & -- & $0.08$ & $0.45$ & -- & -- & -- & -- & -- & $3.43$ \\
        \midrule
        Rollout-only & $0.68$ & $0.03$ & $0.37$ & $1.06$ & $0.51$ & $1.63$ & $2.07$ & $1.35$ & $4.07$ \\
        \method{} & $1.22$ & $0.01$ & $0.37$ & $1.07$ & $0.51$ & $1.66$ & $1.69$ & $1.06$ & $4.02$ \\
        \method{} (staged) & $0.37$ & $0.00$ & $0.37$ & $1.08$ & $0.51$ & $1.66$ & $1.68$ & $1.05$ & $3.95$ \\
        \midrule
        Ground truth & $0.00$ & $0.00$ & -- & -- & -- & -- & $0.84$ & $0.06$ & -- \\
        \bottomrule
    \end{tabular}
\end{table}

\begin{table}[!h]
    \centering
    \small
    \setlength{\tabcolsep}{4.5pt}
    \caption{
    Traffic scene trajectory prediction results on the INTERACTION DR\_USA\_Intersection\_MA
    }
    \begin{tabular}{lccccccccc}
        \toprule
        & \multicolumn{2}{c}{Constraint (\%) $\downarrow$}
        & \multicolumn{4}{c}{Min traj. match $\downarrow$}
        & \multicolumn{2}{c}{Phys. plaus. $\downarrow$}
        & Diversity $\uparrow$ \\
        \cmidrule(lr){2-3}
        \cmidrule(lr){4-7}
        \cmidrule(lr){8-9}
        \cmidrule(lr){10-10}
        Method
        & Offroad & Collision & ADE$_6$ & FDE$_6$ & SADE$_6$ & SFDE$_6$ & F2F$_6$ & MLV$_6$  & MFD$_6$ \\
        \midrule
        EDM~\citep{karras2022elucidating} & $1.36$ & $1.54$ & $0.24$ & $0.66$ & $0.40$ & $1.37$ & $1.29$ & $0.46$ & $4.44$ \\
        MPGD~\citep{hemanifold} & $0.49$ & $0.83$ & $0.22$ & $0.64$ & $0.35$ & $1.22$ & $1.24$ & $0.42$ & $3.17$ \\
        MBM++~\citep{liang2026improved} & $0.42$ & $0.83$ & $0.20$ & $0.63$ & $0.33$ & $1.21$ & $1.22$ & $0.36$ & $3.11$ \\
        PIDM (FT)~\citep{bastek2024physics} & $1.15$ & $1.20$ & $0.21$ & $0.67$ & $0.36$ & $1.39$ & $1.20$ & $0.28$ & $4.26$ \\
        DPOK~\citep{fan2023dpok} & $0.37$ & $0.04$ & $2.07$ & $6.28$ & $2.10$ & $6.60$ & $0.96$ & $0.52$ & $1.77$ \\
        CriticSMC~\citep{lioutascritic} & -- & $0.09$ & $0.45$ & -- & -- & -- & -- & -- & $2.87$ \\
        \midrule
        Rollout-only & $0.03$ & $0.00$ & $0.23$ & $0.64$ & $0.35$ & $1.22$ & $1.33$ & $0.53$ & $3.17$ \\
        \method{} & $0.05$ & $0.00$ & $0.22$ & $0.65$ & $0.35$ & $1.23$ & $1.26$ & $0.46$ & $3.17$ \\
        \method{} (staged) & $0.01$ & $0.00$ & $0.23$ & $0.67$ & $0.36$ & $1.26$ & $1.24$ & $0.45$ & $3.17$ \\
        \midrule
        Ground truth & $0.00$ & $0.00$ & -- & -- & -- & -- & $1.11$ & $0.05$ & --  \\
        \bottomrule
    \end{tabular}
\end{table}

\begin{table}[!h]
    \centering
    \small
    \setlength{\tabcolsep}{4.5pt}
    {
    Traffic scene trajectory prediction results on the INTERACTION DR\_USA\_Roundabout\_FT
    }
    \begin{tabular}{lccccccccc}
        \toprule
        & \multicolumn{2}{c}{Constraint (\%) $\downarrow$}
        & \multicolumn{4}{c}{Min traj. match $\downarrow$}
        & \multicolumn{2}{c}{Phys. plaus. $\downarrow$}
        & Diversity $\uparrow$ \\
        \cmidrule(lr){2-3}
        \cmidrule(lr){4-7}
        \cmidrule(lr){8-9}
        \cmidrule(lr){10-10}
        Method
        & Offroad & Collision & ADE$_6$ & FDE$_6$ & SADE$_6$ & SFDE$_6$ & F2F$_6$ & MLV$_6$ & MFD$_6$ \\
        \midrule
        EDM~\citep{karras2022elucidating} & $9.32$ & $1.37$ & $0.25$ & $0.72$ & $0.41$ & $1.44$ & $1.23$ & $0.47$ & $4.26$ \\
        MPGD~\citep{hemanifold} & $2.64$ & $0.93$ & $0.23$ & $0.68$ & $0.35$ & $1.24$ & $1.21$ & $0.50$ & $2.97$ \\
        MBM++~\citep{liang2026improved} & $2.49$ & $0.92$ & $0.21$ & $0.68$ & $0.34$ & $1.23$ & $1.18$ & $0.46$ & $2.92$ \\
        PIDM (FT)~\citep{bastek2024physics} & $5.79$ & $0.86$ & $0.23$ & $0.83$ & $0.39$ & $1.53$ & $1.02$ & $0.31$ & $3.87$ \\
        DPOK~\citep{fan2023dpok} & $1.95$ & $0.06$ & $2.55$ & $7.84$ & $2.65$ & $8.25$ & $1.04$ & $0.58$  & $1.95$ \\
        CriticSMC~\citep{lioutascritic} & -- & $0.07$ & $0.44$ & -- & -- & -- & -- & -- & $2.97$ \\
        \midrule
        Rollout-only & $0.28$ & $0.11$ & $0.24$ & $0.70$ & $0.36$ & $1.26$ & $1.60$ & $1.01$ & $3.01$ \\
        \method & $0.25$ & $0.06$ & $0.24$ & $0.71$ & $0.37$ & $1.28$ & $1.37$ & $0.75$ & $2.94$ \\
        \method{} (staged) & $0.16$ & $0.05$ & $0.24$ & $0.72$ & $0.36$ & $1.29$ & $1.35$ & $0.77$ & $3.00$ \\
        \midrule
        Ground truth & $0.00$ & $0.00$ & -- & -- & -- & -- & $0.95$ & $0.06$ & -- \\
        \bottomrule
    \end{tabular}
\end{table}

In the rendered visualizations below, we show the final predicted frame so that the full trajectory is visible. Gray dot trajectories indicate ground-truth agent motion, and orange trajectories denote model predictions.

\begin{figure*}[!th]
    \centering
    \begin{subfigure}[b]{0.24\textwidth}
        \centering
       \includegraphics[width=1\textwidth]{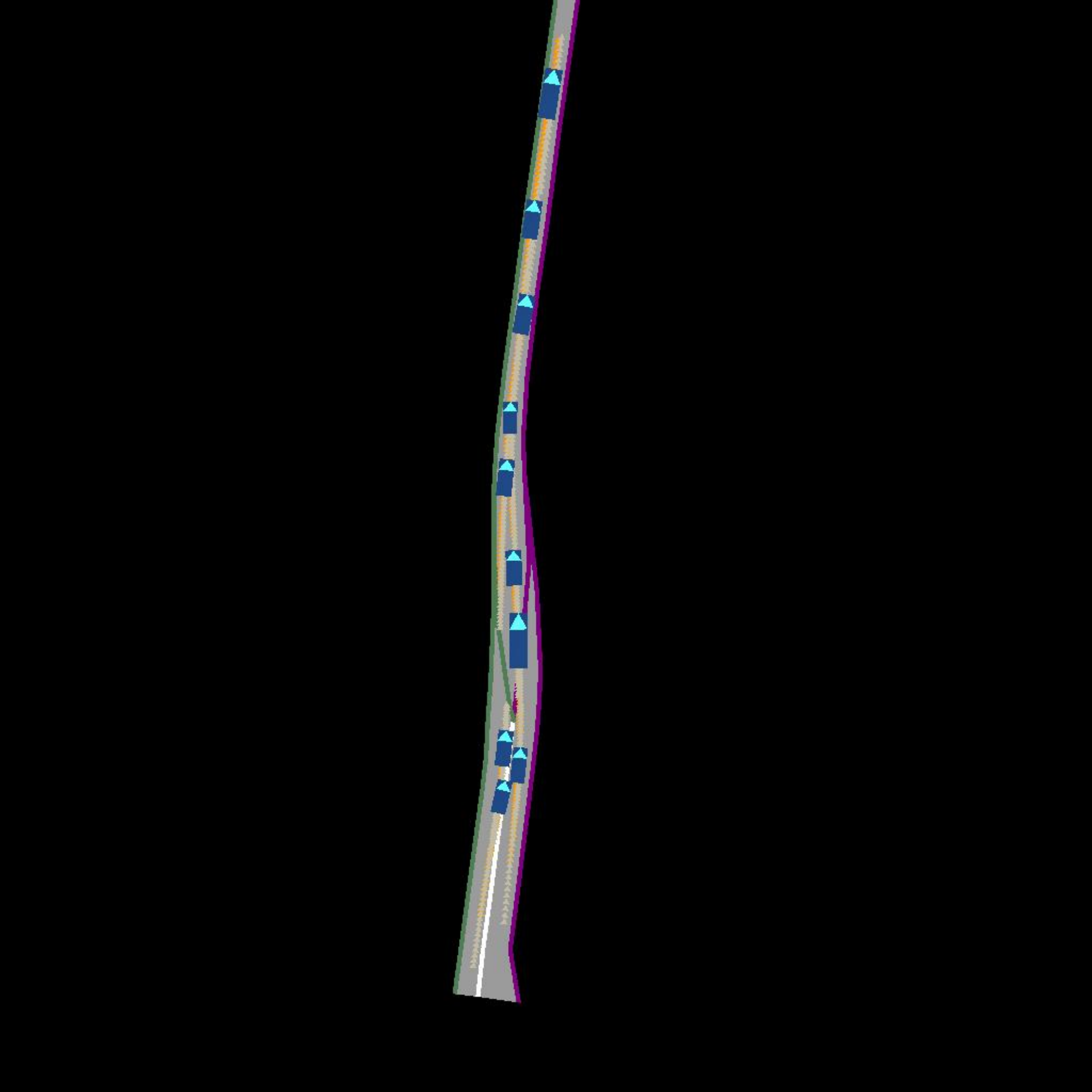}
       \caption*{Standard Diffusion}
        \label{}
    \end{subfigure}
    \hfill
    \begin{subfigure}[b]{0.24\textwidth}
        \centering
        \includegraphics[width=1\textwidth]{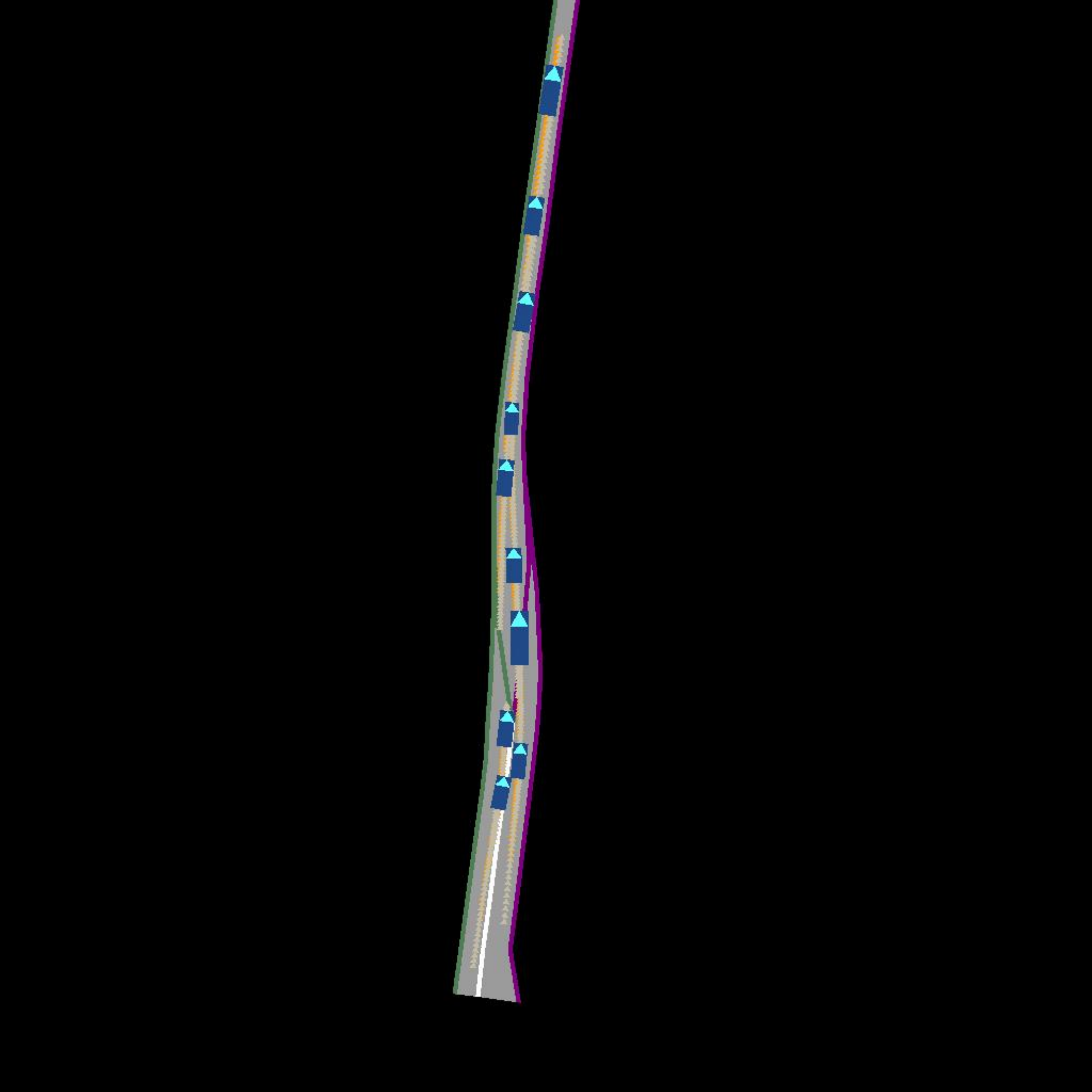}
        \caption*{MPGD w/o projection}
        \label{}
    \end{subfigure}
    \hfill
    \begin{subfigure}[b]{0.24\textwidth}
        \centering
        \includegraphics[width=1\textwidth]{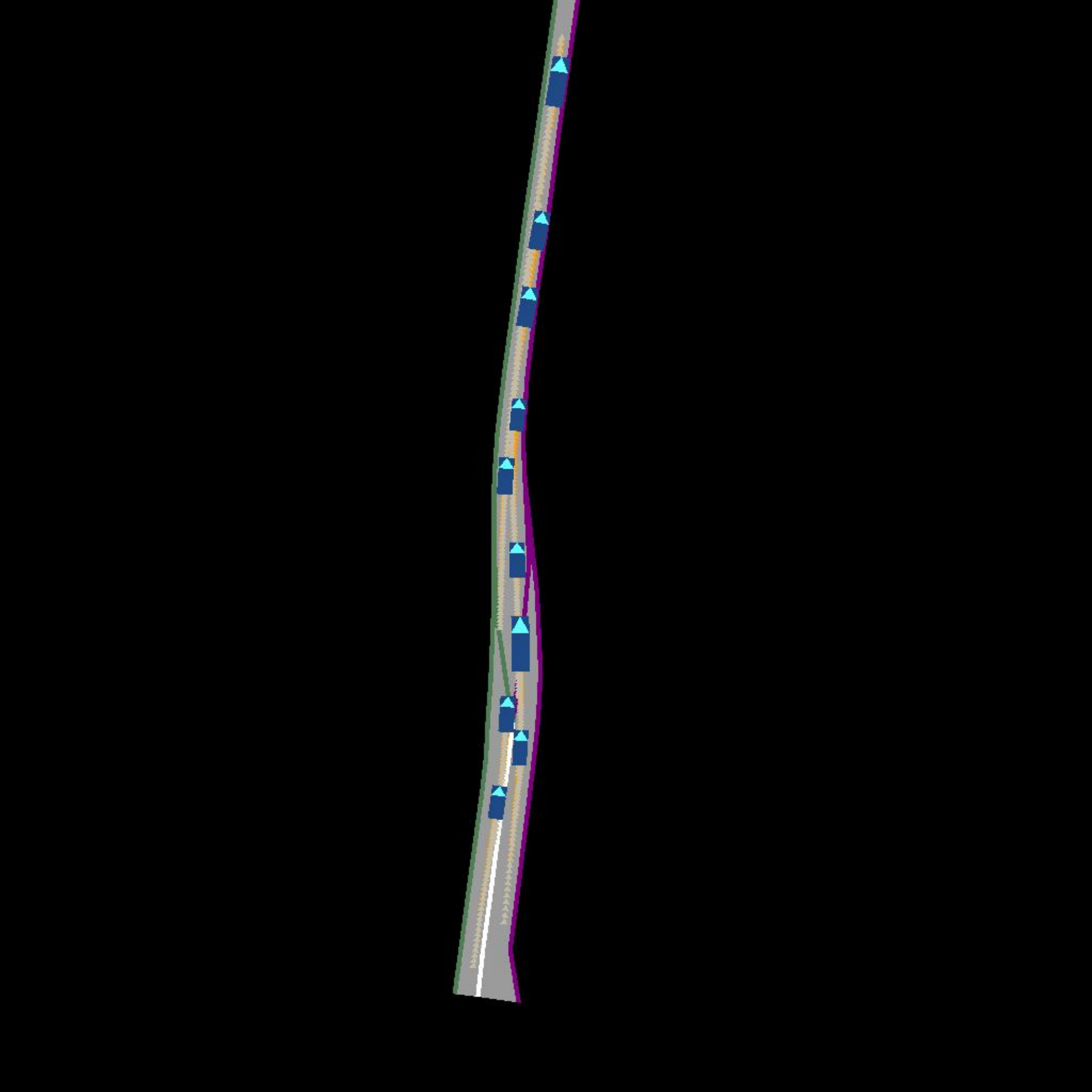}
        \caption*{MBM++}
        \label{}
    \end{subfigure}
    \hfill
    \begin{subfigure}[b]{0.24\textwidth}
        \centering
        \includegraphics[width=1\textwidth]{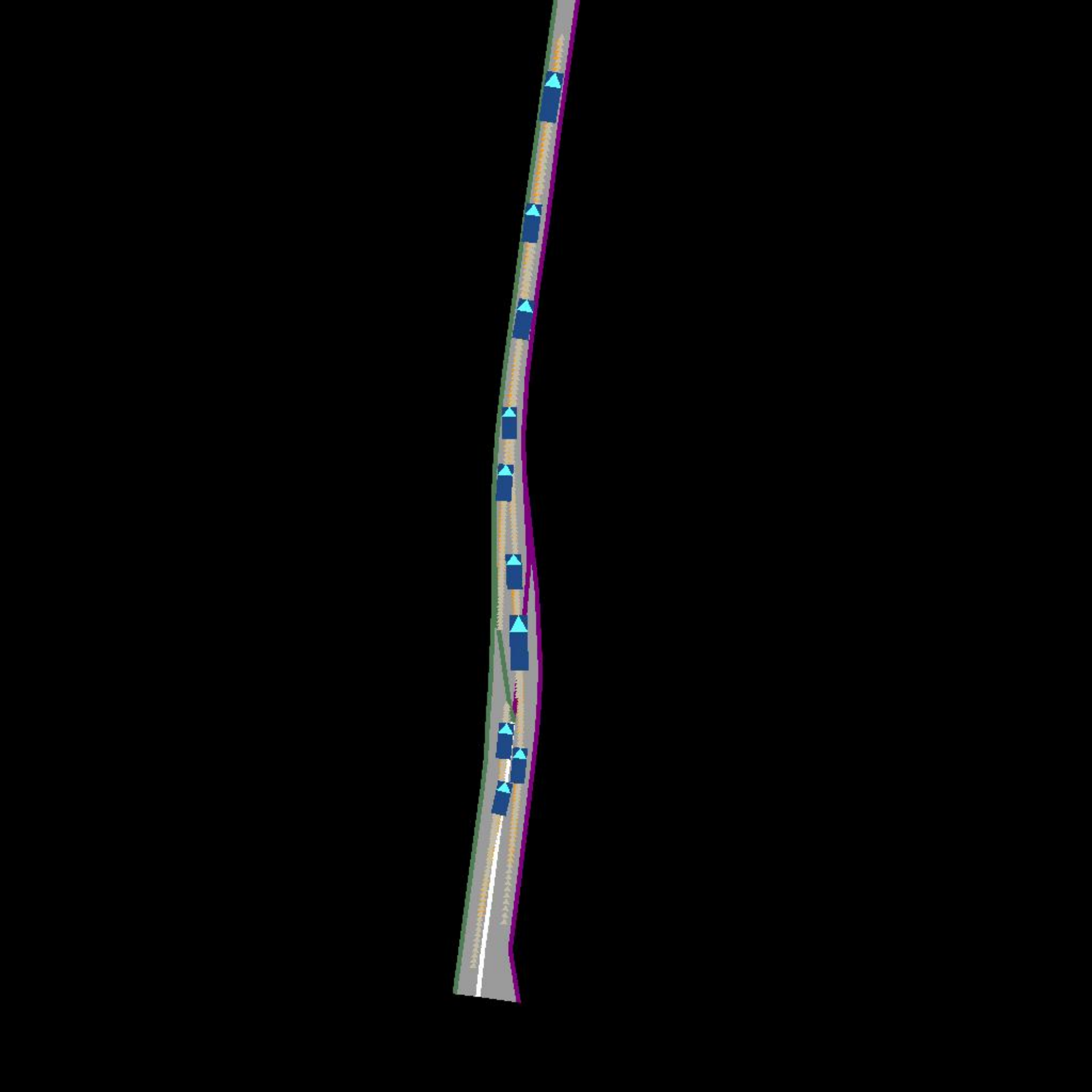}
        \caption*{PIDM}
        \label{}
    \end{subfigure}
    \vspace{0.6em}
    \begin{subfigure}[b]{0.24\textwidth}
        \centering
       \includegraphics[width=1\textwidth]{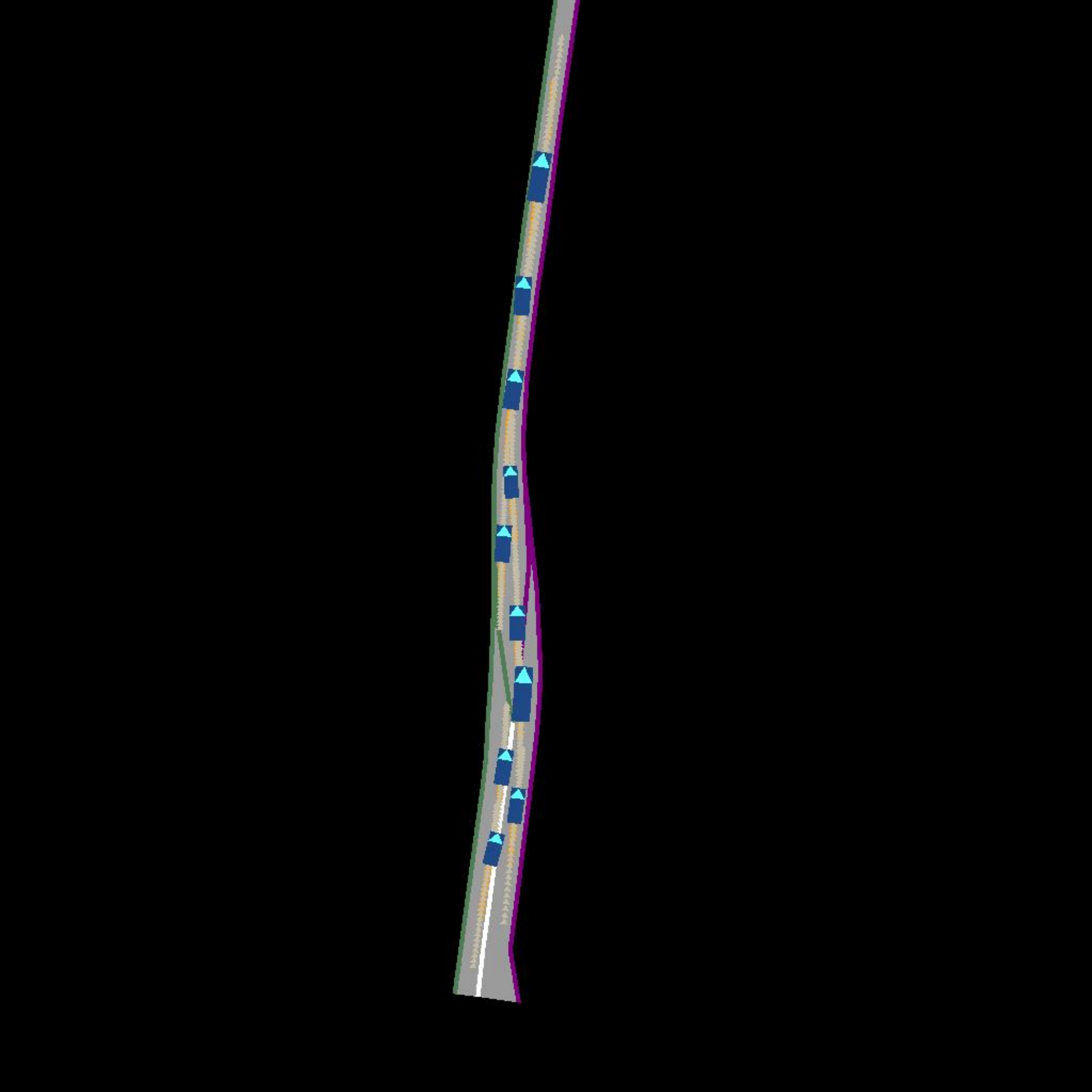}
       \caption*{DPOK}
        \label{}
    \end{subfigure}
    \hfill
    \begin{subfigure}[b]{0.24\textwidth}
        \centering
        \includegraphics[width=1\textwidth]{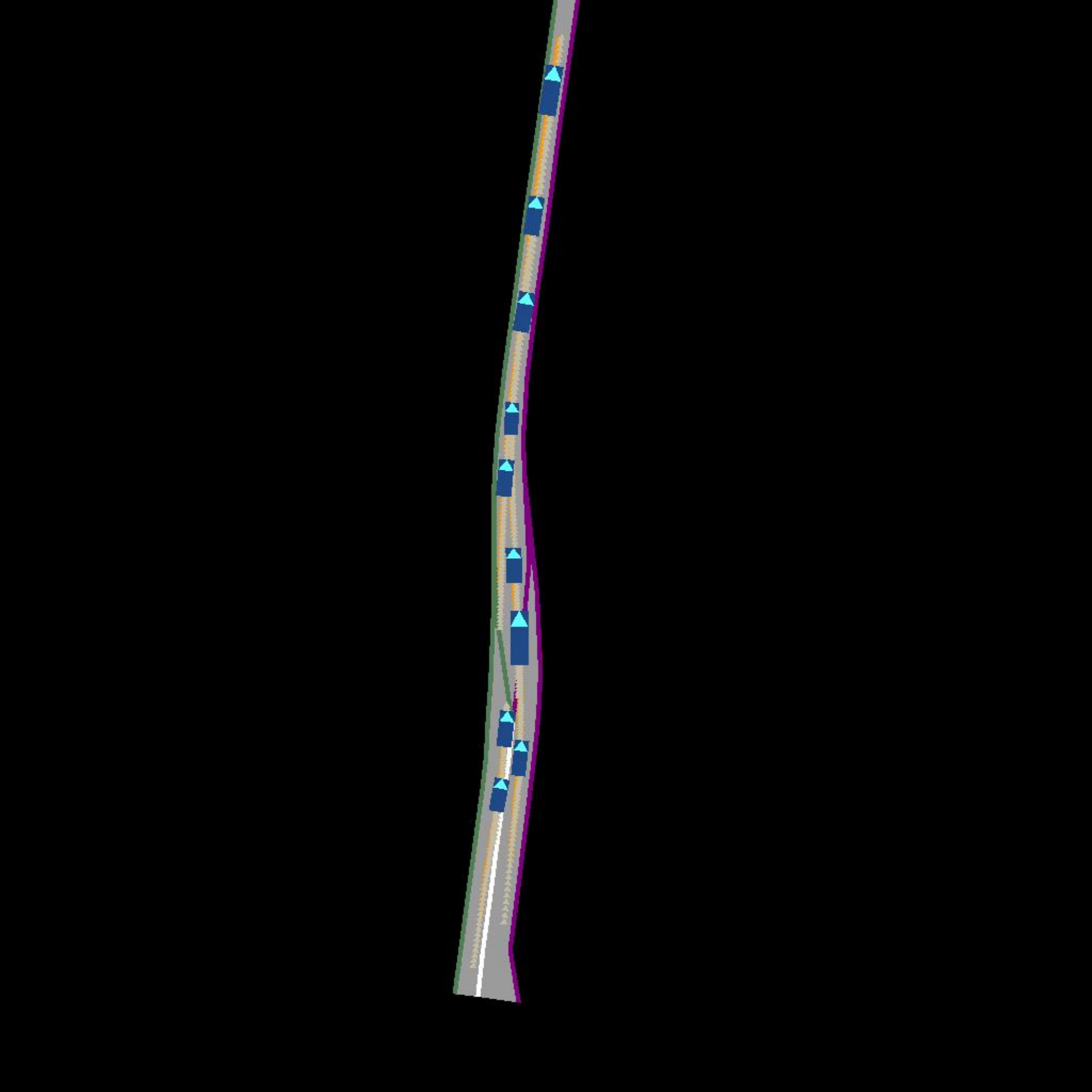}
        \caption*{Rollout-only}
        \label{}
    \end{subfigure}
    \hfill
    \begin{subfigure}[b]{0.24\textwidth}
        \centering
        \includegraphics[width=1\textwidth]{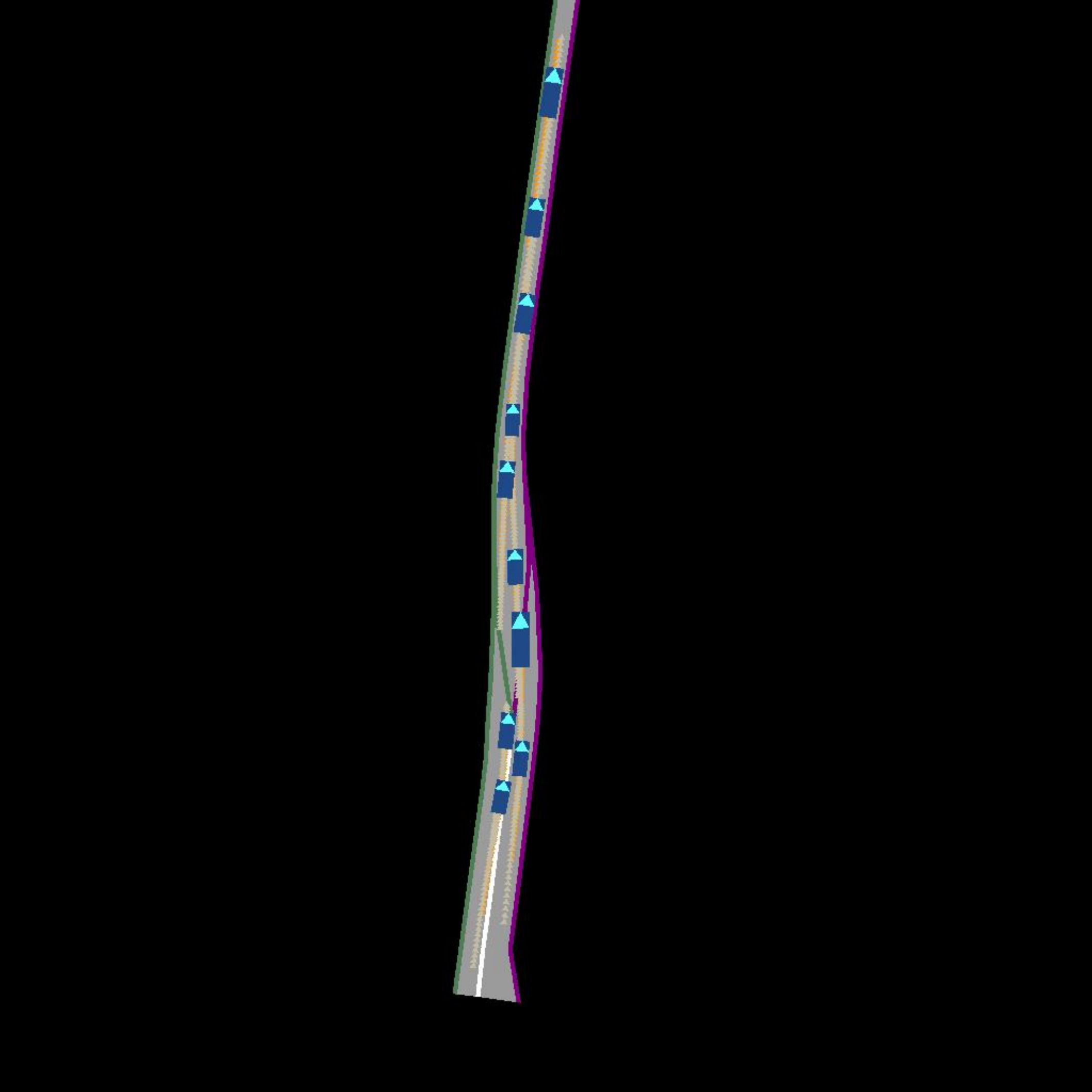}
        \caption*{\method{}}
        \label{}
    \end{subfigure}
    \hfill
    \begin{subfigure}[b]{0.24\textwidth}
        \centering
        \includegraphics[width=1\textwidth]{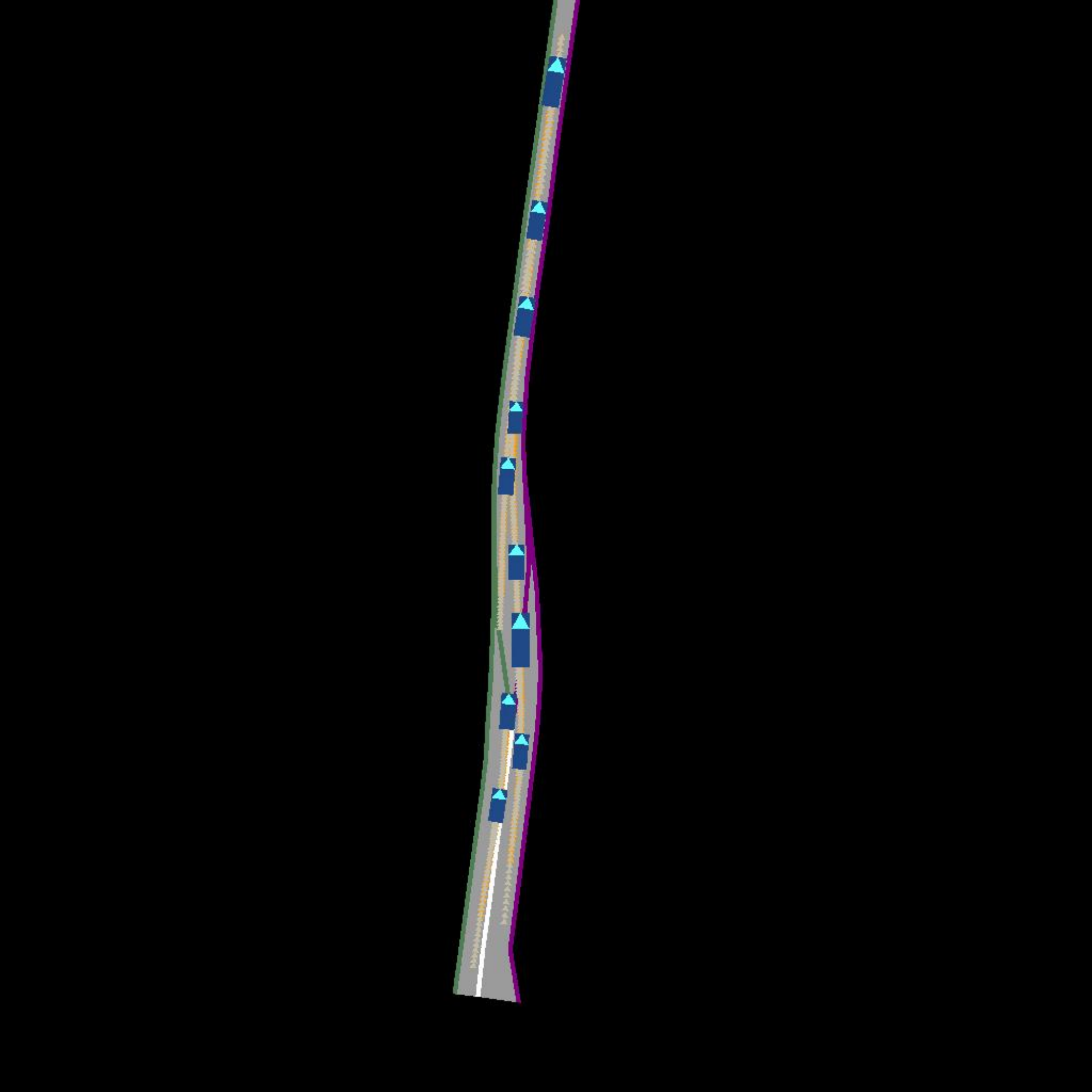}
        \caption*{\method{} (staged)}
        \label{}
    \end{subfigure}
    
\end{figure*}

\begin{figure*}[!h]
    \centering
    \begin{subfigure}[b]{0.24\textwidth}
        \centering
       \includegraphics[width=1\textwidth]{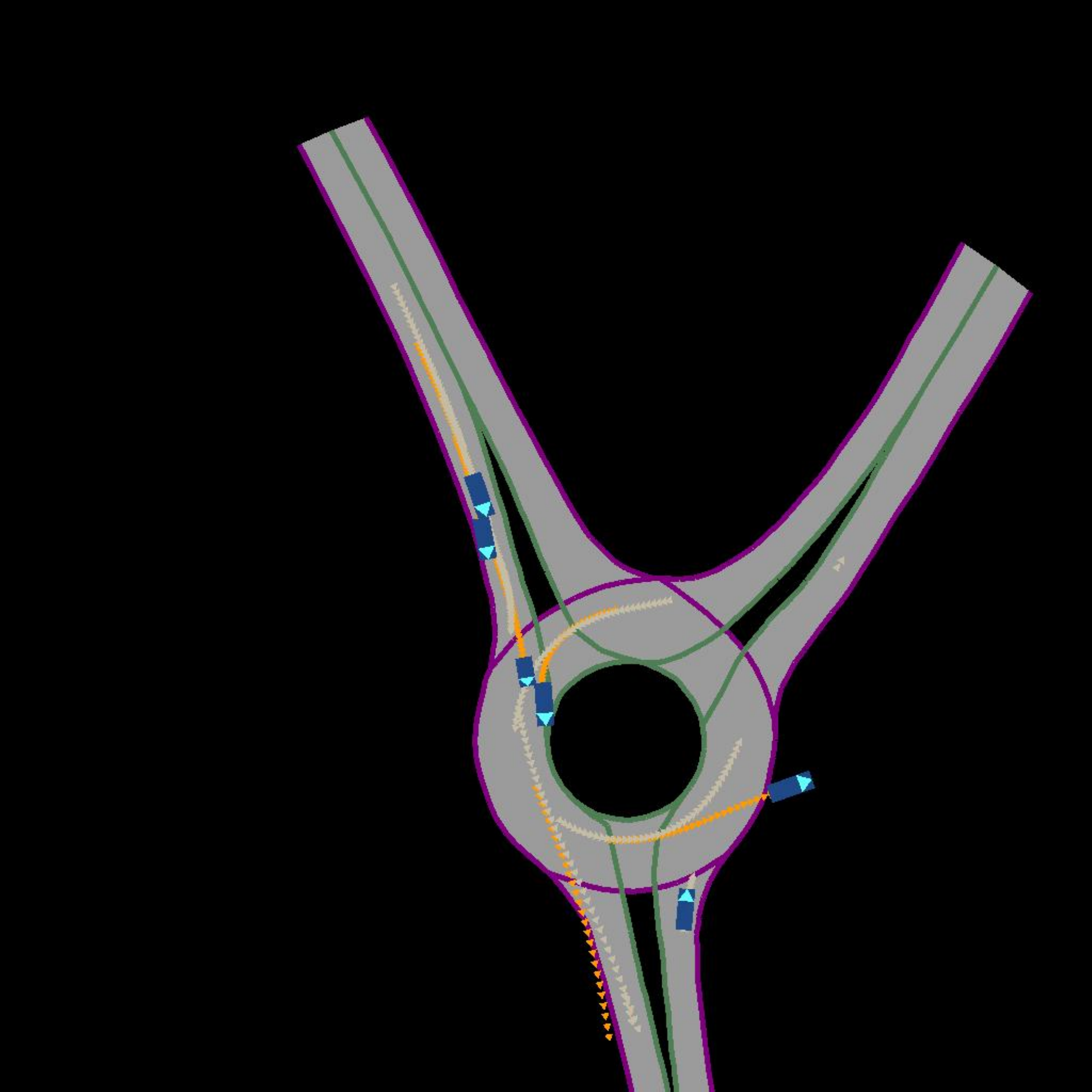}
       \caption*{Standard Diffusion}
        \label{}
    \end{subfigure}
    \hfill
    \begin{subfigure}[b]{0.24\textwidth}
        \centering
        \includegraphics[width=1\textwidth]{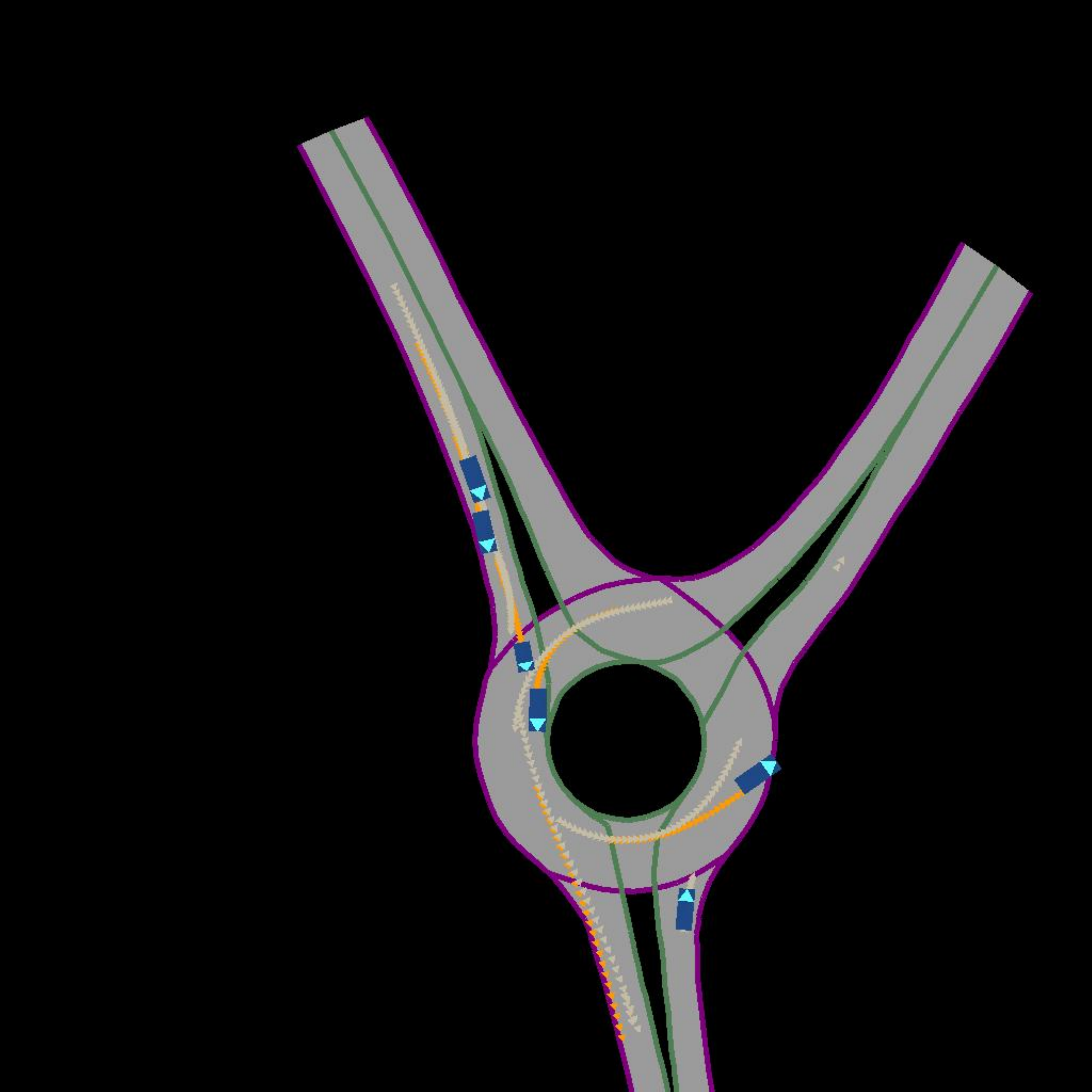}
        \caption*{MPGD w/o projection}
        \label{}
    \end{subfigure}
    \hfill
    \begin{subfigure}[b]{0.24\textwidth}
        \centering
        \includegraphics[width=1\textwidth]{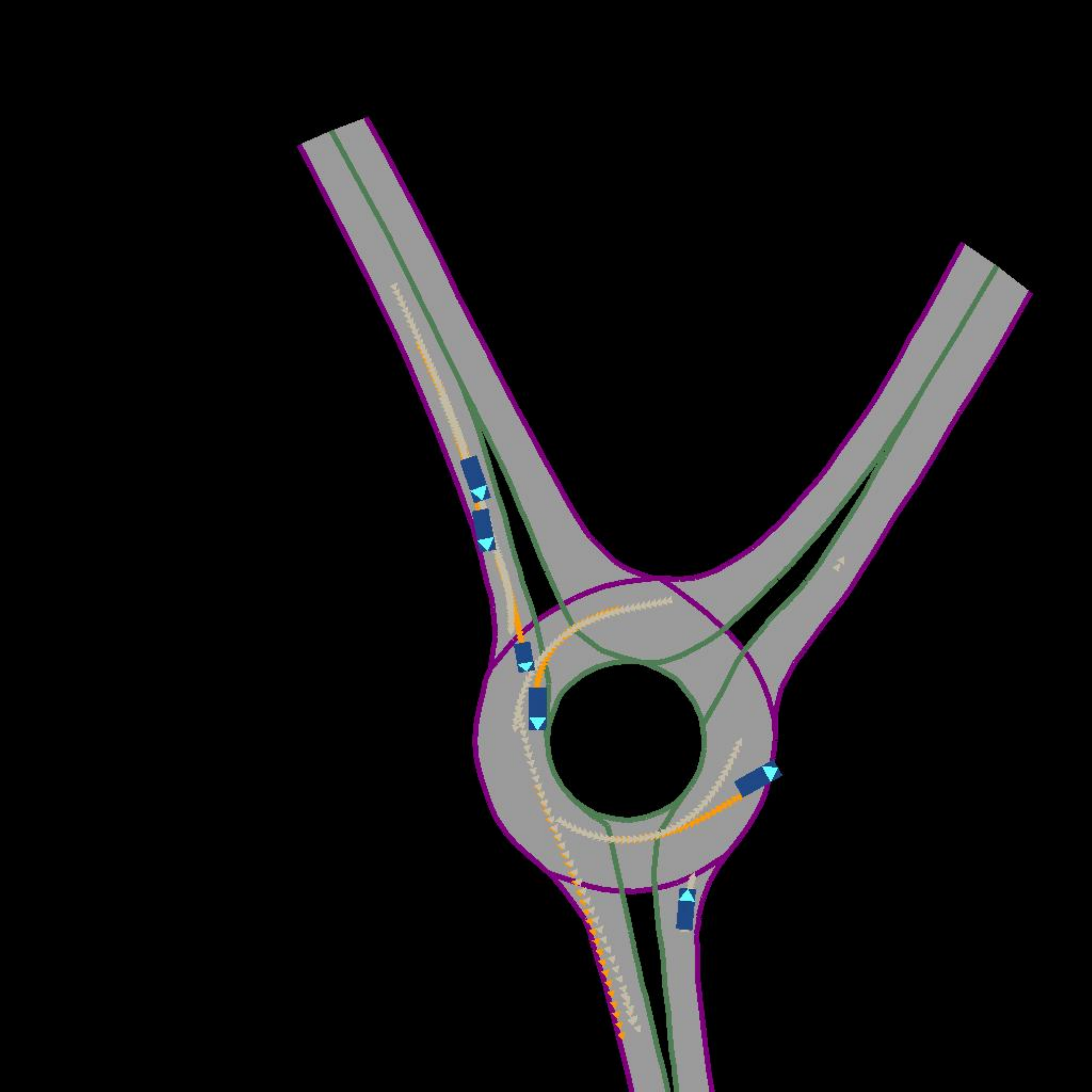}
        \caption*{MBM++}
        \label{}
    \end{subfigure}
    \hfill
    \begin{subfigure}[b]{0.24\textwidth}
        \centering
        \includegraphics[width=1\textwidth]{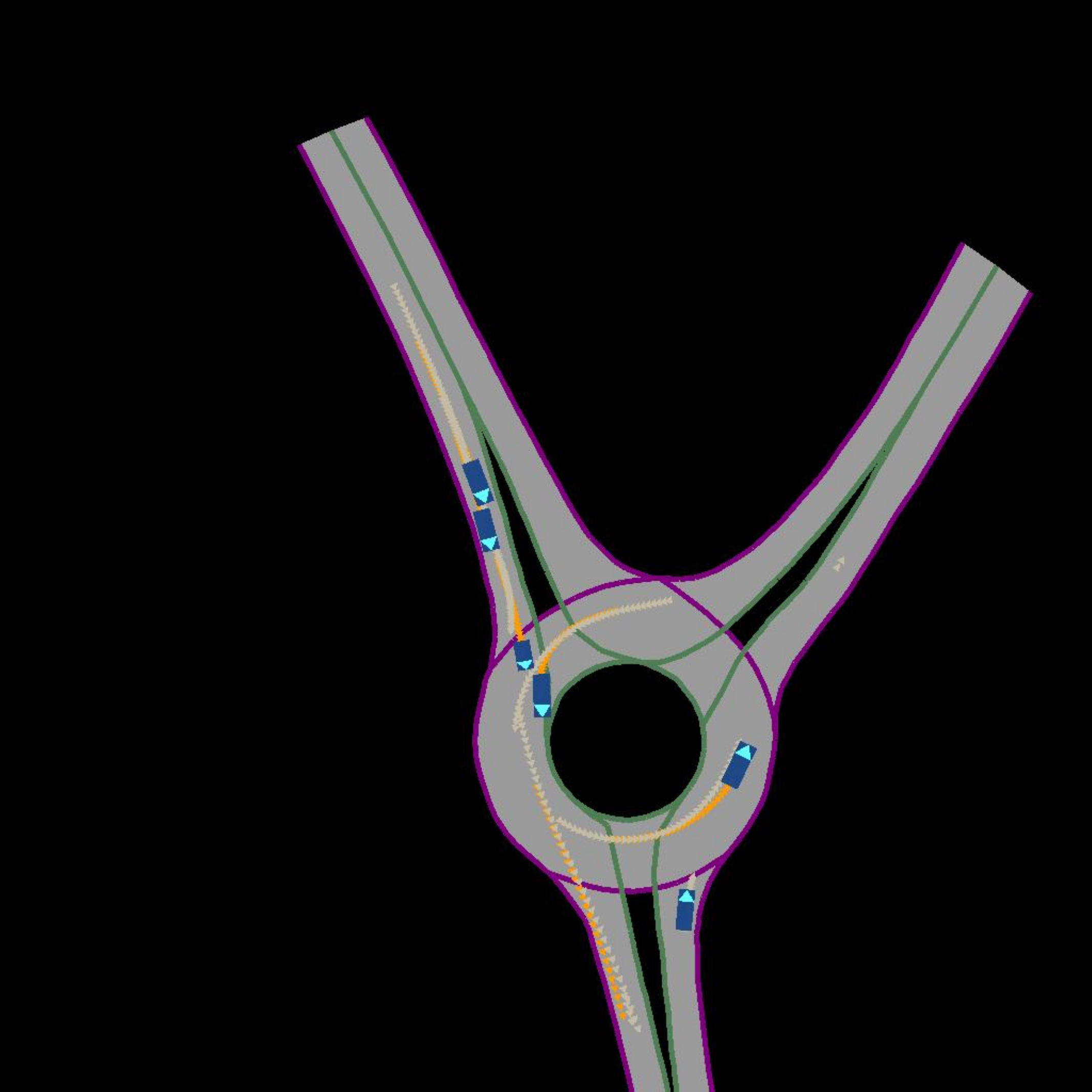}
        \caption*{PIDM}
        \label{}
    \end{subfigure}
    \vspace{0.6em}
    \begin{subfigure}[b]{0.24\textwidth}
        \centering
       \includegraphics[width=1\textwidth]{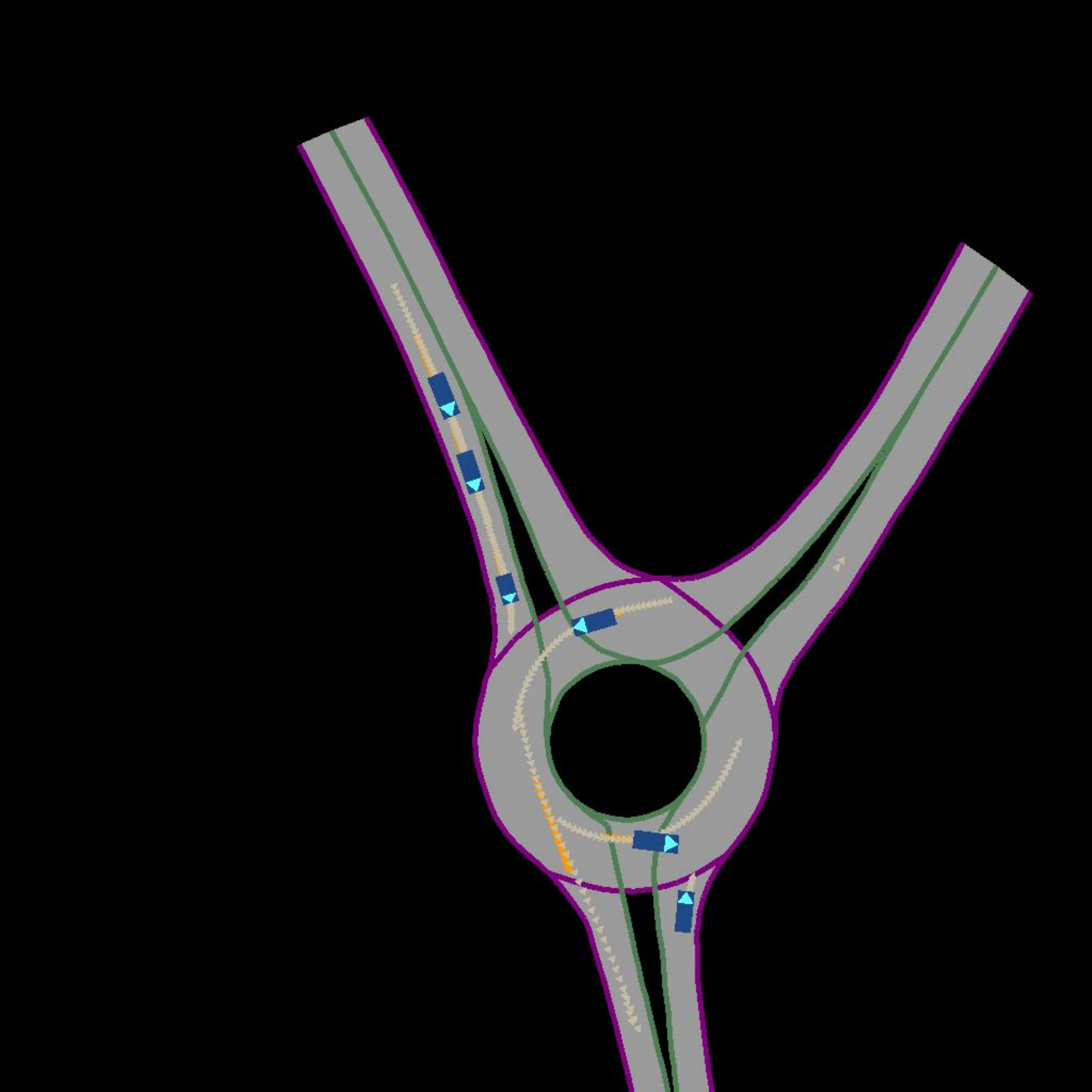}
       \caption*{DPOK}
        \label{}
    \end{subfigure}
    \hfill
    \begin{subfigure}[b]{0.24\textwidth}
        \centering
        \includegraphics[width=1\textwidth]{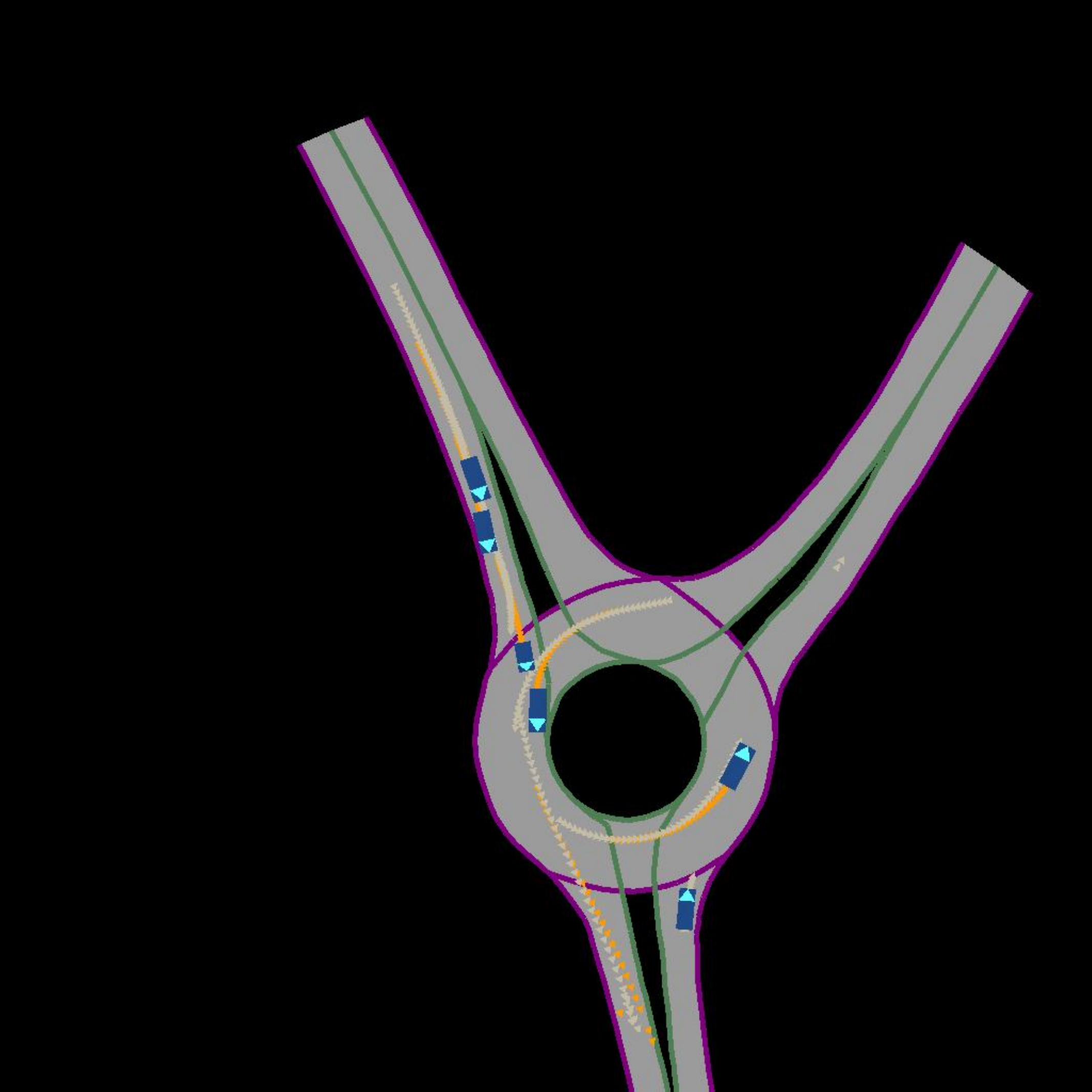}
        \caption*{Rollout-only}
        \label{}
    \end{subfigure}
    \hfill
    \begin{subfigure}[b]{0.24\textwidth}
        \centering
        \includegraphics[width=1\textwidth]{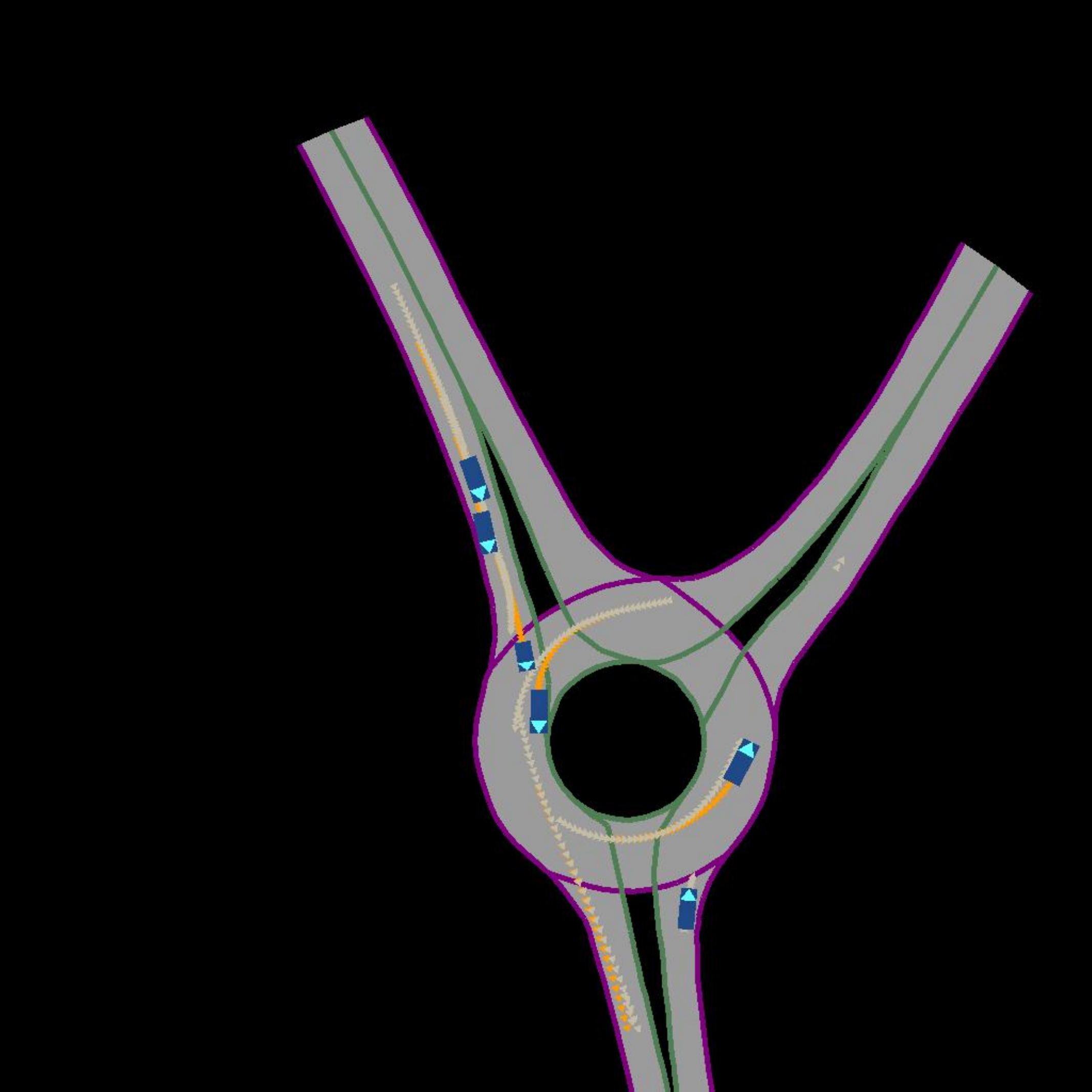}
        \caption*{\method{}}
        \label{}
    \end{subfigure}
    \hfill
    \begin{subfigure}[b]{0.24\textwidth}
        \centering
        \includegraphics[width=1\textwidth]{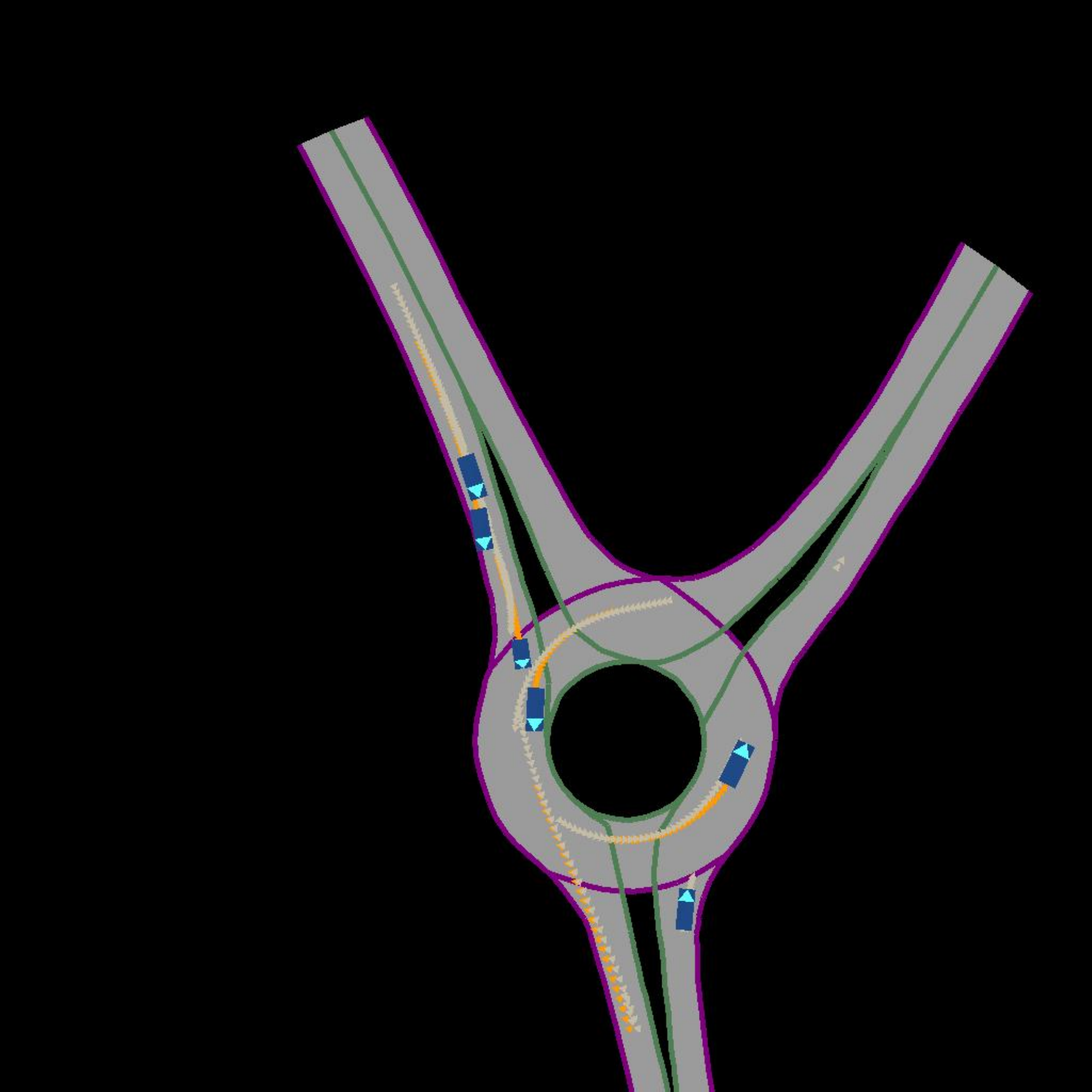}
        \caption*{\method{} (staged)}
        \label{}
    \end{subfigure}
\end{figure*}

\begin{figure*}[!h]
    \centering
    \begin{subfigure}[b]{0.24\textwidth}
        \centering
       \includegraphics[width=1\textwidth]{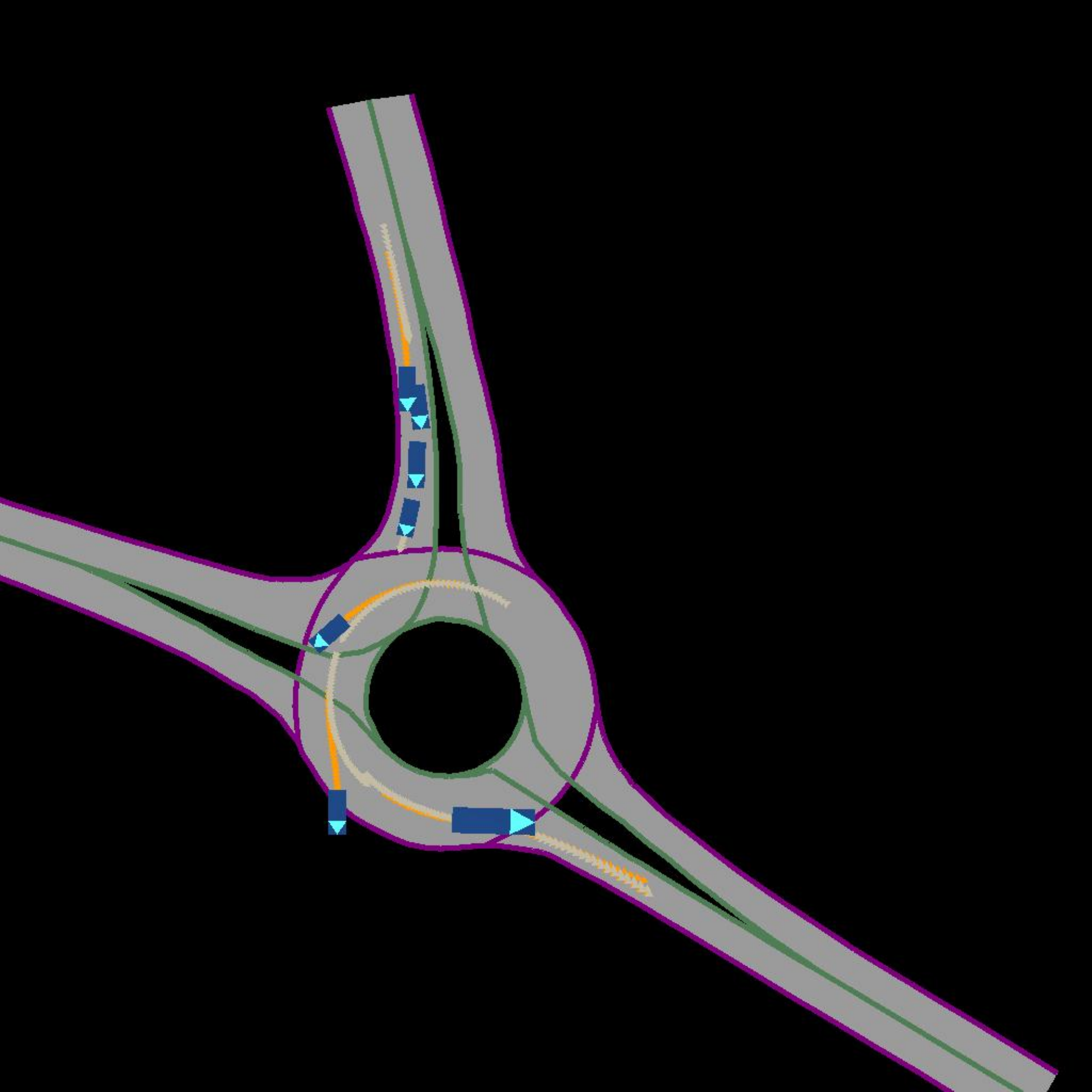}
       \caption*{Standard Diffusion}
        \label{}
    \end{subfigure}
    \hfill
    \begin{subfigure}[b]{0.24\textwidth}
        \centering
        \includegraphics[width=1\textwidth]{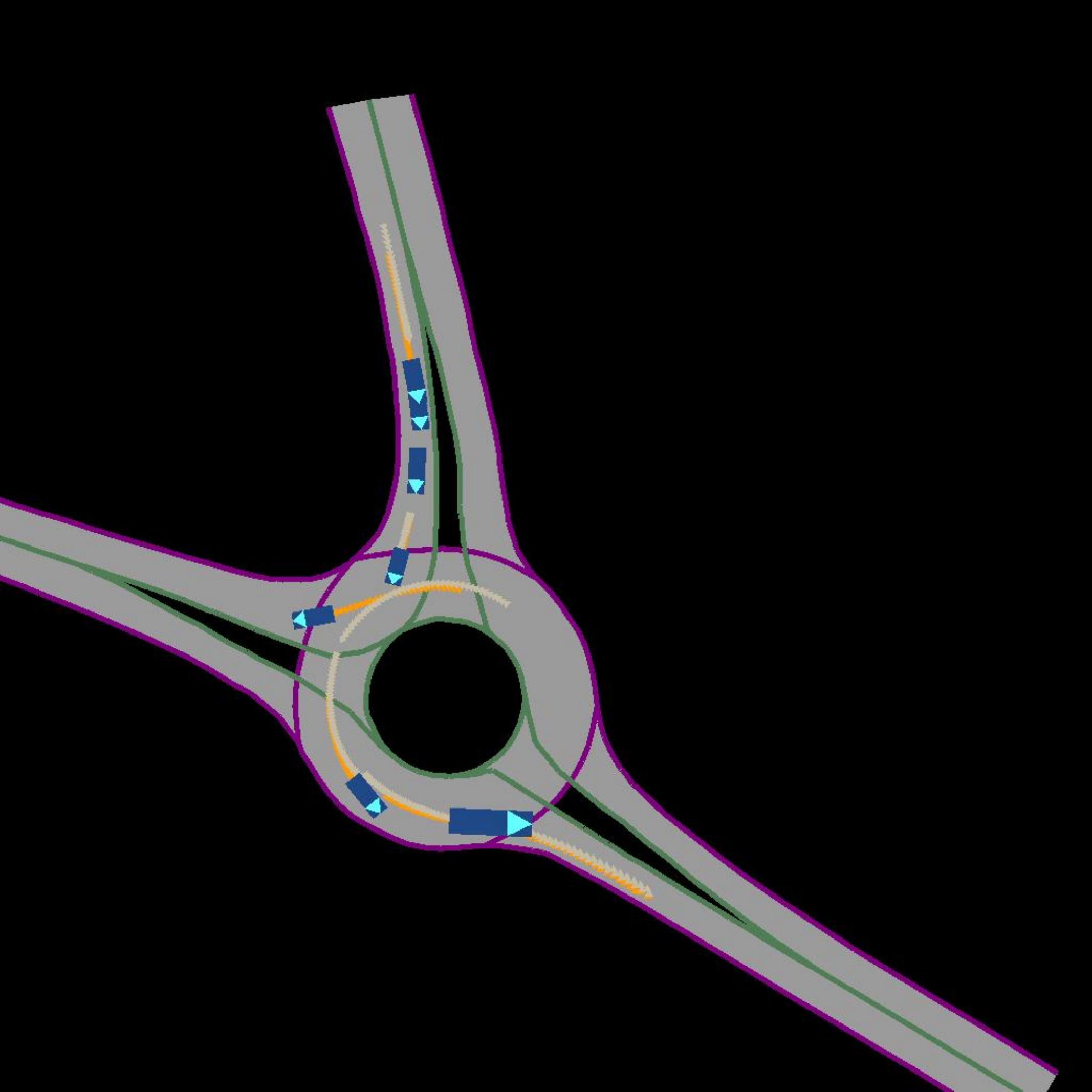}
        \caption*{MPGD w/o projection}
        \label{}
    \end{subfigure}
    \hfill
    \begin{subfigure}[b]{0.24\textwidth}
        \centering
        \includegraphics[width=1\textwidth]{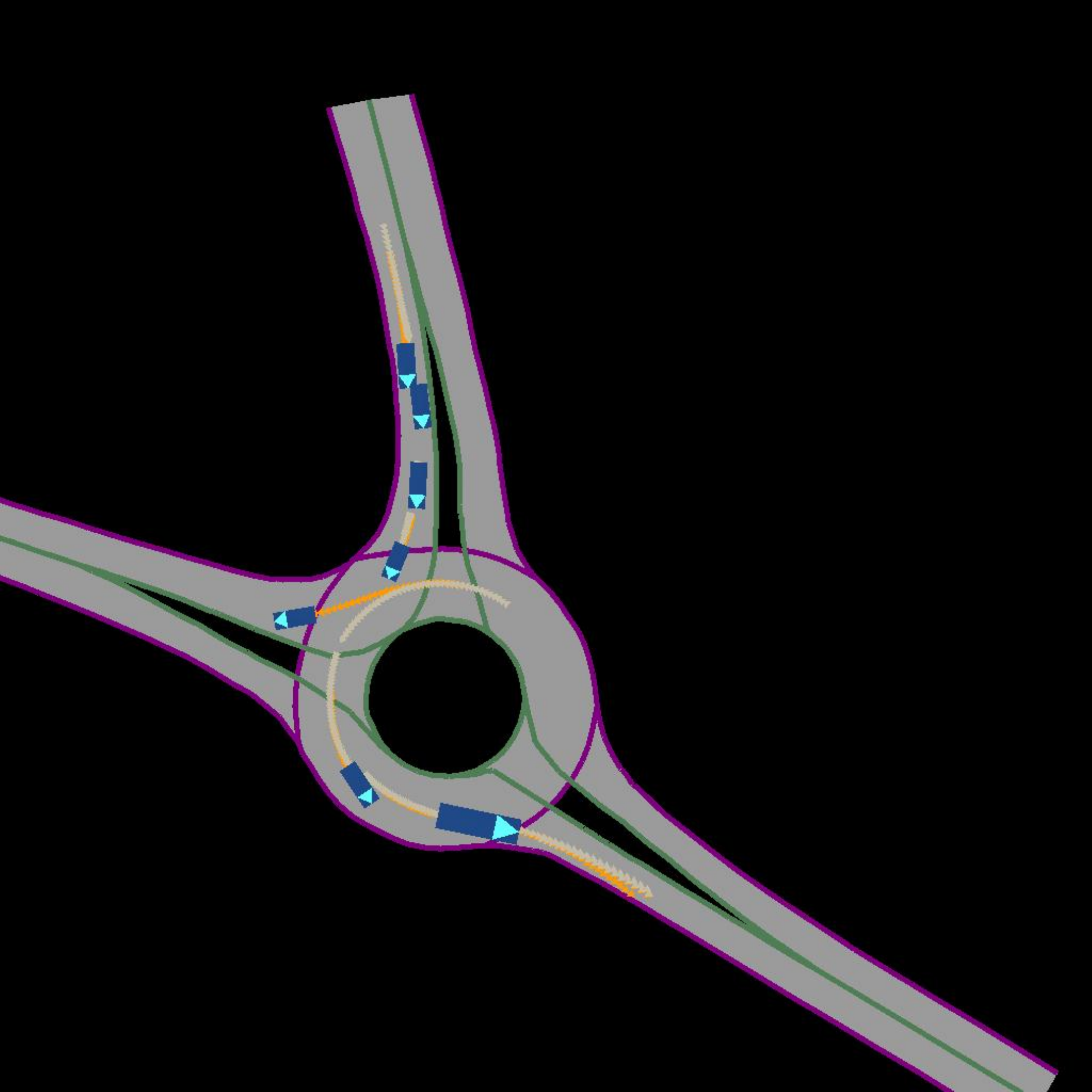}
        \caption*{MBM++}
        \label{}
    \end{subfigure}
    \hfill
    \begin{subfigure}[b]{0.24\textwidth}
        \centering
        \includegraphics[width=1\textwidth]{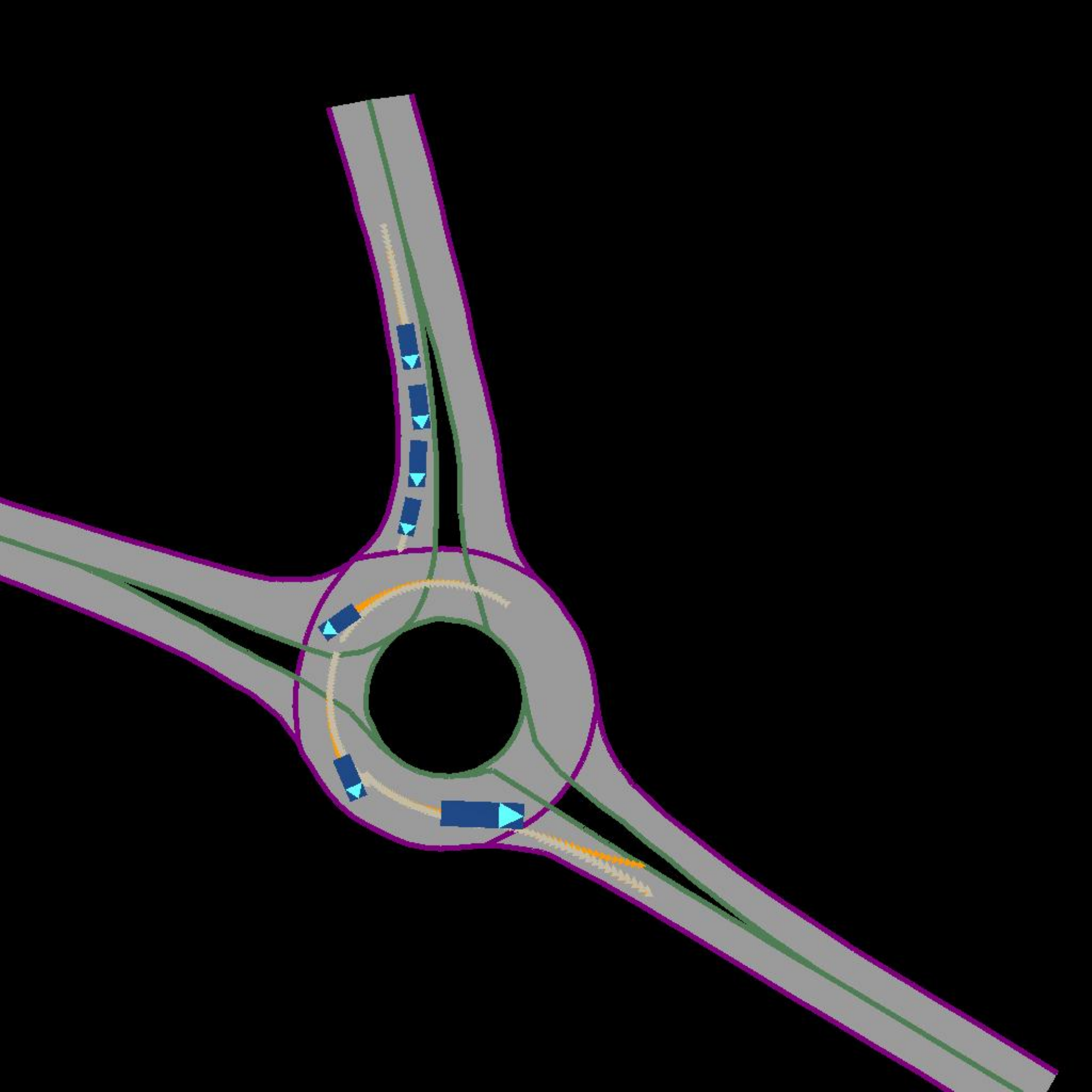}
        \caption*{PIDM}
        \label{}
    \end{subfigure}
    \vspace{0.6em}
    \begin{subfigure}[b]{0.24\textwidth}
        \centering
       \includegraphics[width=1\textwidth]{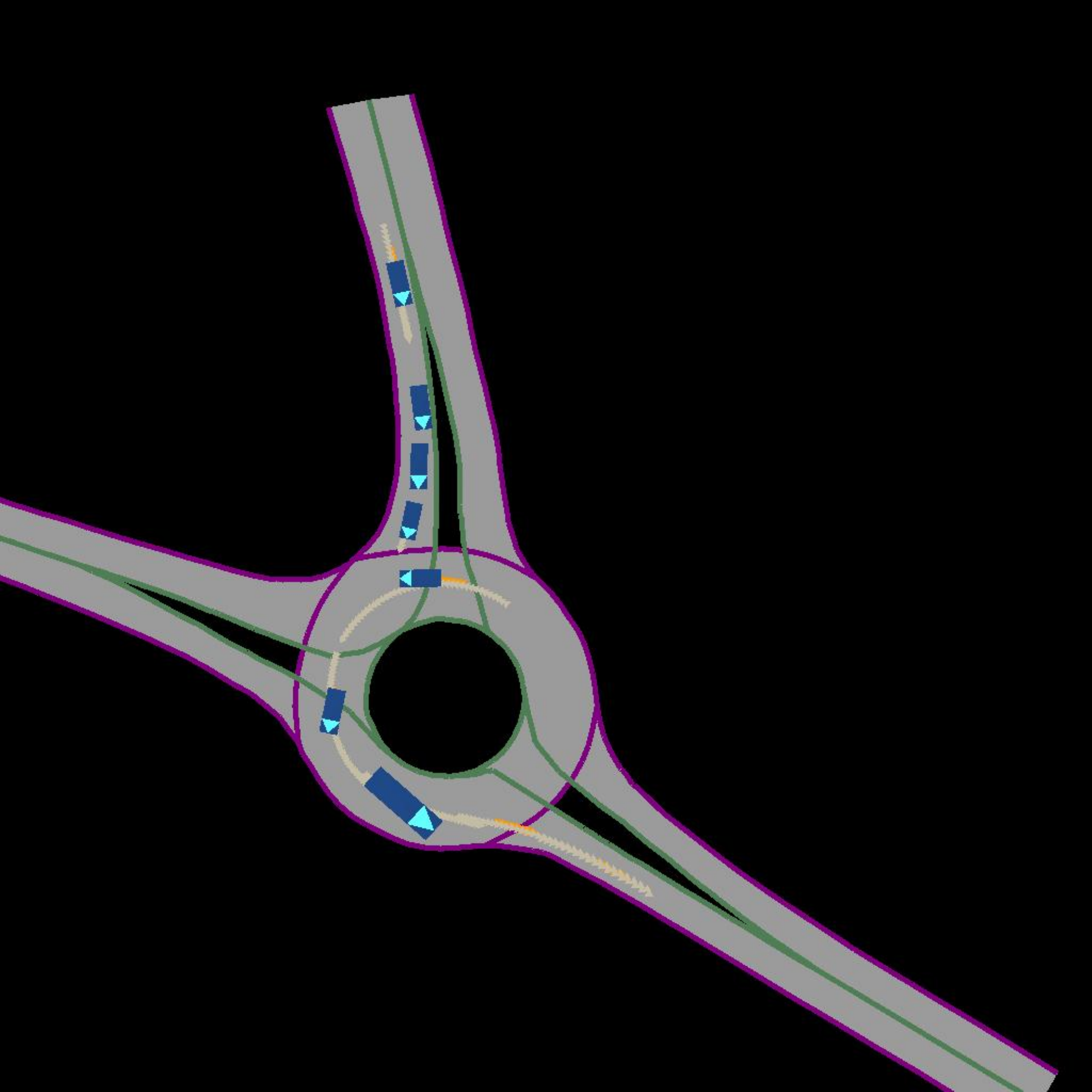}
       \caption*{DPOK}
        \label{}
    \end{subfigure}
    \hfill
    \begin{subfigure}[b]{0.24\textwidth}
        \centering
        \includegraphics[width=1\textwidth]{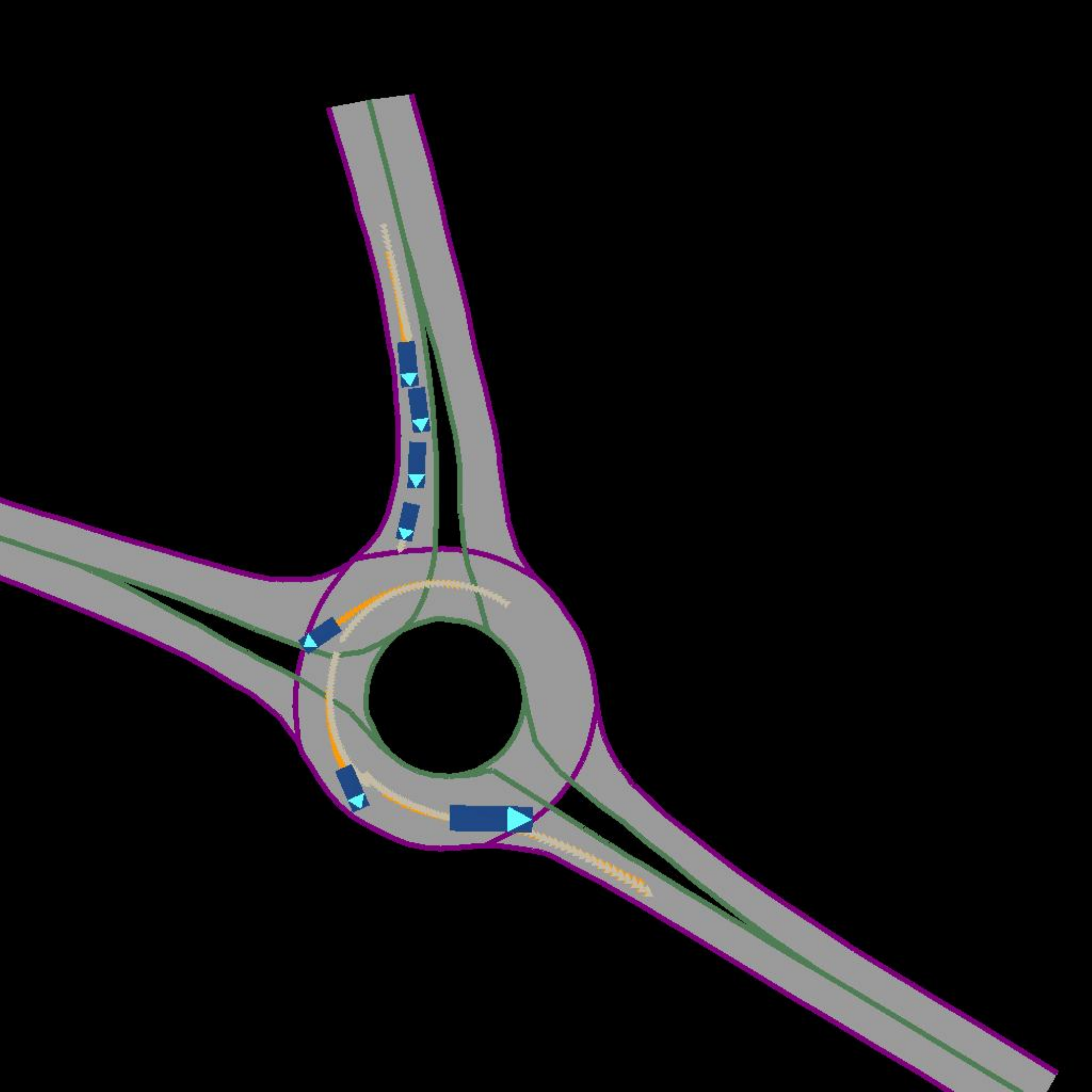}
        \caption*{Rollout-only}
        \label{}
    \end{subfigure}
    \hfill
    \begin{subfigure}[b]{0.24\textwidth}
        \centering
        \includegraphics[width=1\textwidth]{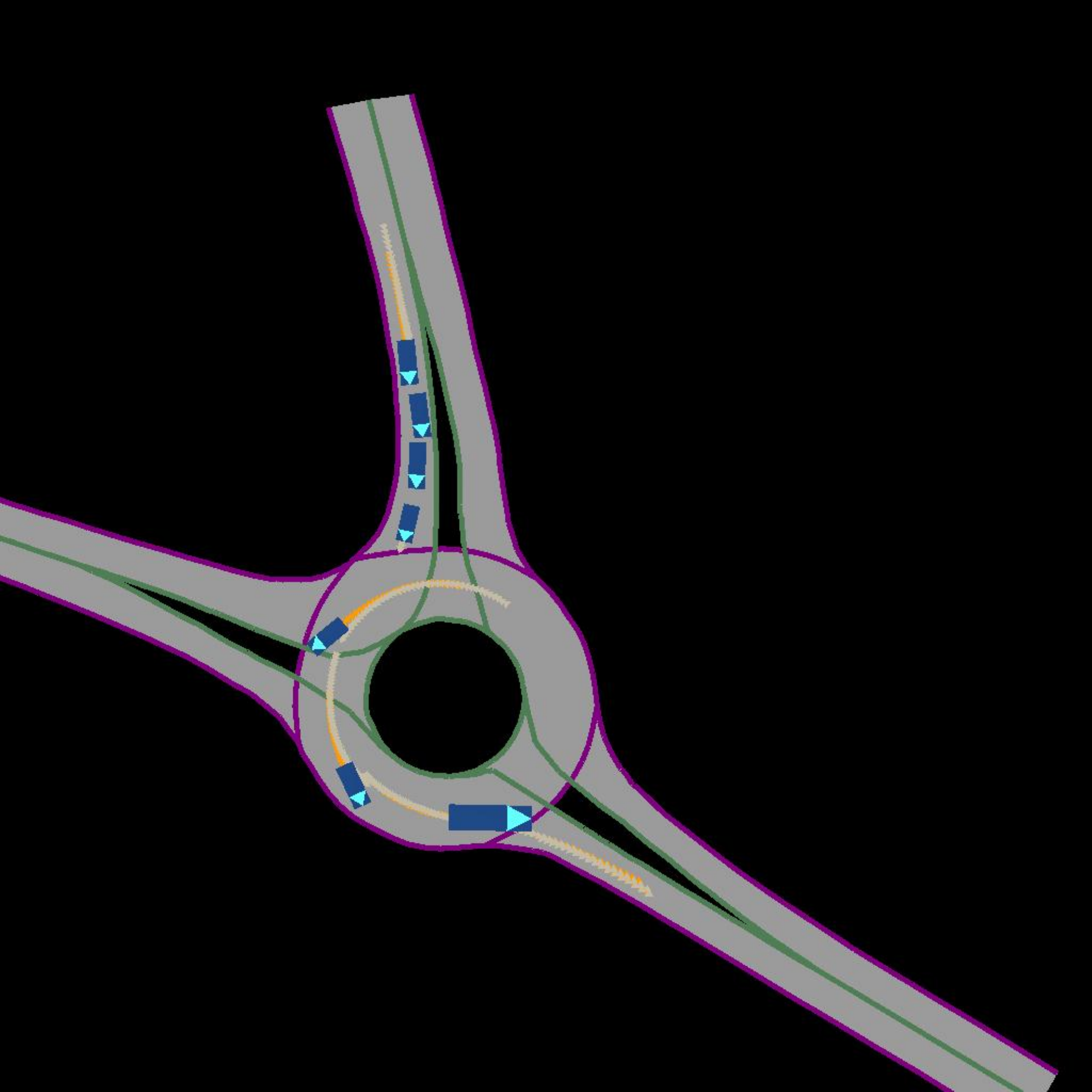}
        \caption*{\method{}}
        \label{}
    \end{subfigure}
    \hfill
    \begin{subfigure}[b]{0.24\textwidth}
        \centering
        \includegraphics[width=1\textwidth]{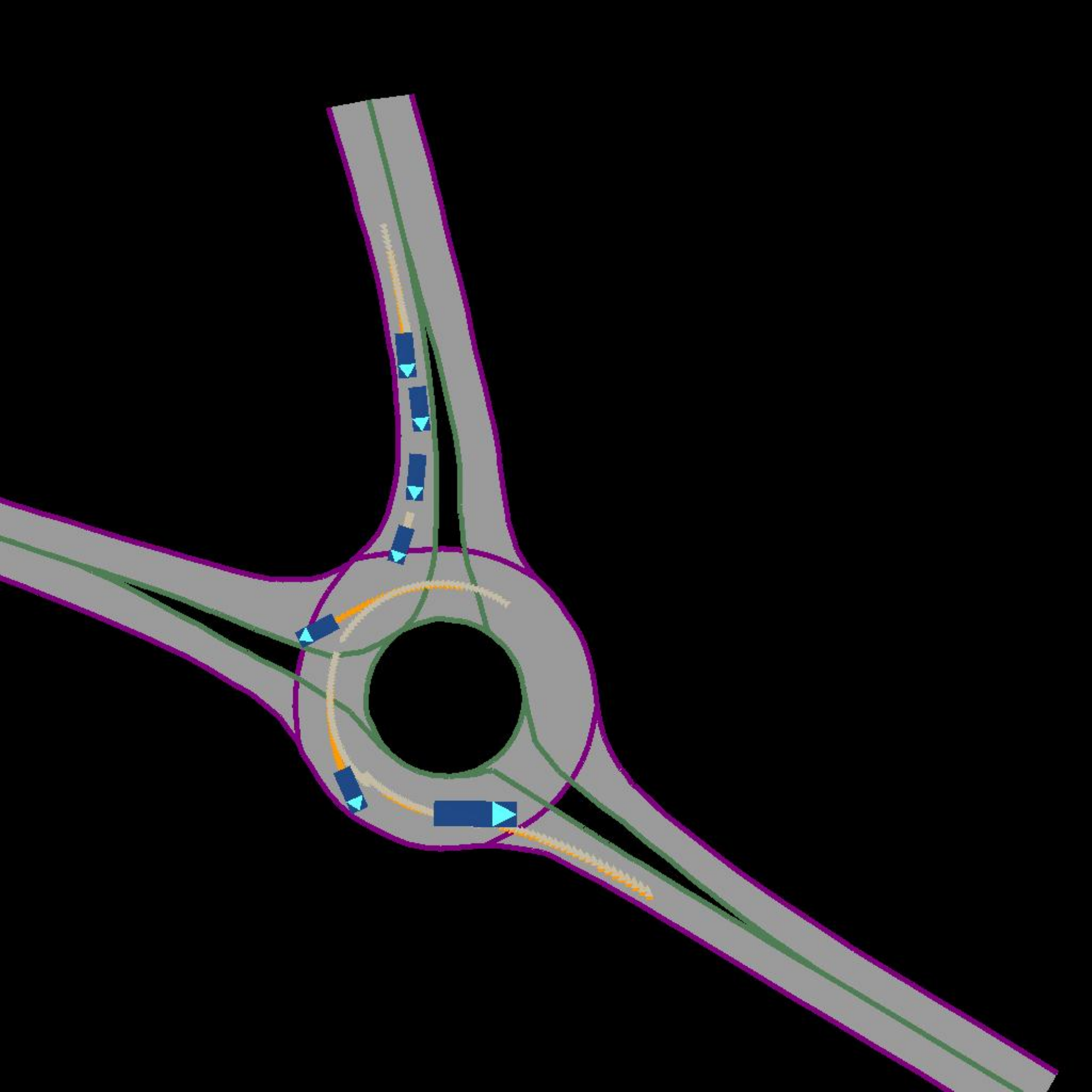}
        \caption*{\method{} (staged)}
        \label{}
    \end{subfigure}
    
\end{figure*}

\begin{figure*}[!h]
    \centering
    \begin{subfigure}[b]{0.24\textwidth}
        \centering
       \includegraphics[width=1\textwidth]{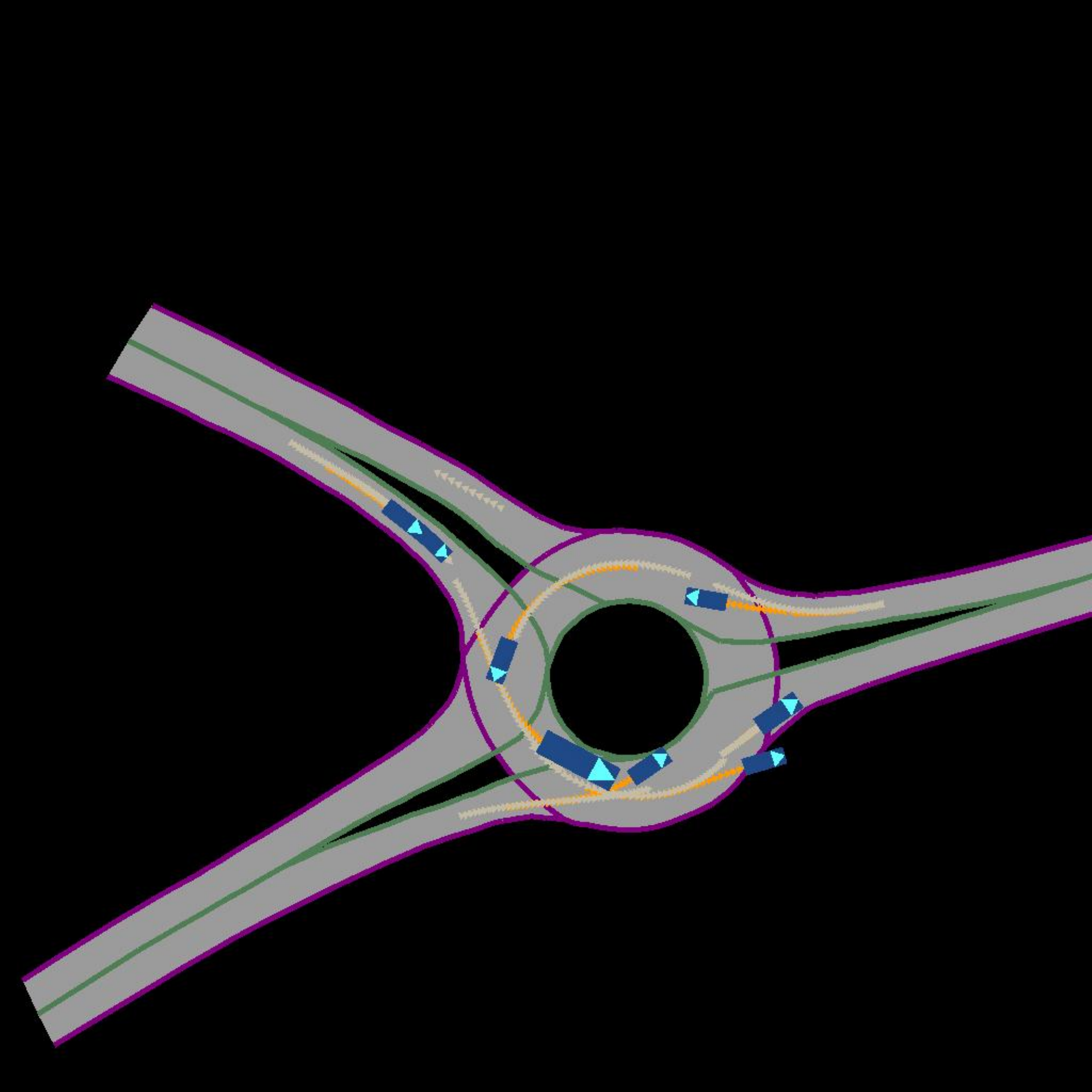}
       \caption*{Standard Diffusion}
        \label{}
    \end{subfigure}
    \hfill
    \begin{subfigure}[b]{0.24\textwidth}
        \centering
        \includegraphics[width=1\textwidth]{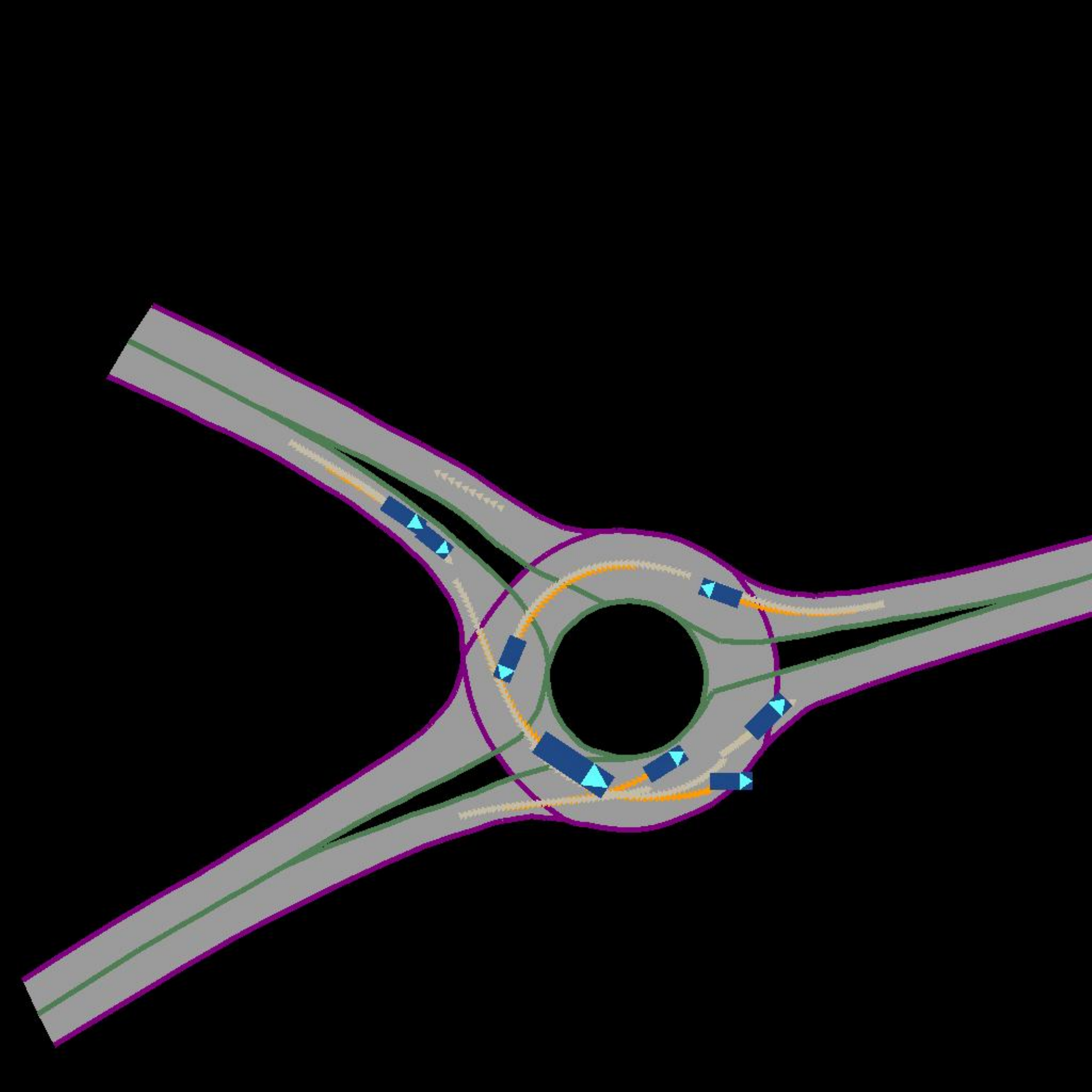}
        \caption*{MPGD w/o projection}
        \label{}
    \end{subfigure}
    \hfill
    \begin{subfigure}[b]{0.24\textwidth}
        \centering
        \includegraphics[width=1\textwidth]{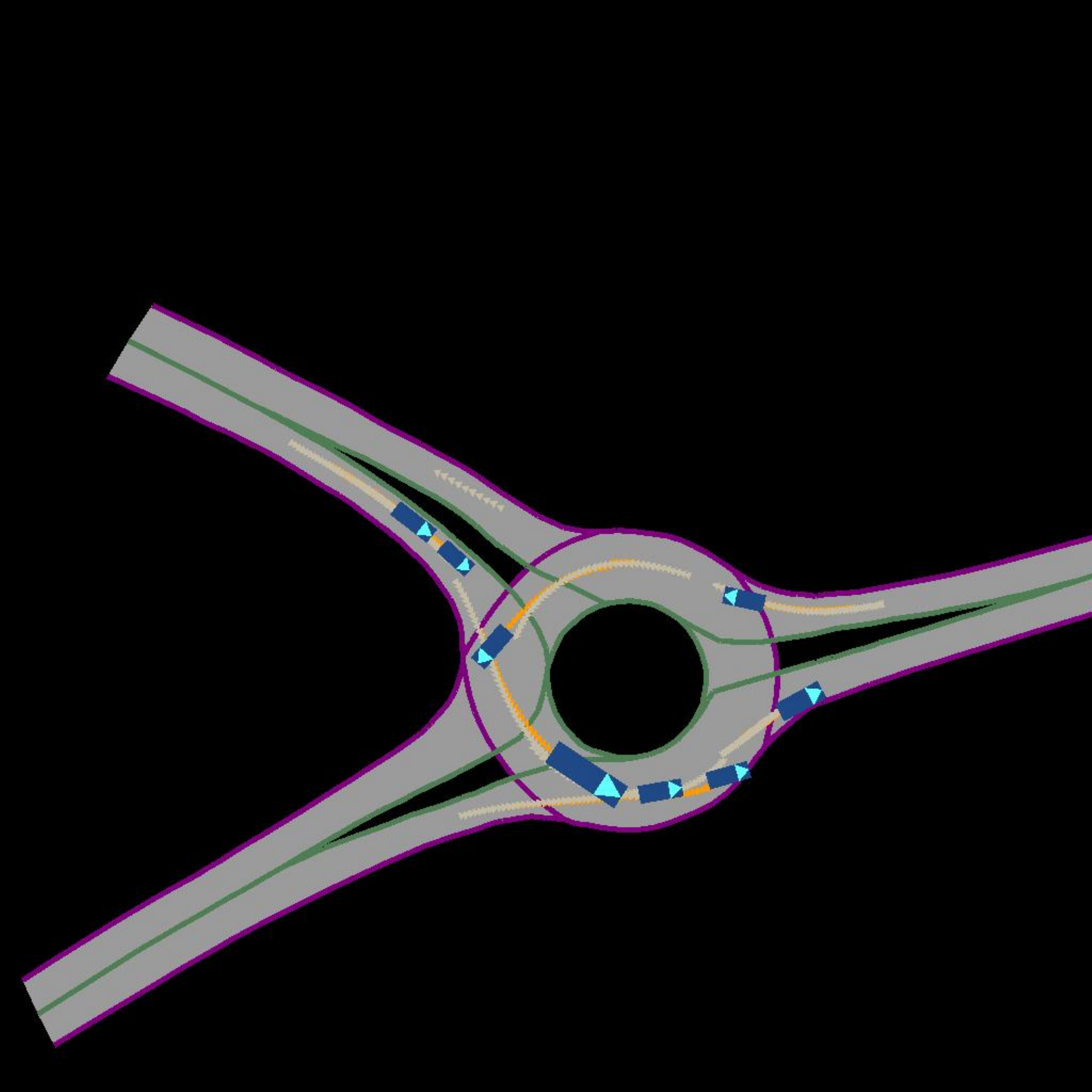}
        \caption*{MBM++}
        \label{}
    \end{subfigure}
    \hfill
    \begin{subfigure}[b]{0.24\textwidth}
        \centering
        \includegraphics[width=1\textwidth]{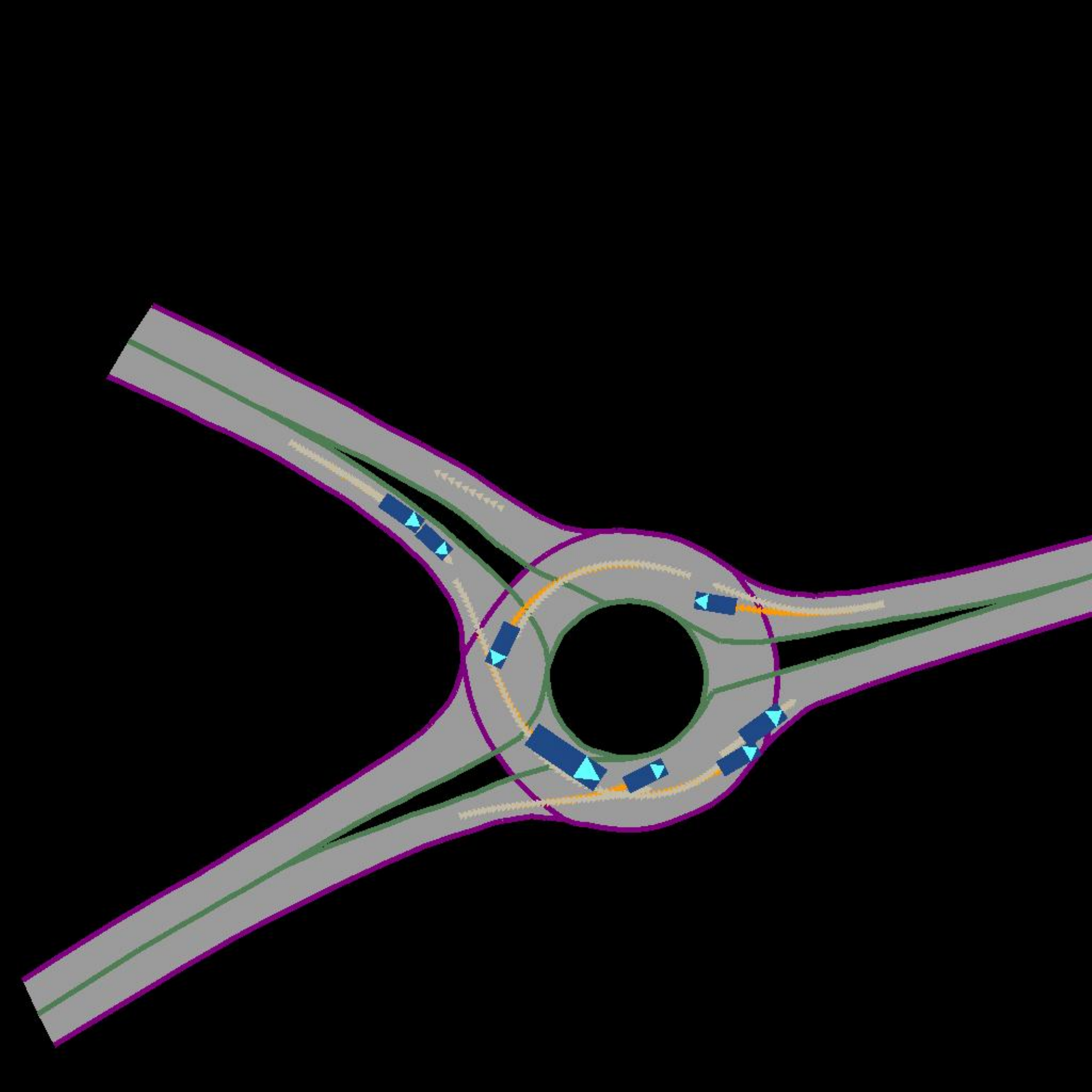}
        \caption*{PIDM}
        \label{}
    \end{subfigure}
    \vspace{0.6em}
    \begin{subfigure}[b]{0.24\textwidth}
        \centering
       \includegraphics[width=1\textwidth]{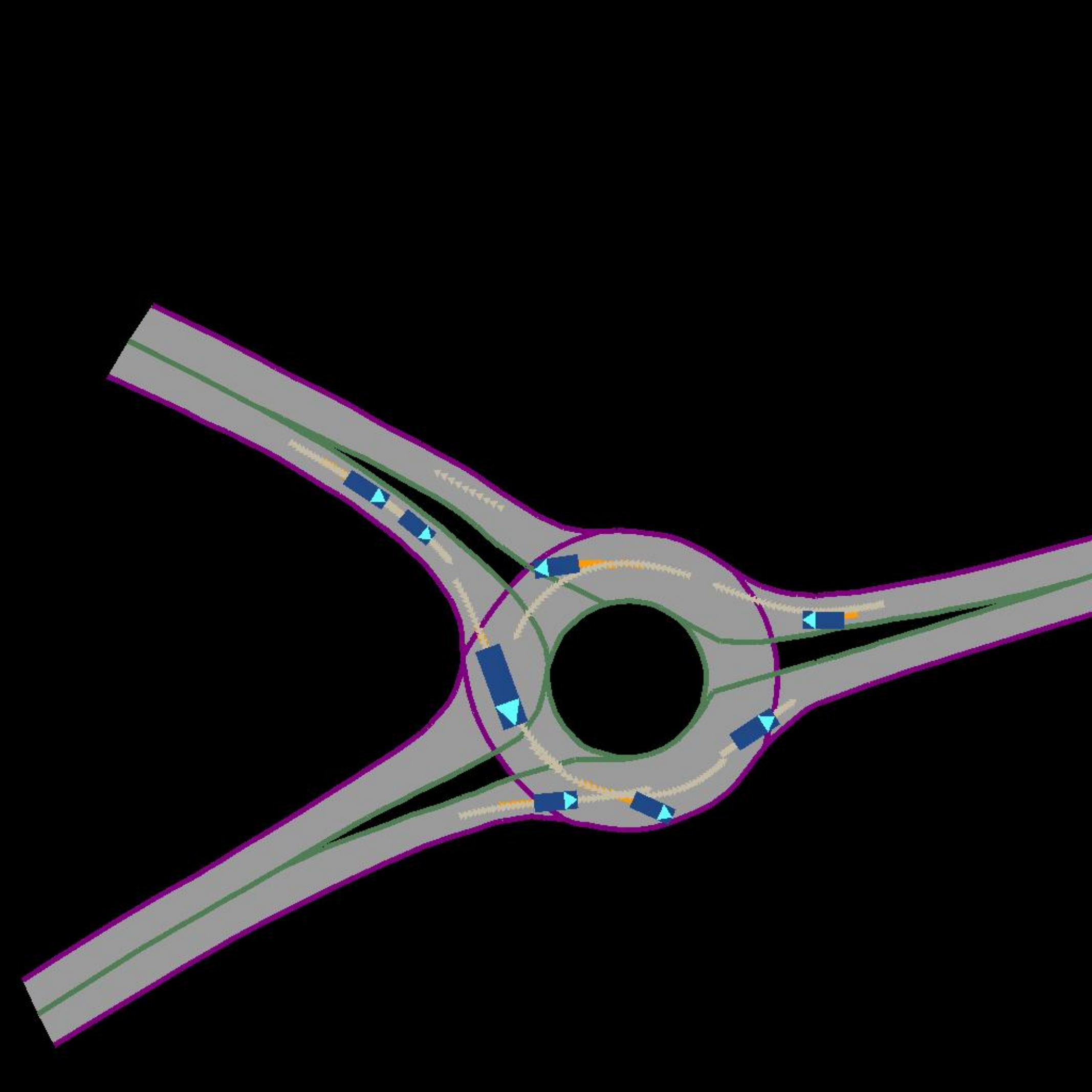}
       \caption*{DPOK}
        \label{}
    \end{subfigure}
    \hfill
    \begin{subfigure}[b]{0.24\textwidth}
        \centering
        \includegraphics[width=1\textwidth]{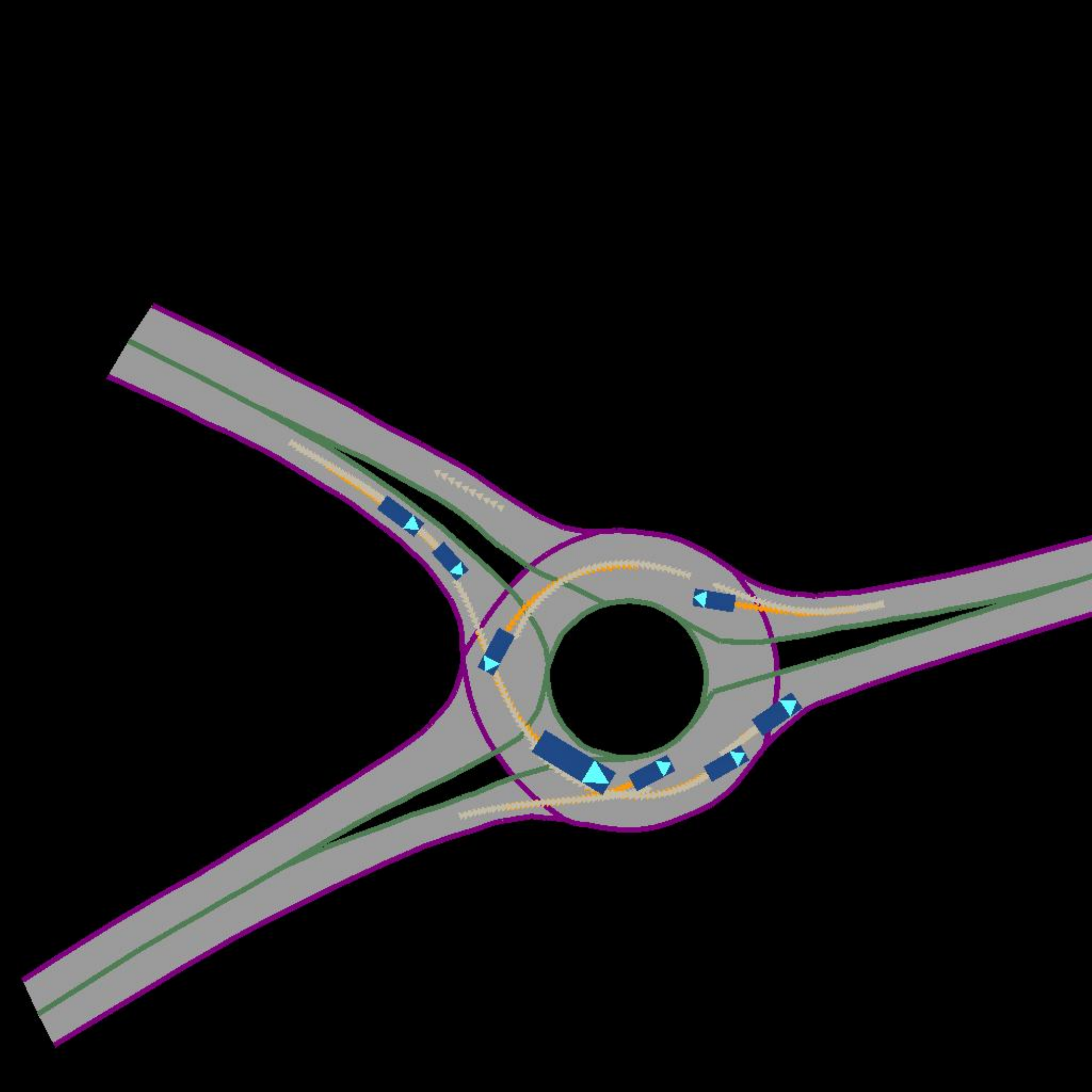}
        \caption*{Rollout-only}
        \label{}
    \end{subfigure}
    \hfill
    \begin{subfigure}[b]{0.24\textwidth}
        \centering
        \includegraphics[width=1\textwidth]{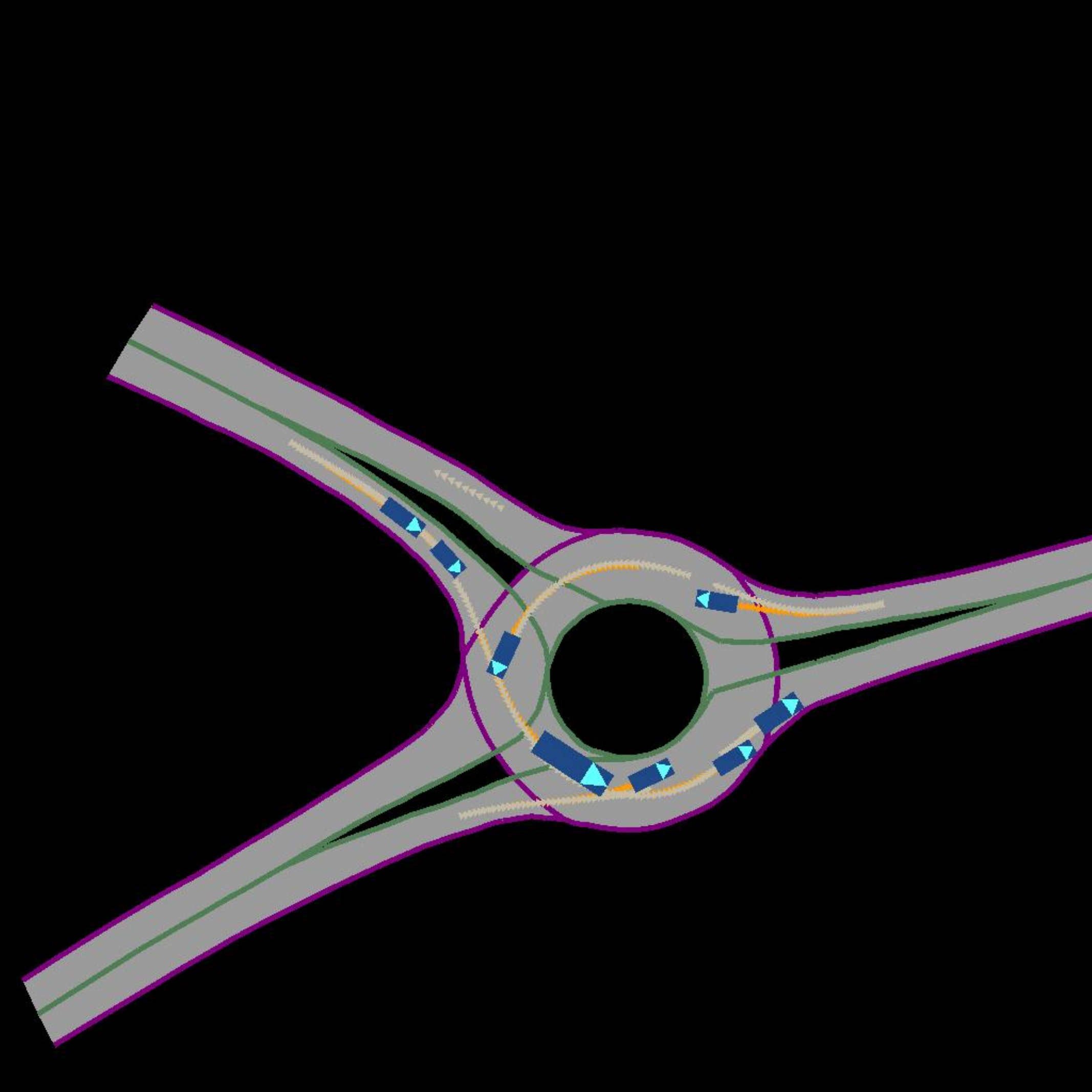
        }
        \caption*{\method{}}
        \label{}
    \end{subfigure}
    \hfill
    \begin{subfigure}[b]{0.24\textwidth}
        \centering
        \includegraphics[width=1\textwidth]{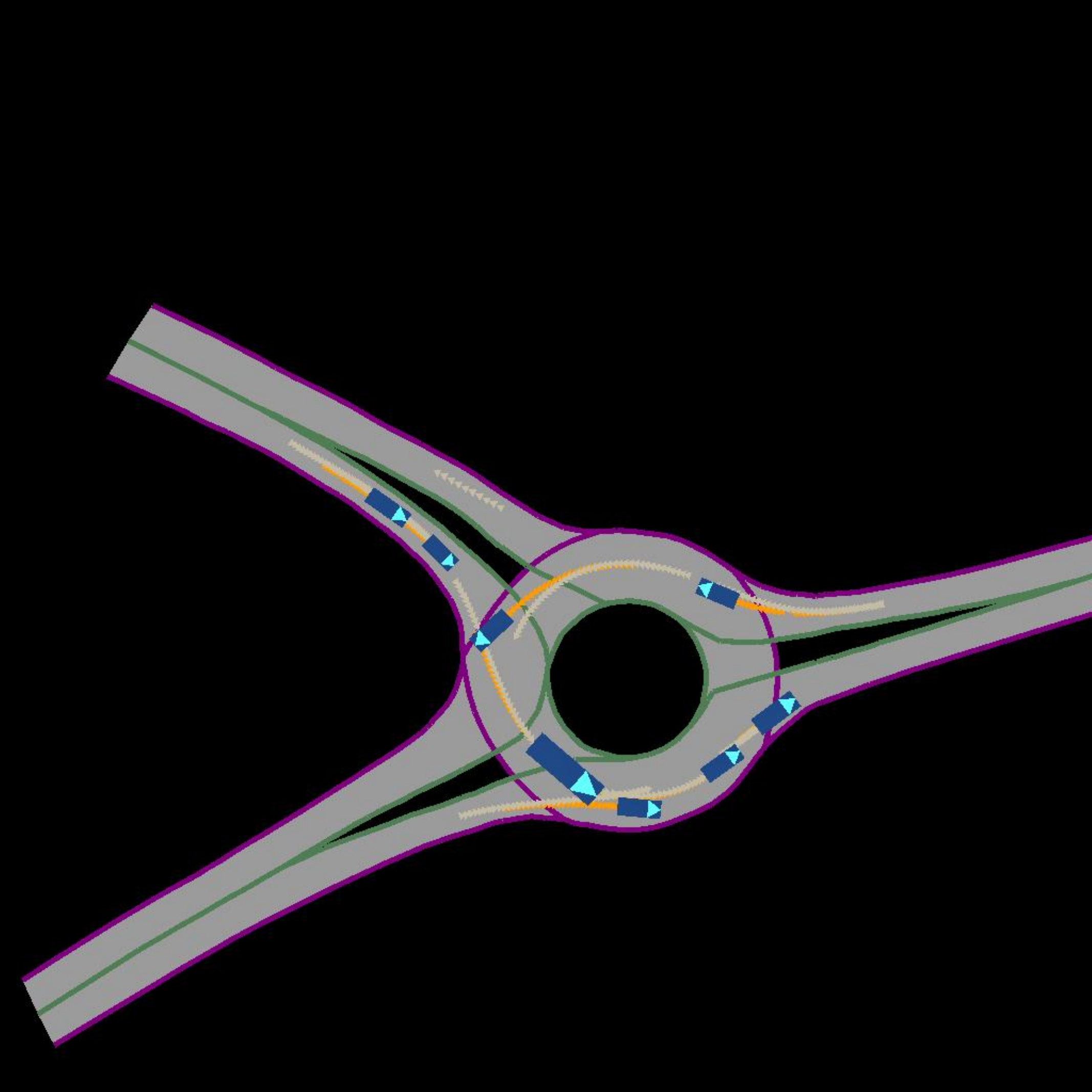}
        \caption*{\method{} (staged)}
        \label{}
    \end{subfigure}
    
\end{figure*}

\begin{figure*}[!h]
    \centering
    \begin{subfigure}[b]{0.24\textwidth}
        \centering
       \includegraphics[width=1\textwidth]{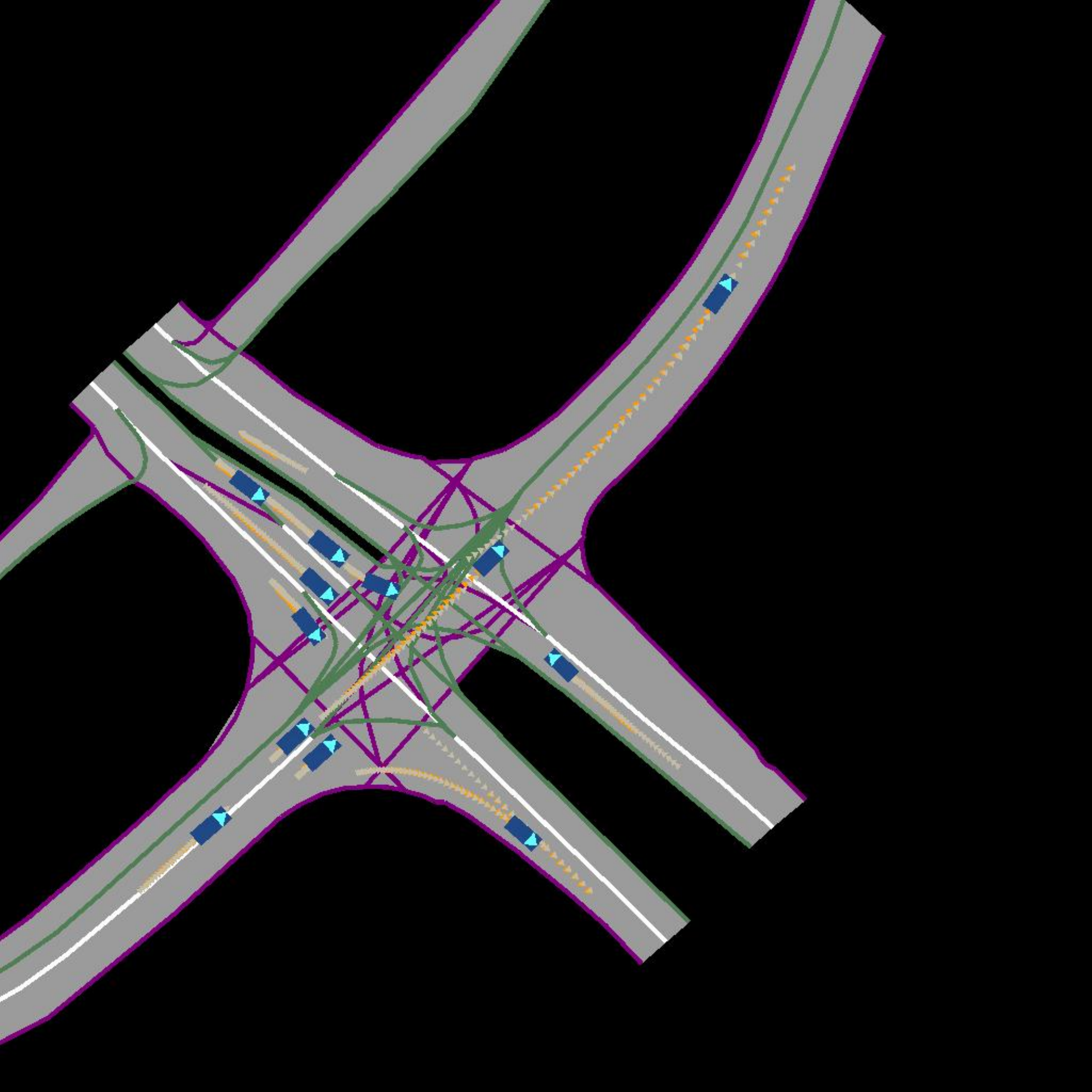}
       \caption*{Standard Diffusion}
        \label{}
    \end{subfigure}
    \hfill
    \begin{subfigure}[b]{0.24\textwidth}
        \centering
        \includegraphics[width=1\textwidth]{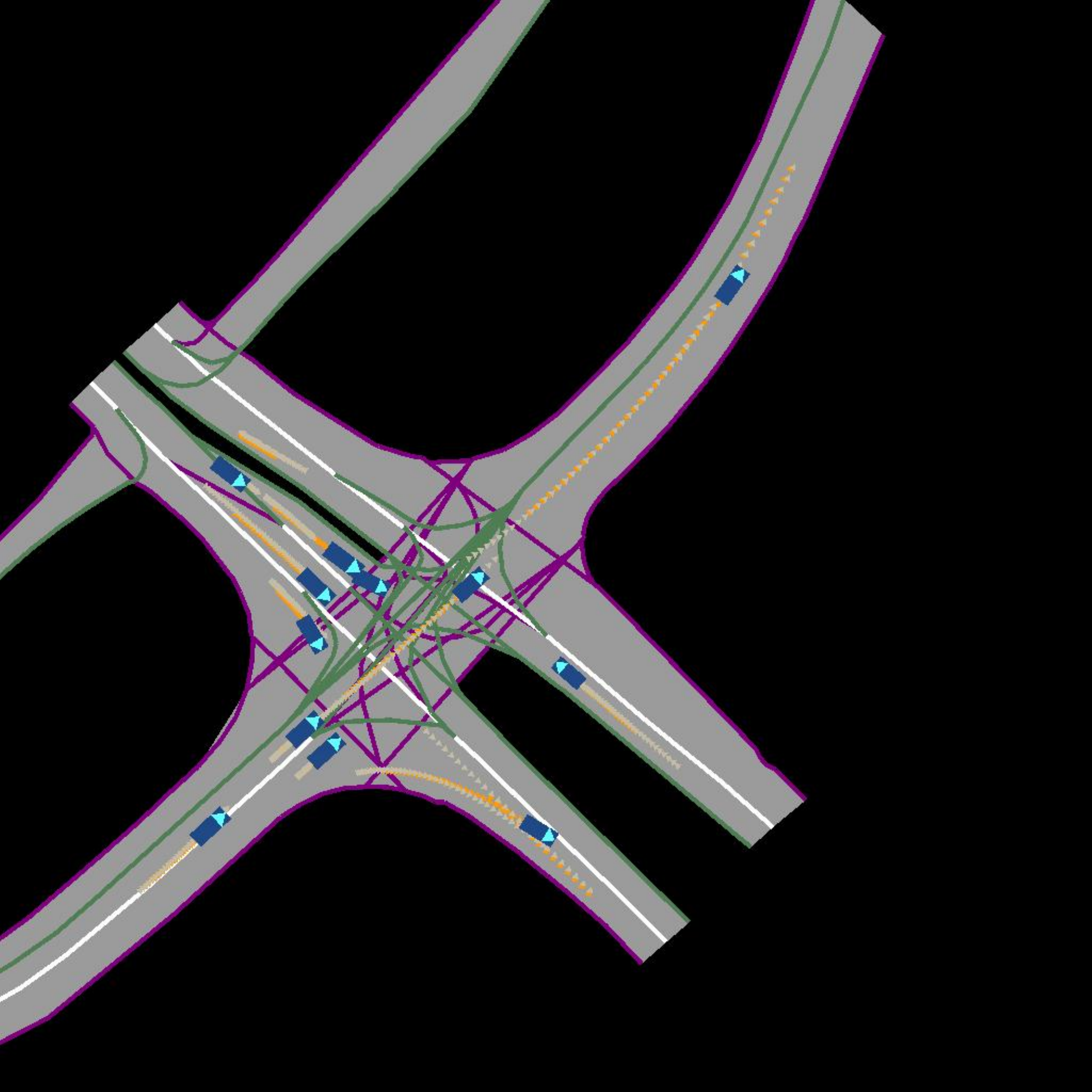}
        \caption*{MPGD w/o projection}
        \label{}
    \end{subfigure}
    \hfill
    \begin{subfigure}[b]{0.24\textwidth}
        \centering
        \includegraphics[width=1\textwidth]{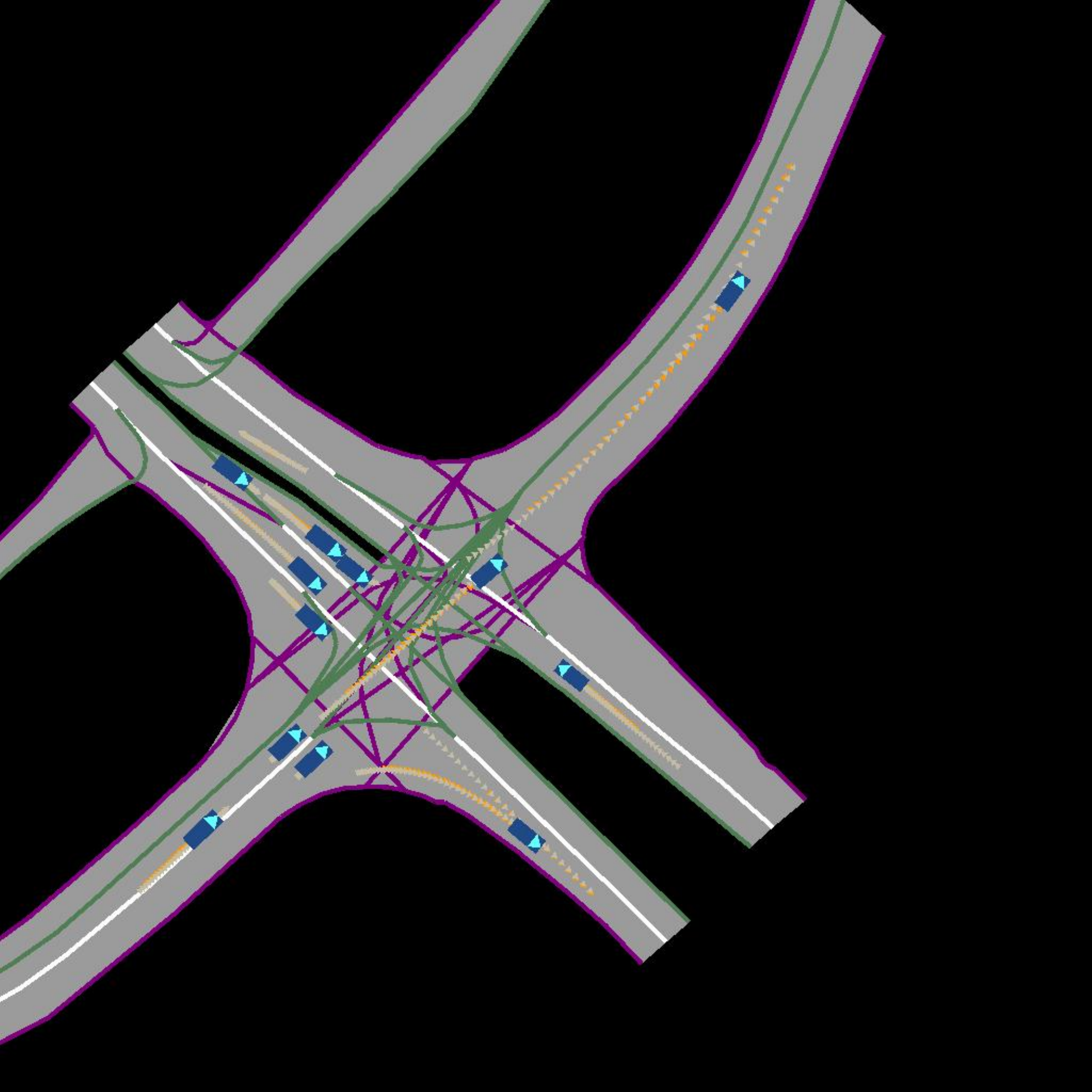}
        \caption*{MBM++}
        \label{}
    \end{subfigure}
    \hfill
    \begin{subfigure}[b]{0.24\textwidth}
        \centering
        \includegraphics[width=1\textwidth]{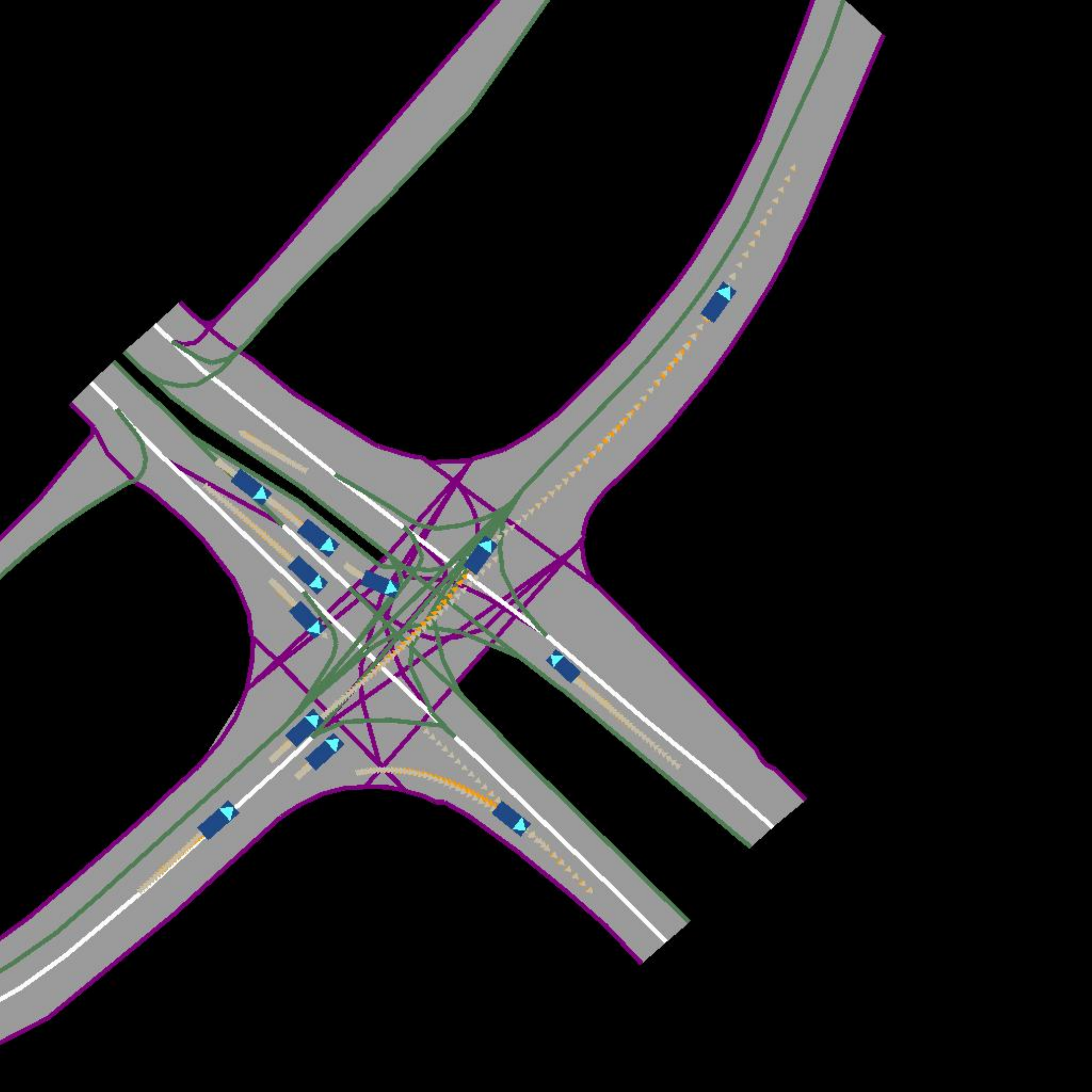}
        \caption*{PIDM}
        \label{}
    \end{subfigure}
    \vspace{0.6em}
    \begin{subfigure}[b]{0.24\textwidth}
        \centering
       \includegraphics[width=1\textwidth]{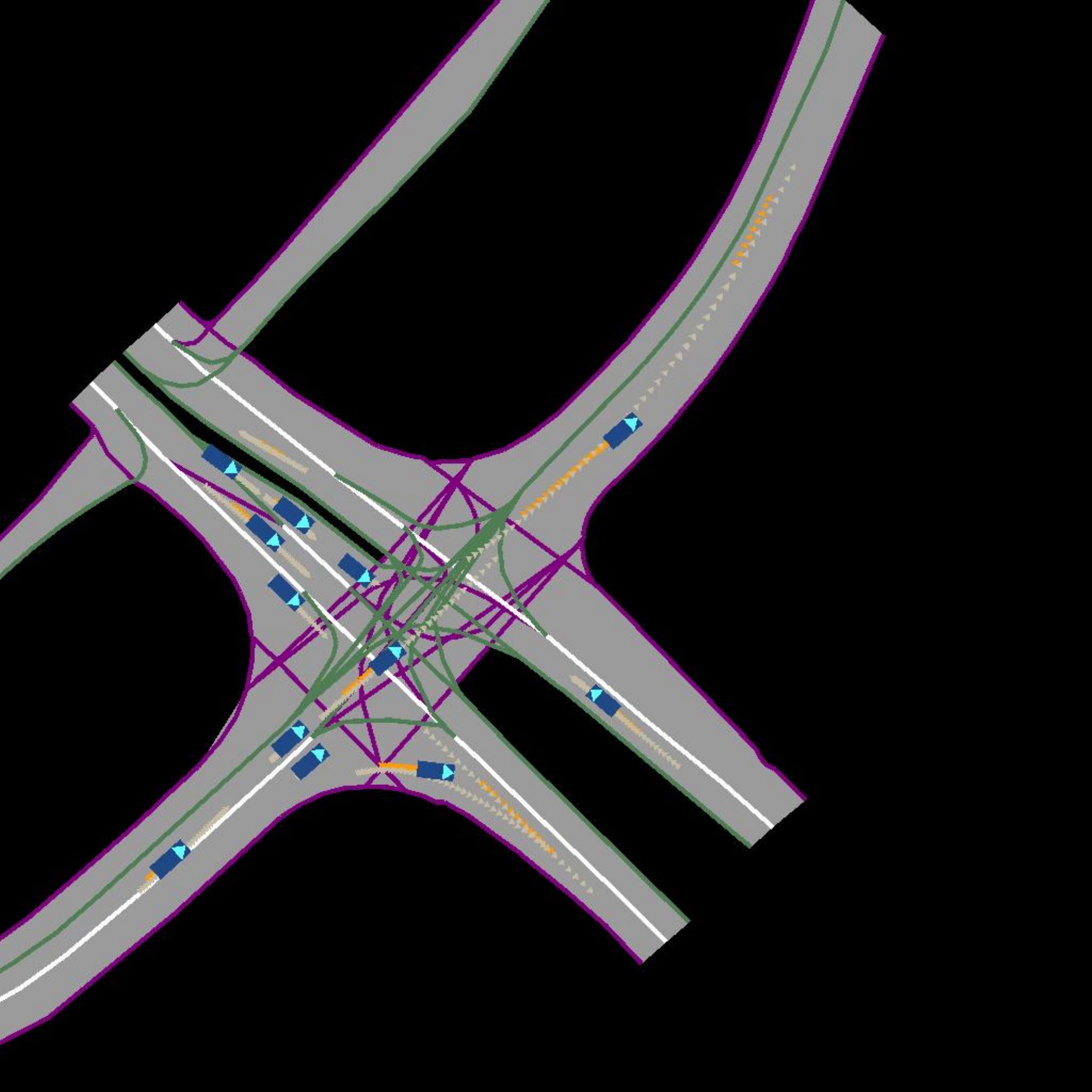}
       \caption*{DPOK}
        \label{}
    \end{subfigure}
    \hfill
    \begin{subfigure}[b]{0.24\textwidth}
        \centering
        \includegraphics[width=1\textwidth]{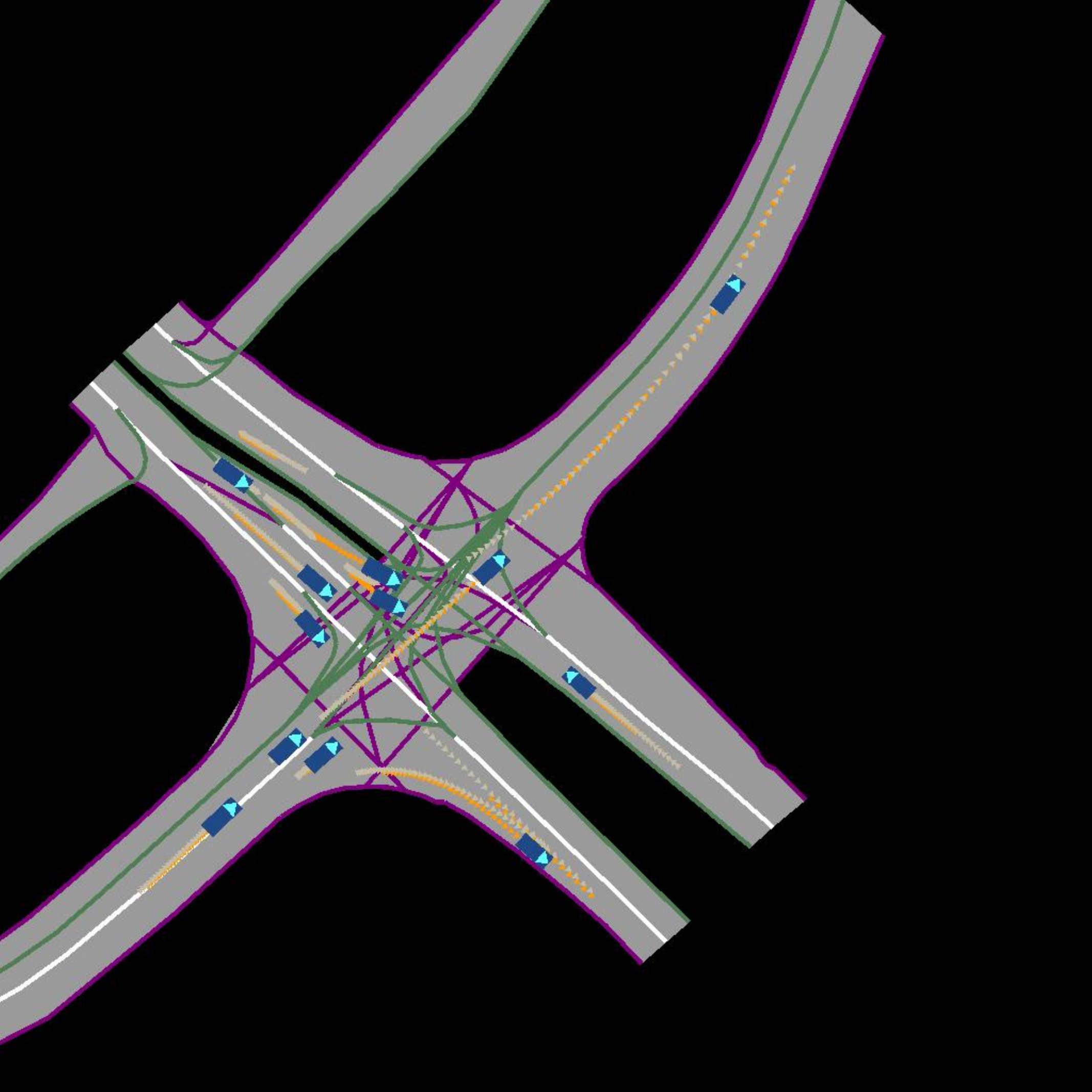}
        \caption*{Rollout-only}
        \label{}
    \end{subfigure}
    \hfill
    \begin{subfigure}[b]{0.24\textwidth}
        \centering
        \includegraphics[width=1\textwidth]{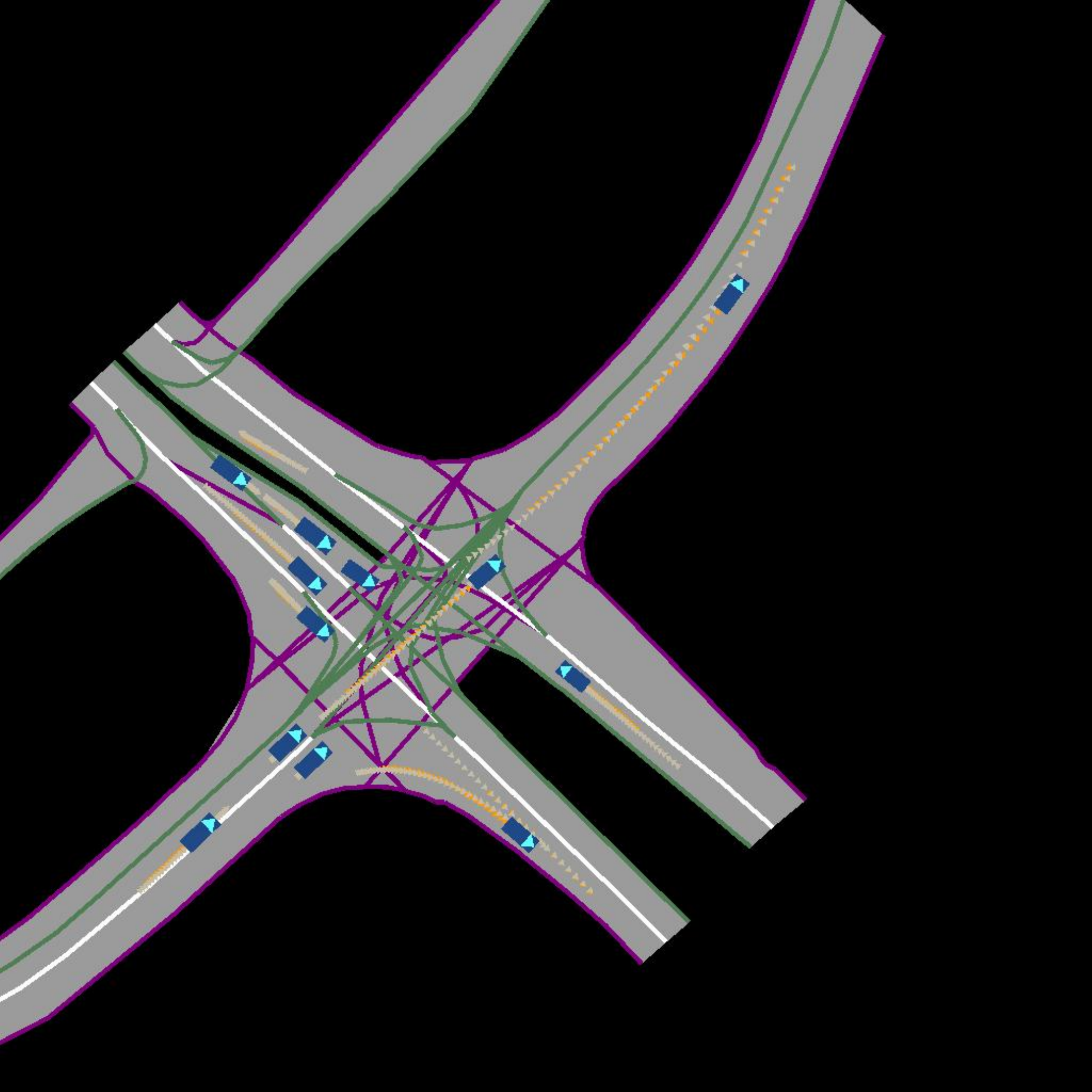
        }
        \caption*{\method{}}
        \label{}
    \end{subfigure}
    \hfill
    \begin{subfigure}[b]{0.24\textwidth}
        \centering
        \includegraphics[width=1\textwidth]{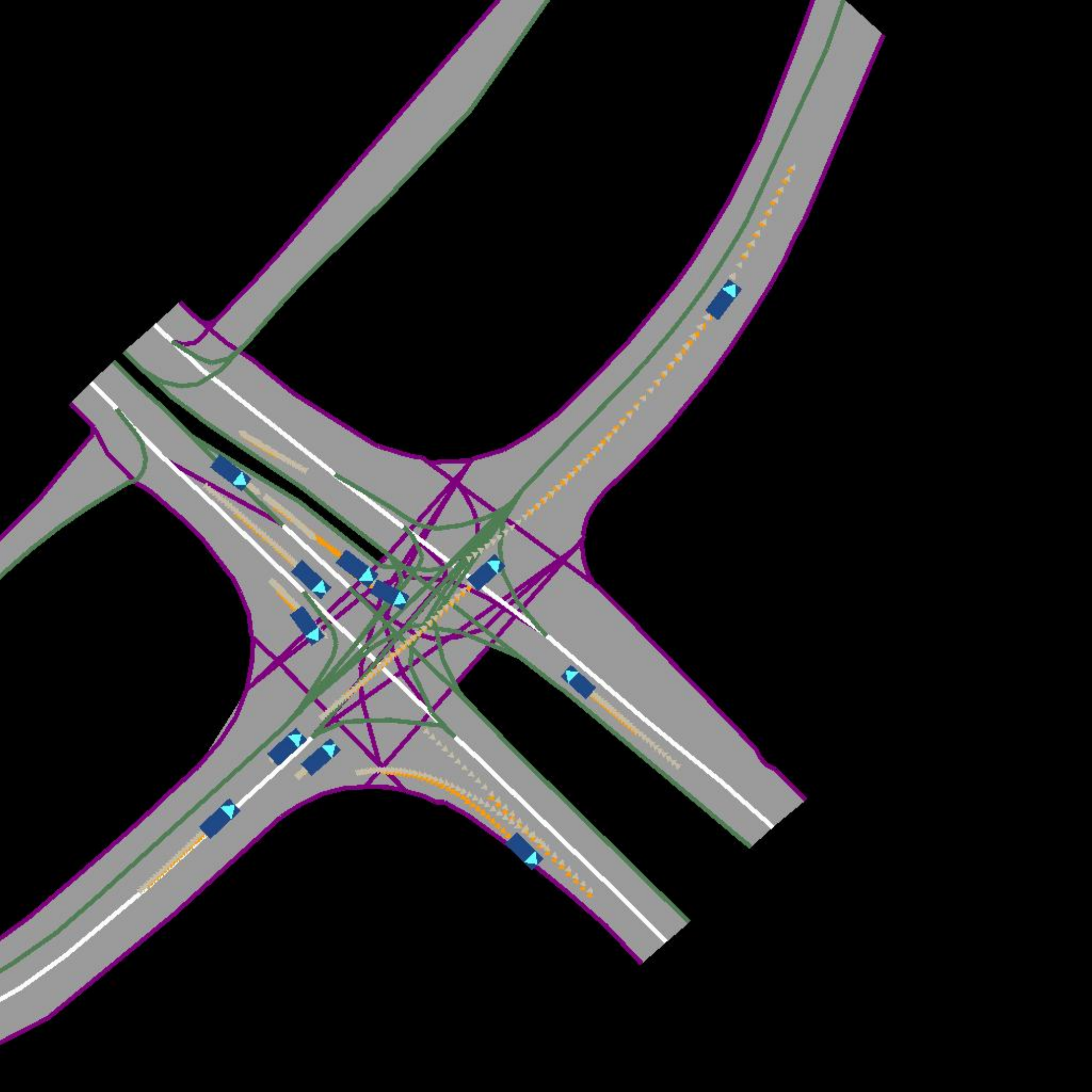}
        \caption*{\method{} (staged)}
        \label{}
    \end{subfigure}
\end{figure*}

\begin{figure*}[!h]
    \centering
    \begin{subfigure}[b]{0.24\textwidth}
        \centering
       \includegraphics[width=1\textwidth]{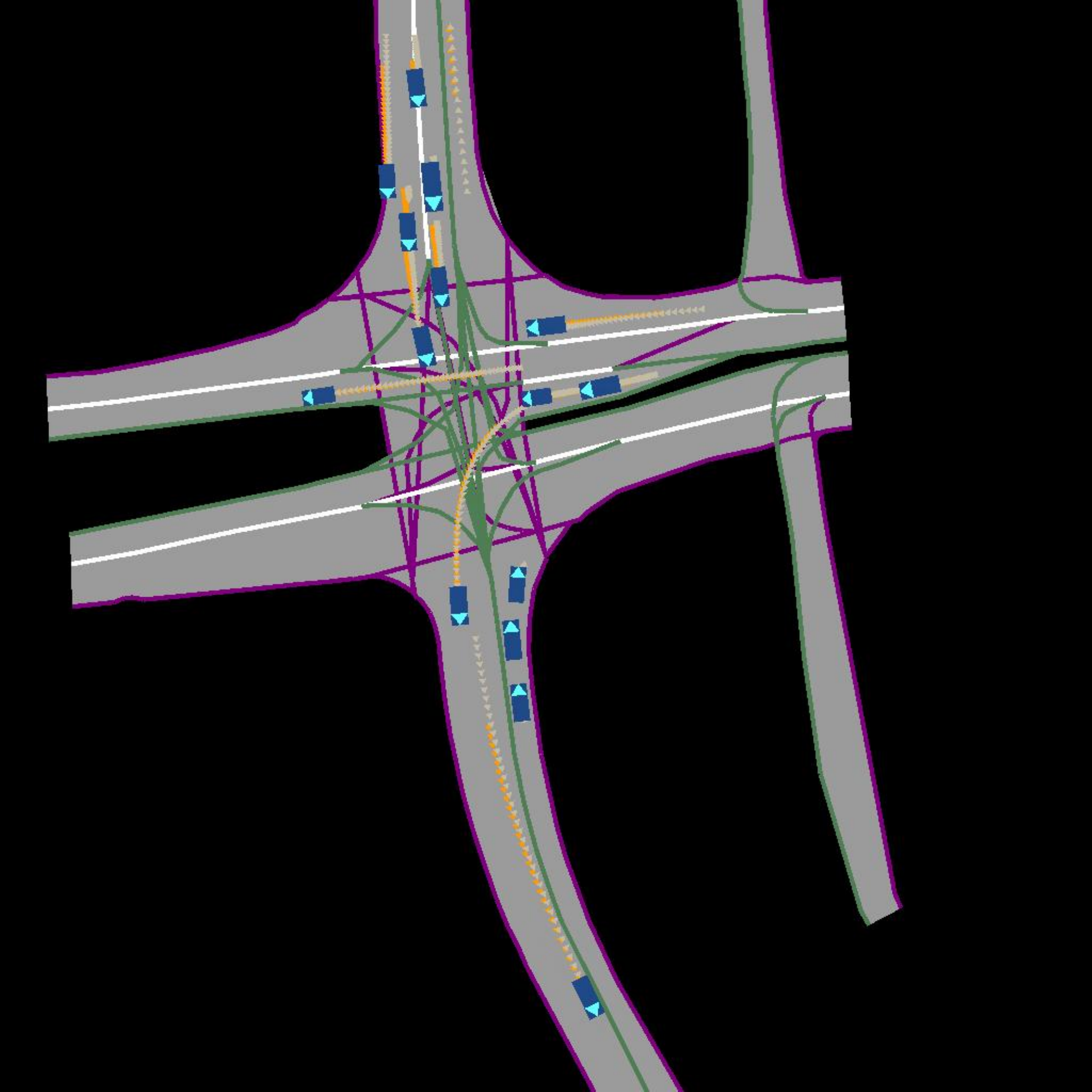}
       \caption*{Standard Diffusion}
        \label{}
    \end{subfigure}
    \hfill
    \begin{subfigure}[b]{0.24\textwidth}
        \centering
        \includegraphics[width=1\textwidth]{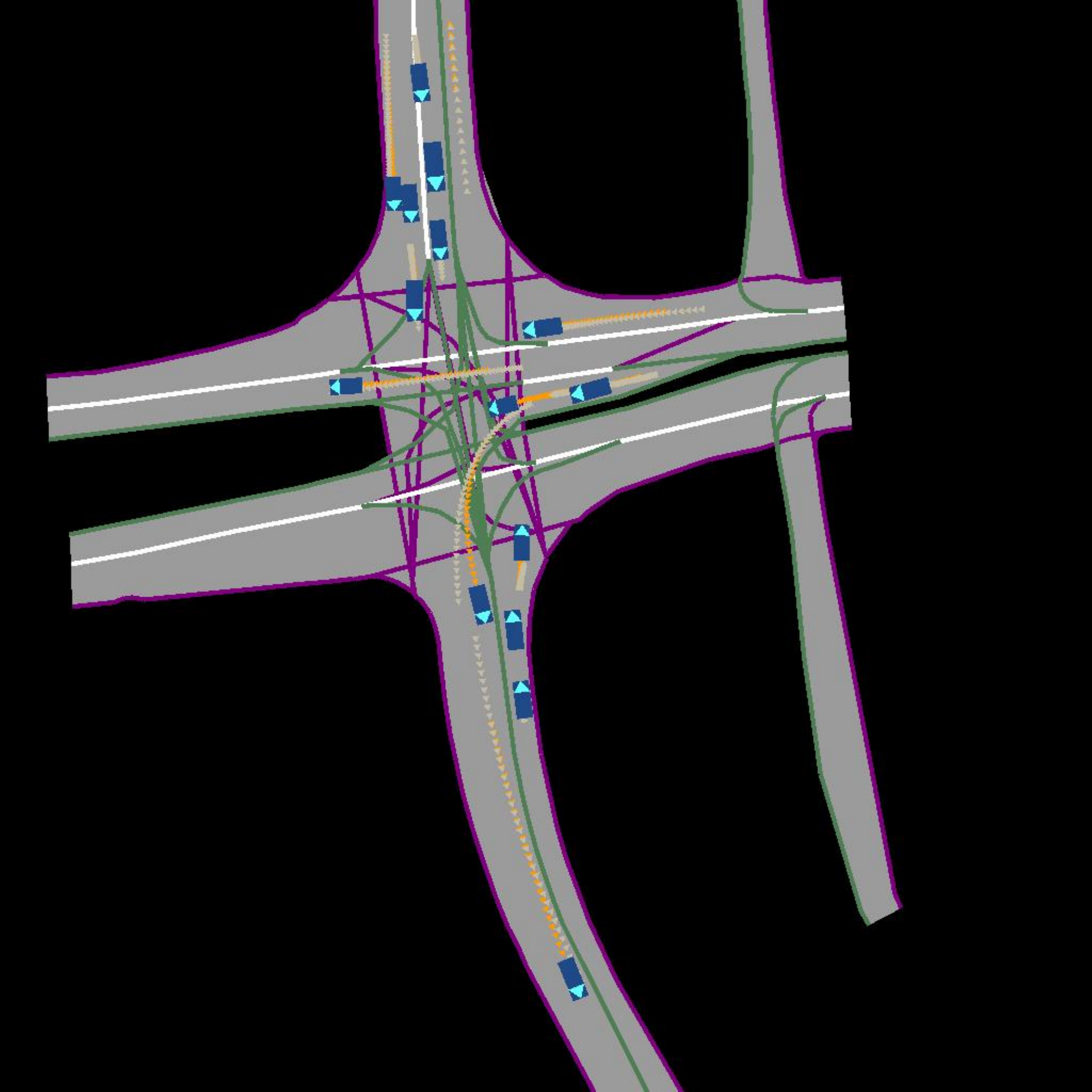}
        \caption*{MPGD w/o projection}
        \label{}
    \end{subfigure}
    \hfill
    \begin{subfigure}[b]{0.24\textwidth}
        \centering
        \includegraphics[width=1\textwidth]{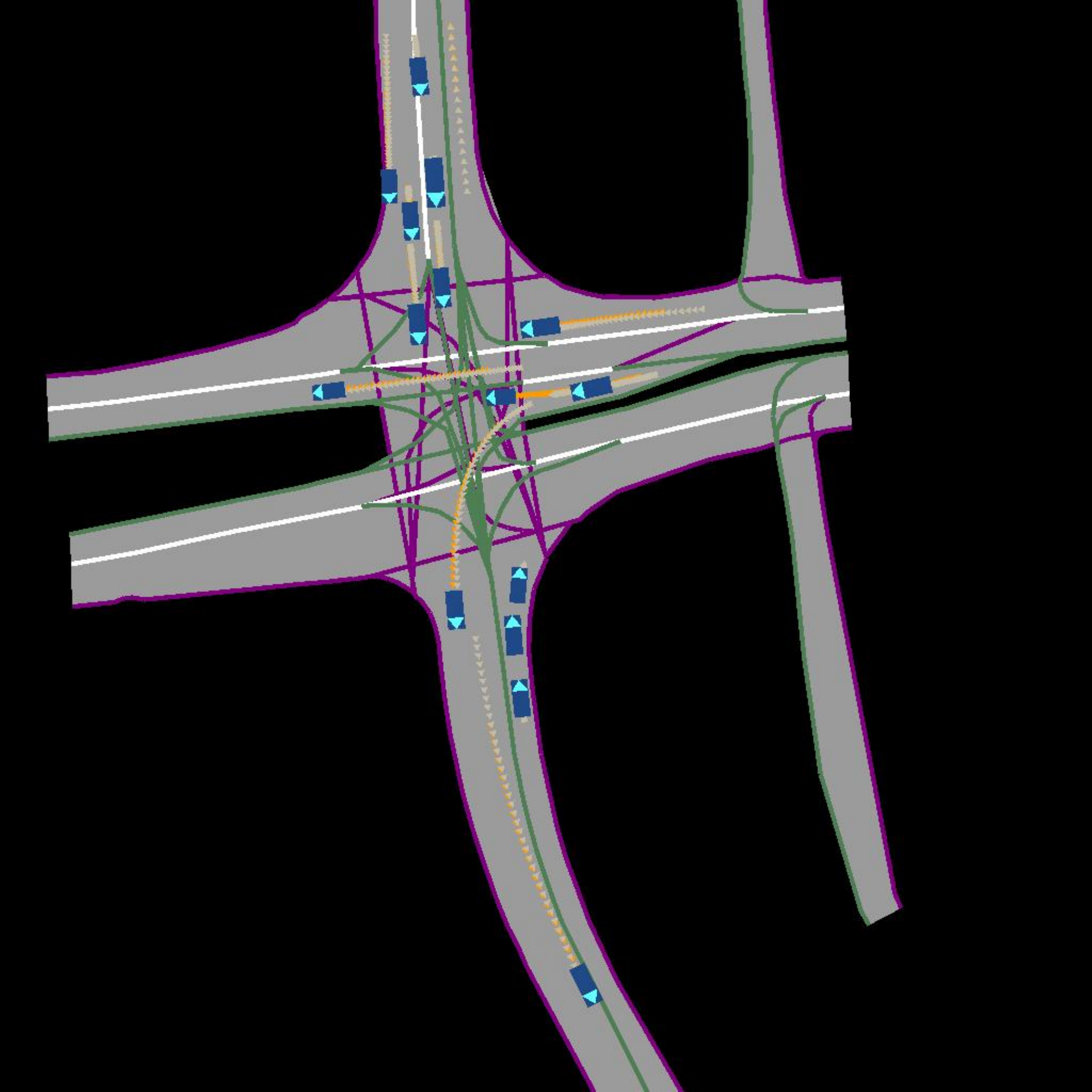}
        \caption*{MBM++}
        \label{}
    \end{subfigure}
    \hfill
    \begin{subfigure}[b]{0.24\textwidth}
        \centering
        \includegraphics[width=1\textwidth]{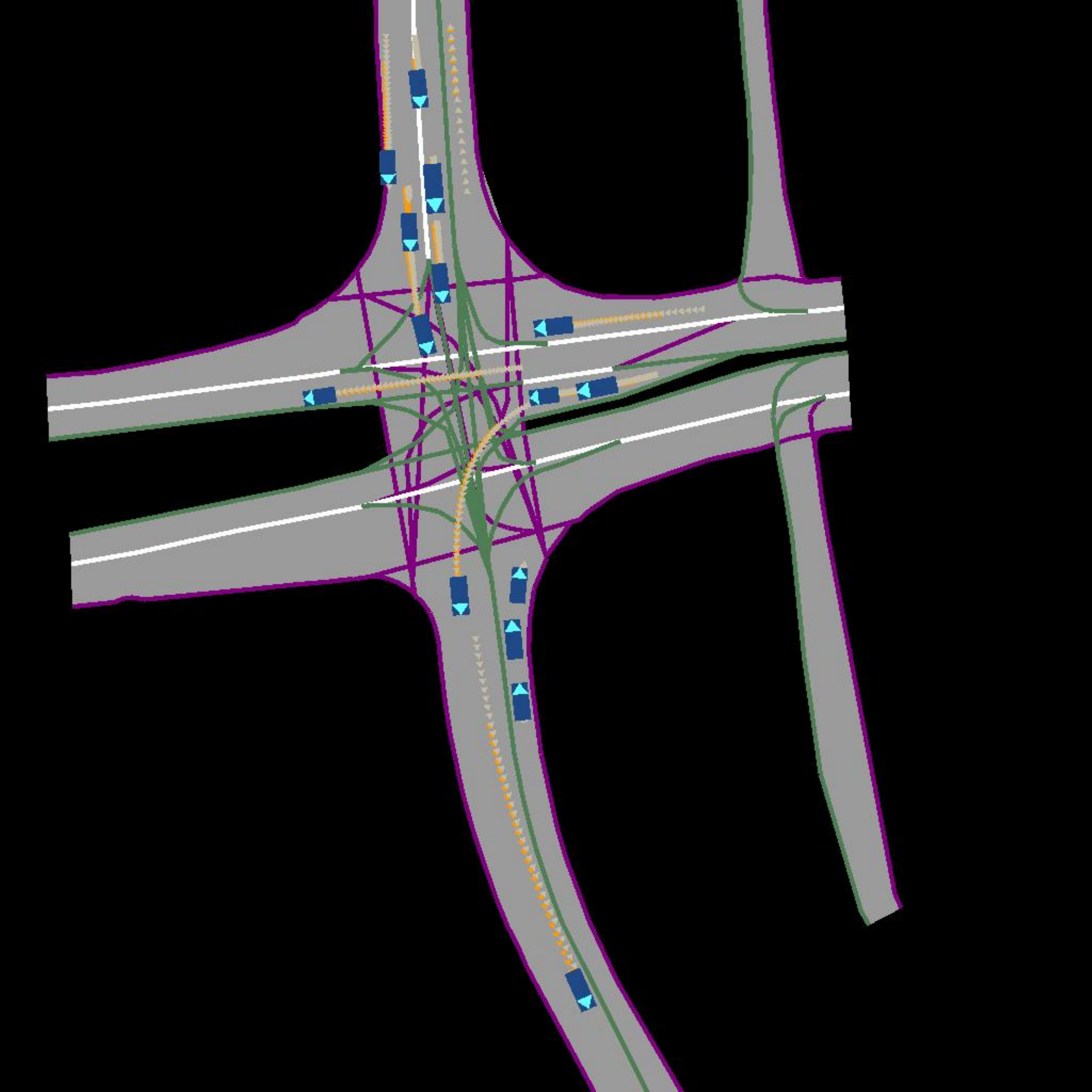}
        \caption*{PIDM}
        \label{}
    \end{subfigure}
    \vspace{0.6em}
    \begin{subfigure}[b]{0.24\textwidth}
        \centering
       \includegraphics[width=1\textwidth]{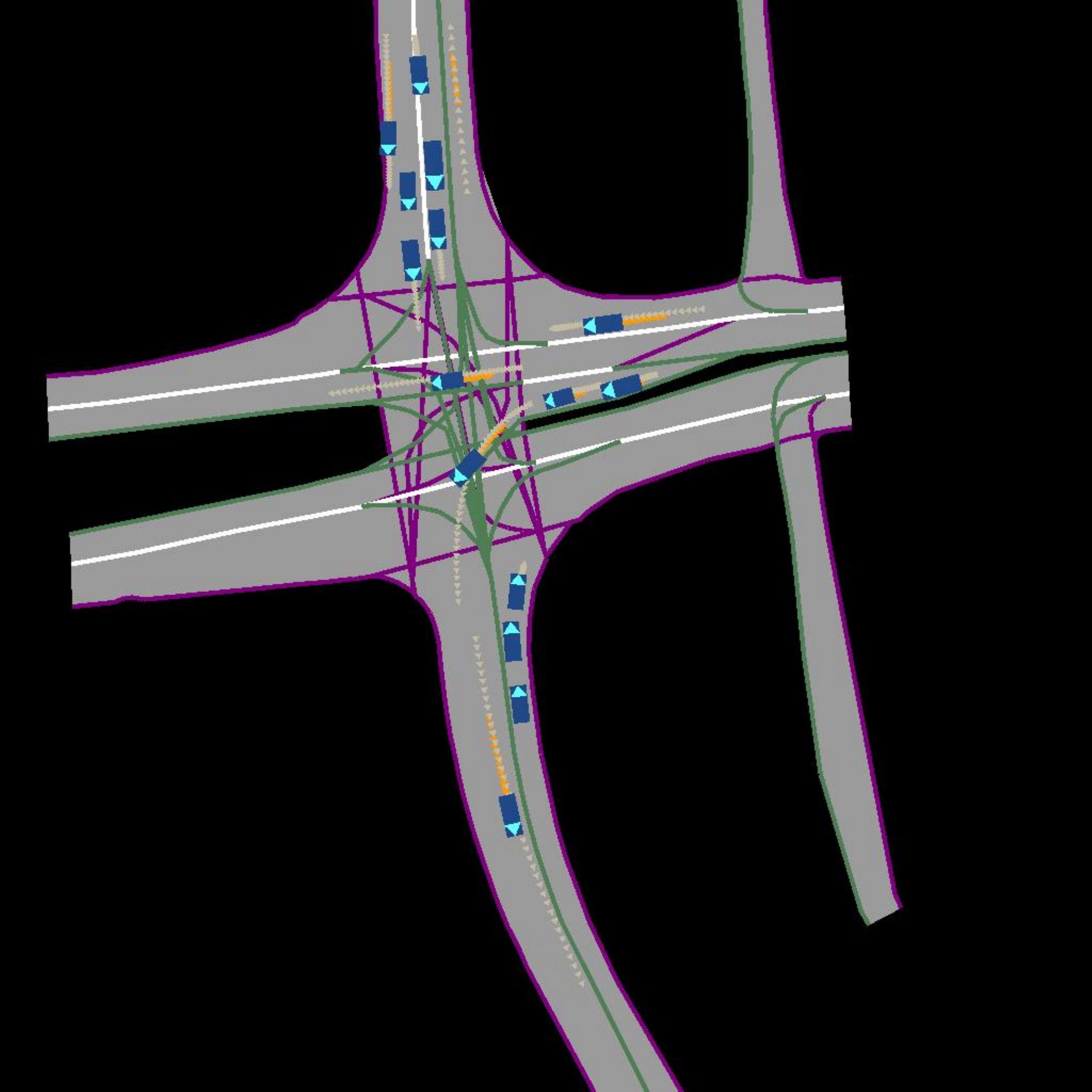}
       \caption*{DPOK}
        \label{}
    \end{subfigure}
    \hfill
    \begin{subfigure}[b]{0.24\textwidth}
        \centering
        \includegraphics[width=1\textwidth]{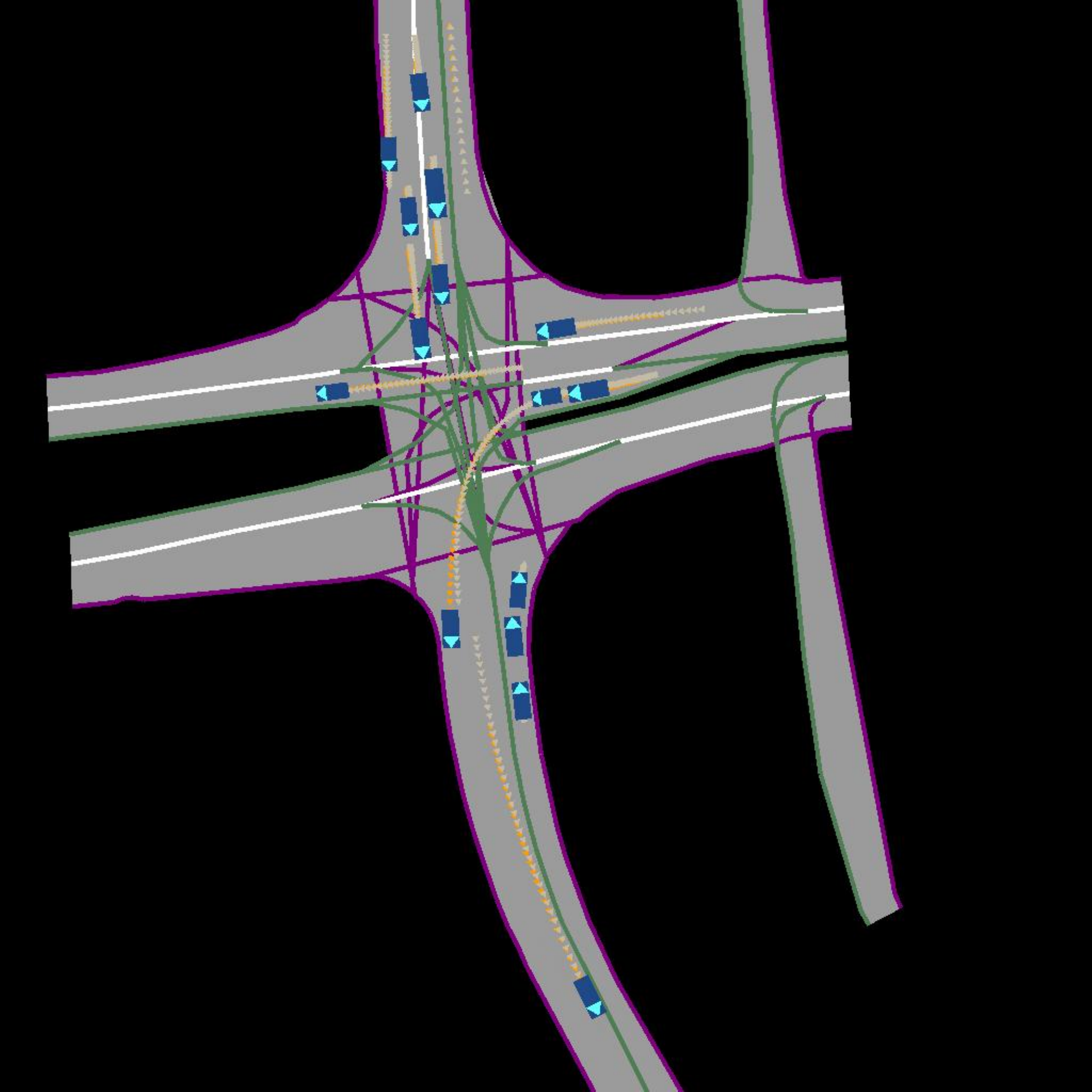}
        \caption*{Rollout-only}
        \label{}
    \end{subfigure}
    \hfill
    \begin{subfigure}[b]{0.24\textwidth}
        \centering
        \includegraphics[width=1\textwidth]{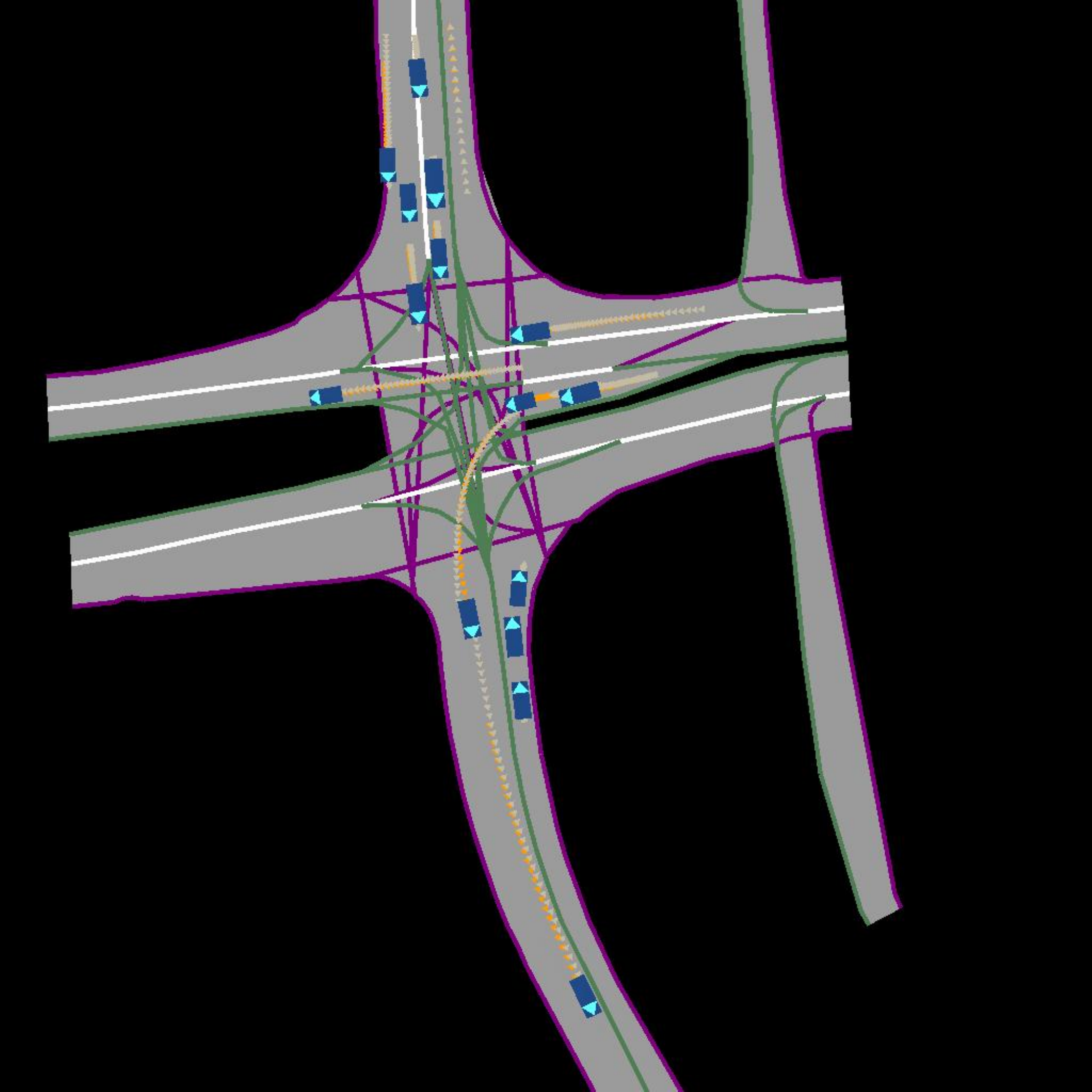
        }
        \caption*{\method{}}
        \label{}
    \end{subfigure}
    \hfill
    \begin{subfigure}[b]{0.24\textwidth}
        \centering
        \includegraphics[width=1\textwidth]{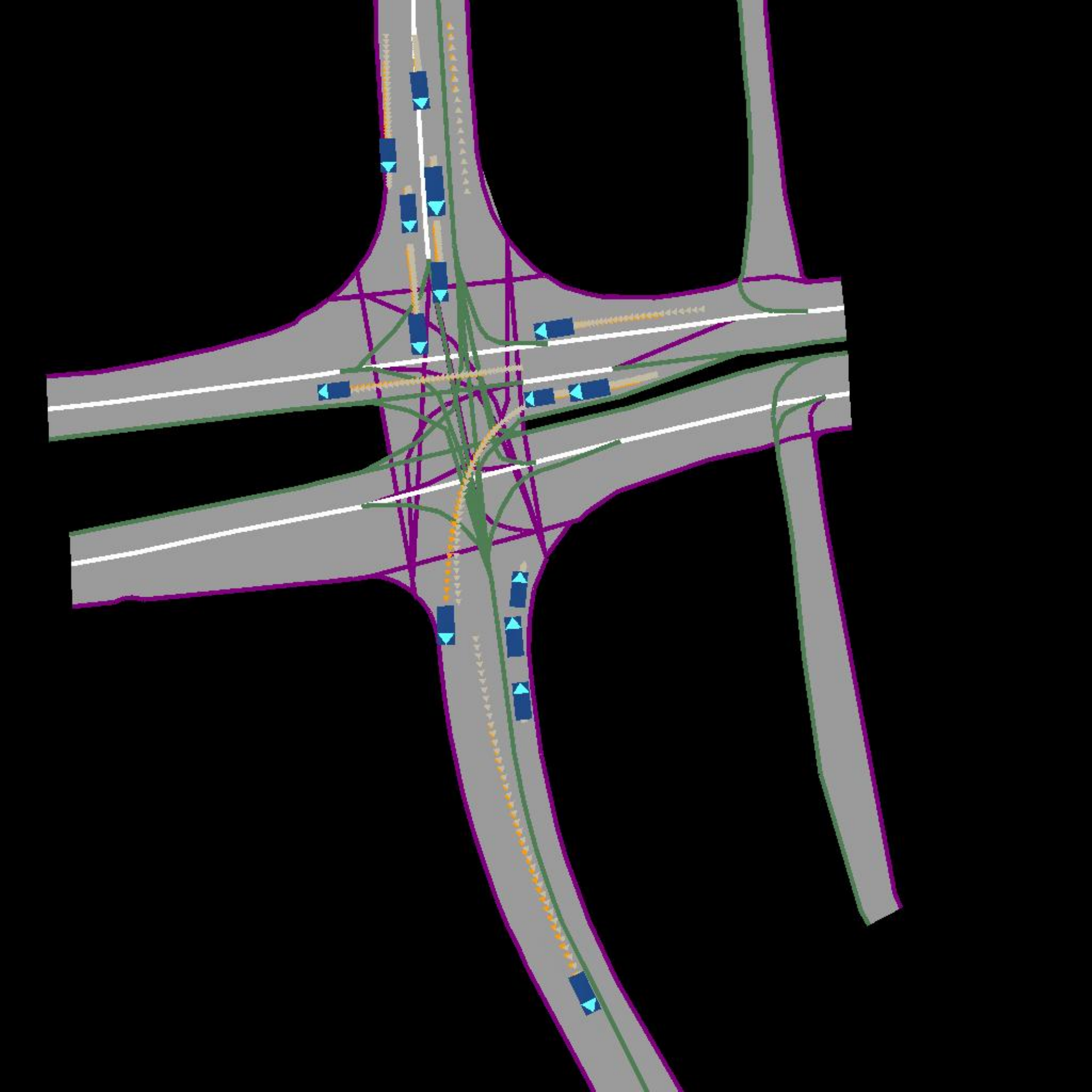}
        \caption*{\method{} (staged)}
        \label{}
    \end{subfigure}
\end{figure*}

\begin{figure*}[!h]
    \centering
    \begin{subfigure}[b]{0.24\textwidth}
        \centering
       \includegraphics[width=1\textwidth]{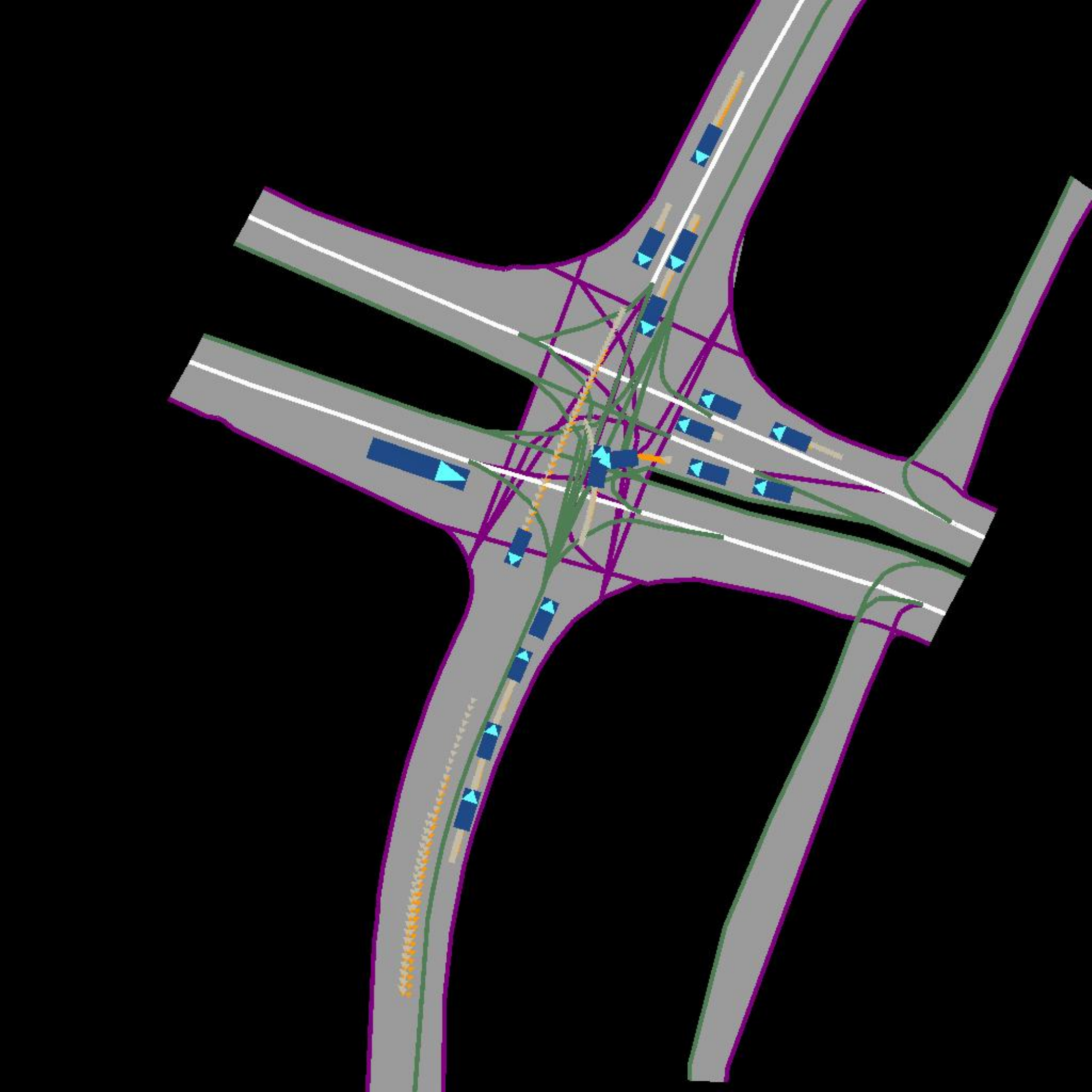}
       \caption*{Standard Diffusion}
        \label{}
    \end{subfigure}
    \hfill
    \begin{subfigure}[b]{0.24\textwidth}
        \centering
        \includegraphics[width=1\textwidth]{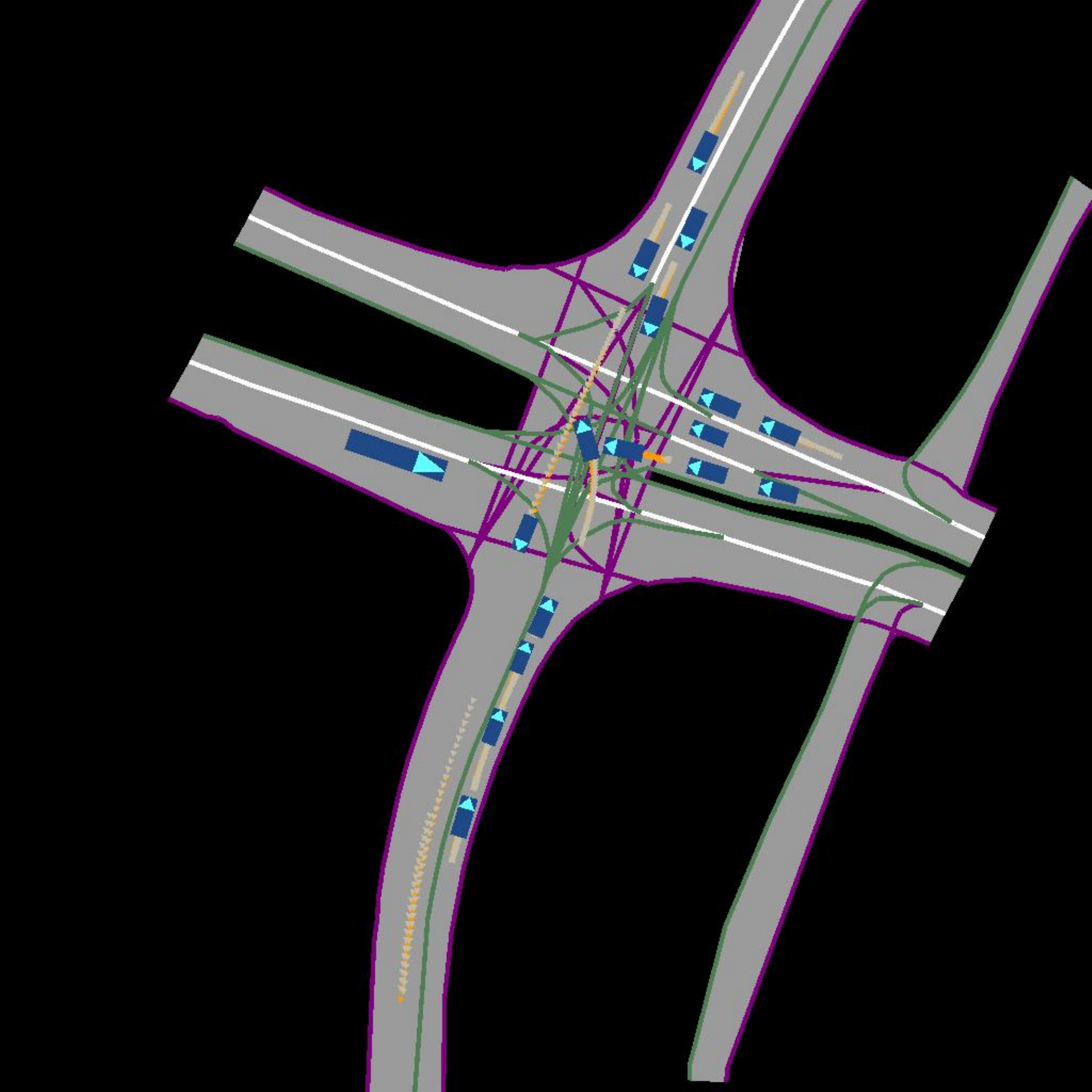}
        \caption*{MPGD w/o projection}
        \label{}
    \end{subfigure}
    \hfill
    \begin{subfigure}[b]{0.24\textwidth}
        \centering
        \includegraphics[width=1\textwidth]{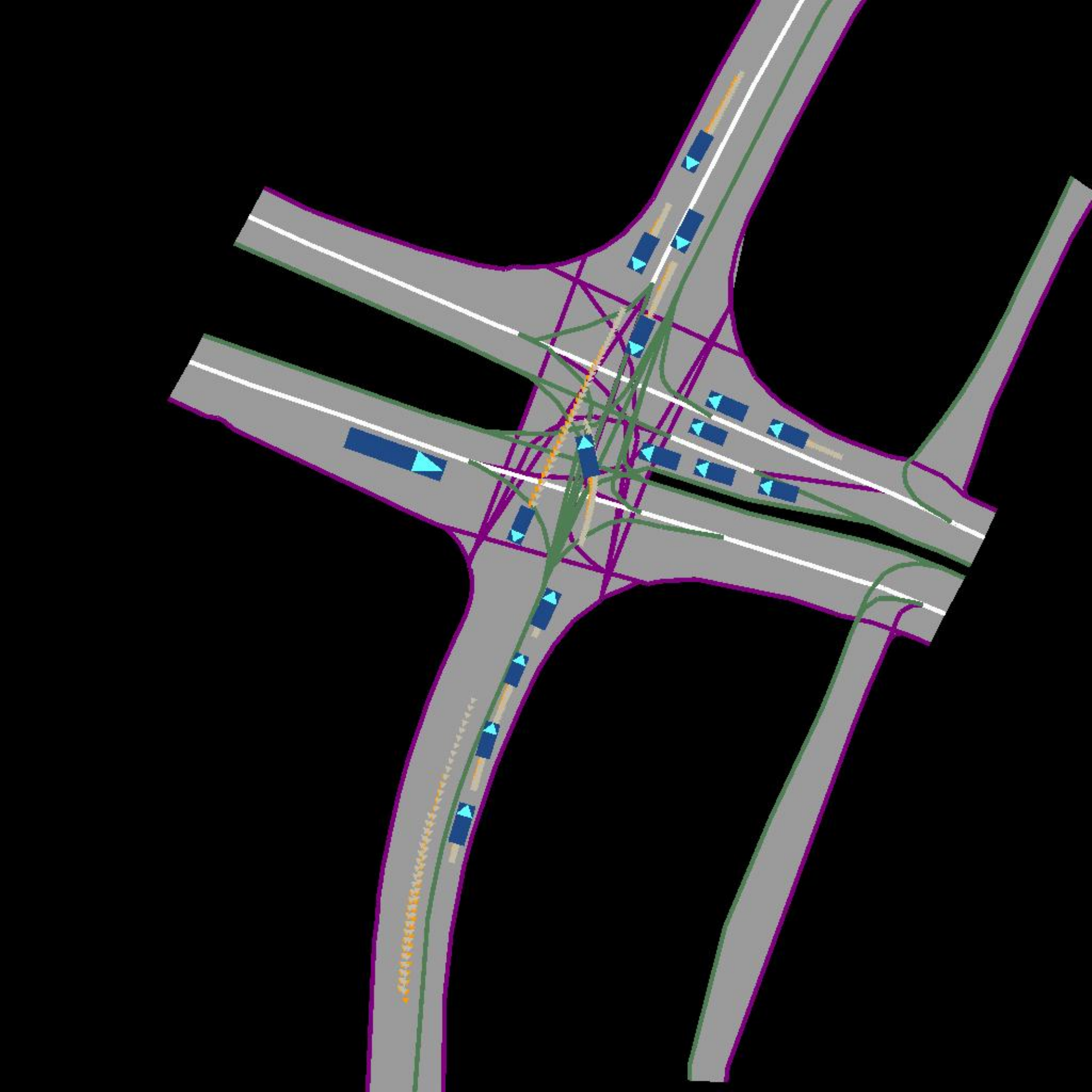}
        \caption*{MBM++}
        \label{}
    \end{subfigure}
    \hfill
    \begin{subfigure}[b]{0.24\textwidth}
        \centering
        \includegraphics[width=1\textwidth]{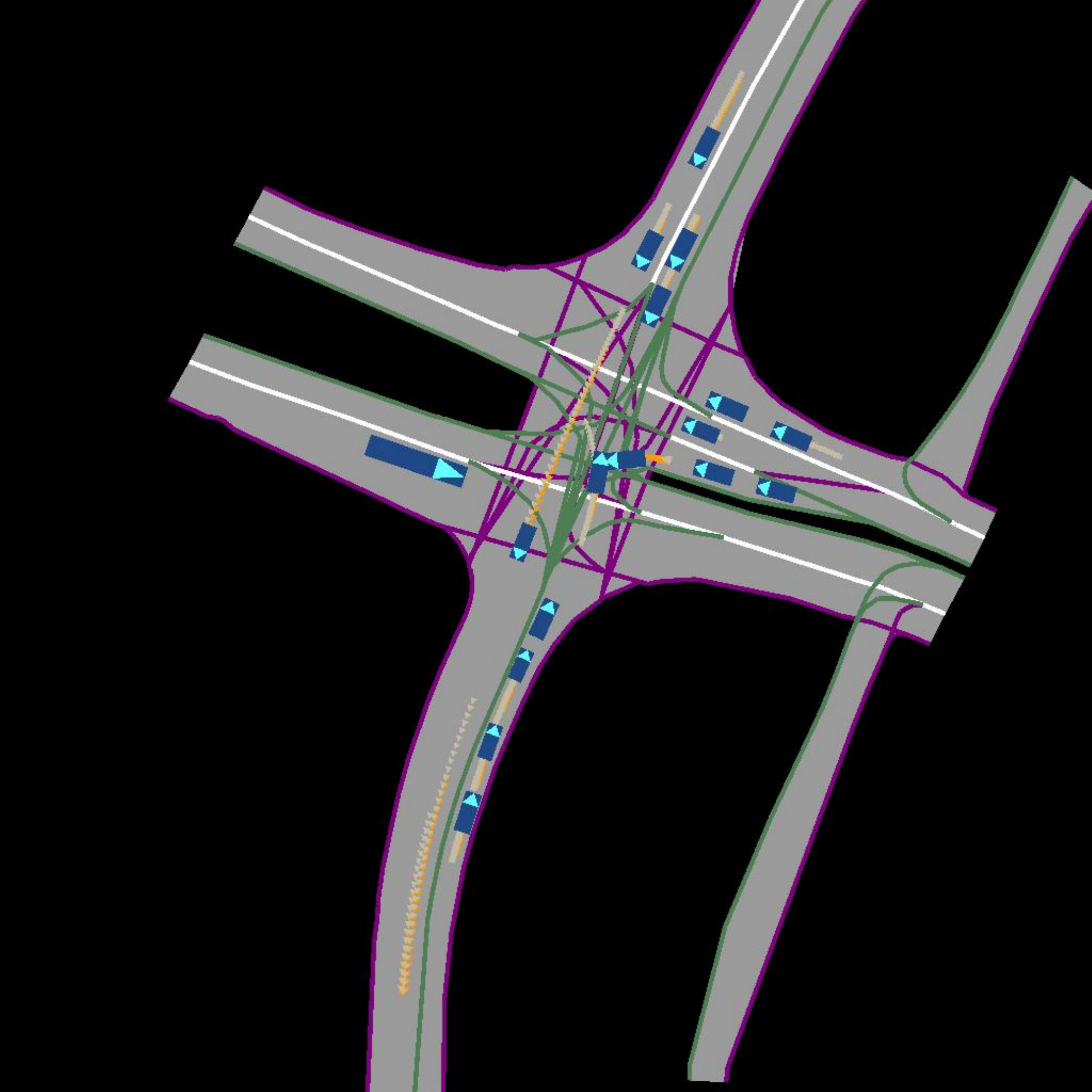}
        \caption*{PIDM}
        \label{}
    \end{subfigure}
    \vspace{0.6em}
    \begin{subfigure}[b]{0.24\textwidth}
        \centering
       \includegraphics[width=1\textwidth]{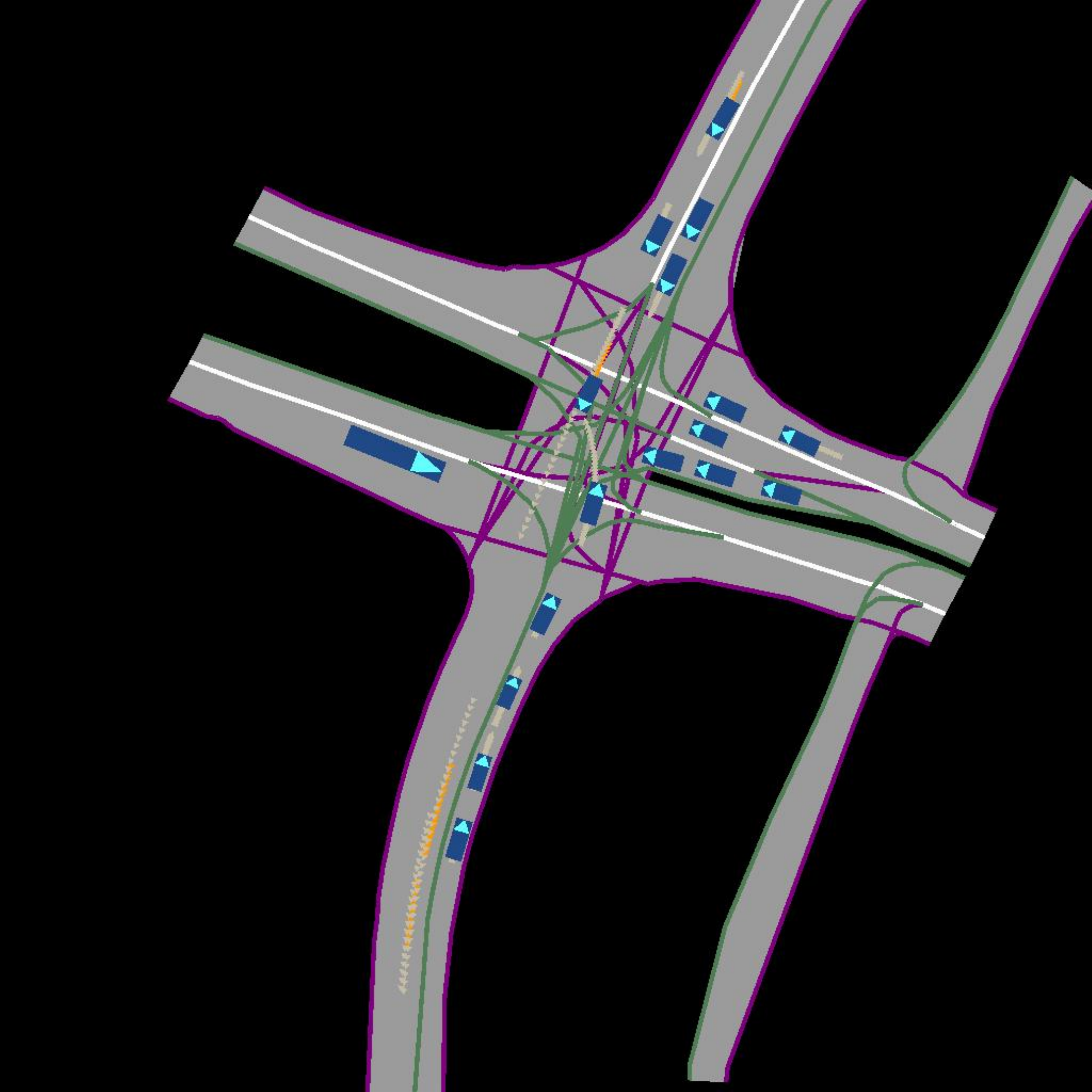}
       \caption*{DPOK}
        \label{}
    \end{subfigure}
    \hfill
    \begin{subfigure}[b]{0.24\textwidth}
        \centering
        \includegraphics[width=1\textwidth]{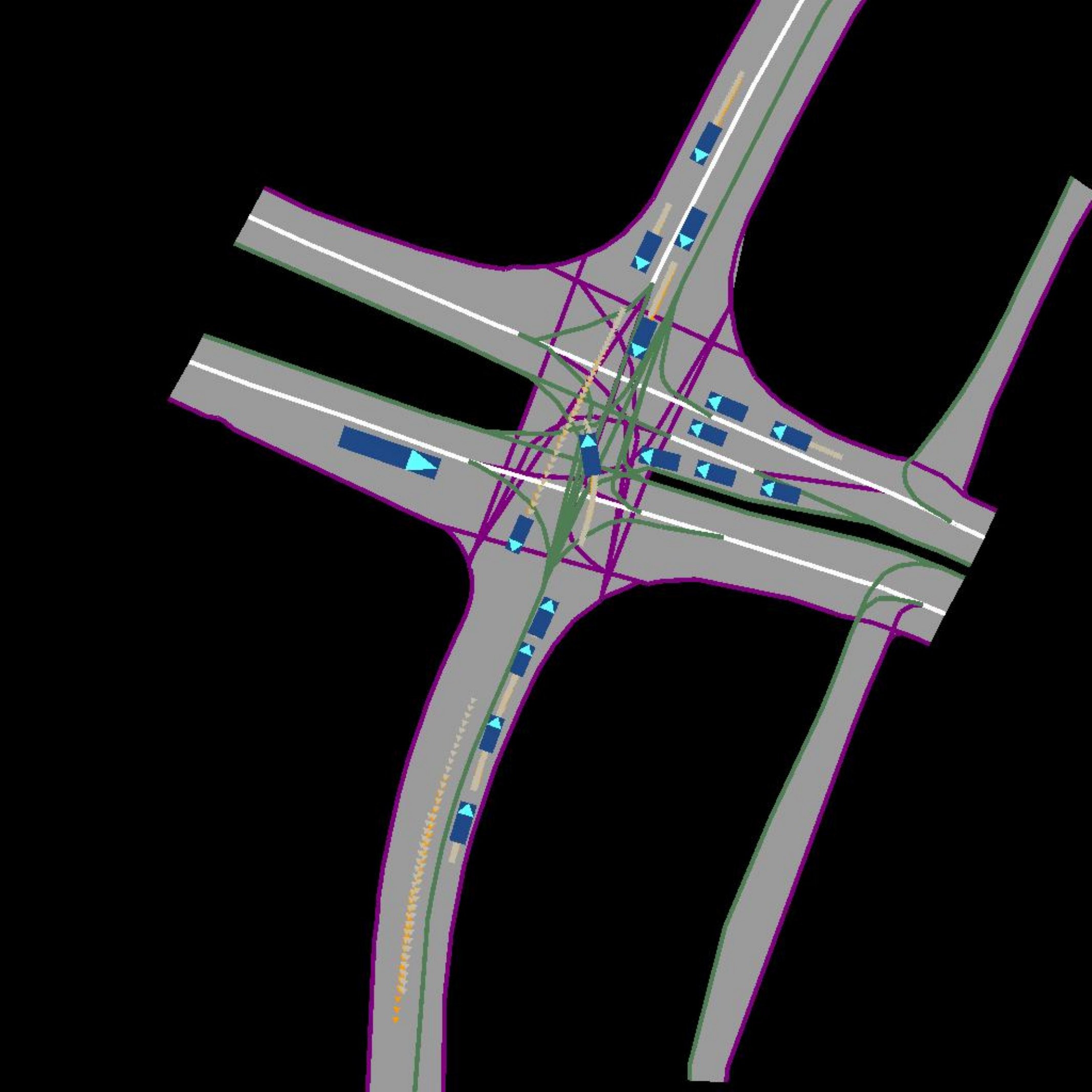}
        \caption*{Rollout-only}
        \label{}
    \end{subfigure}
    \hfill
    \begin{subfigure}[b]{0.24\textwidth}
        \centering
        \includegraphics[width=1\textwidth]{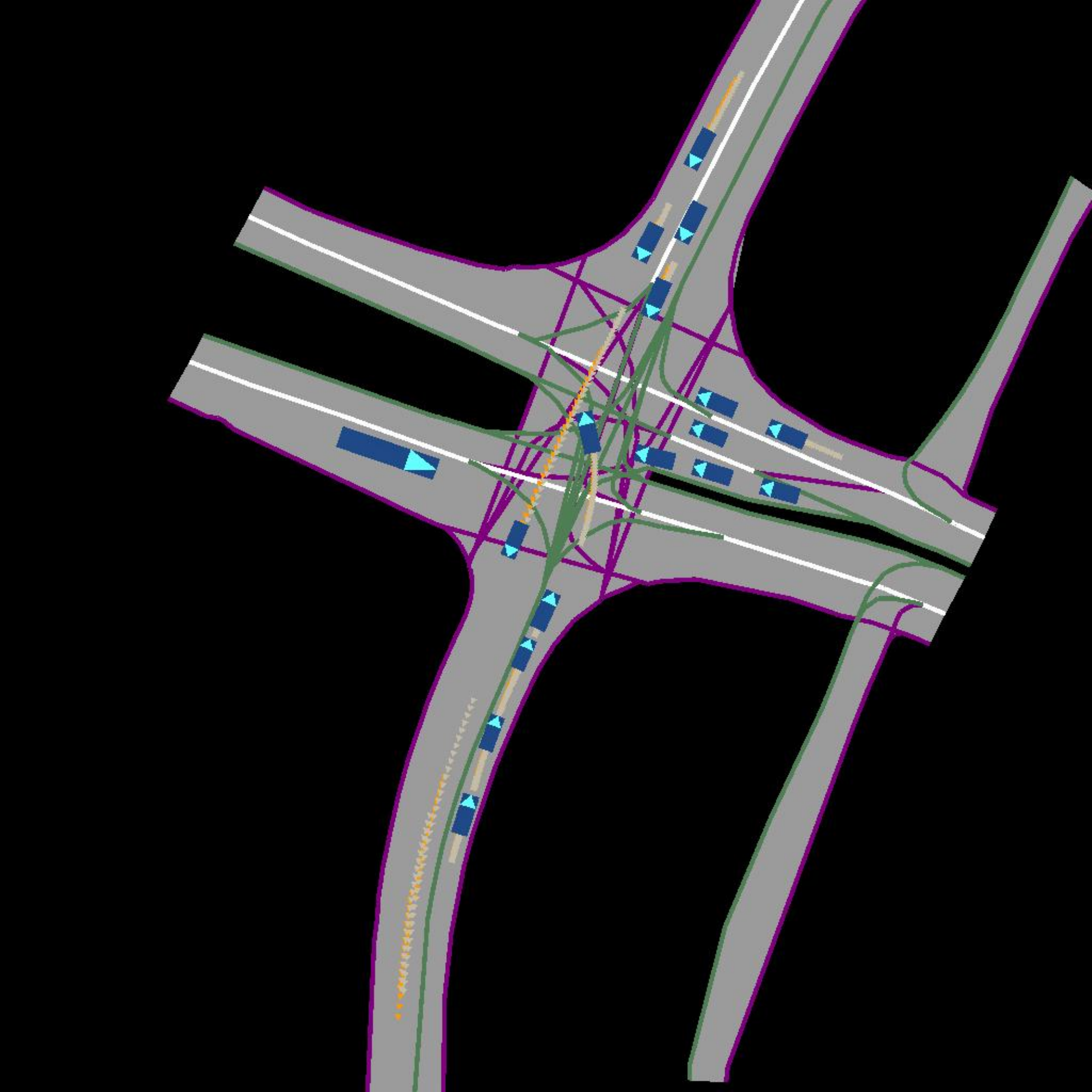
        }
        \caption*{\method{}}
        \label{}
    \end{subfigure}
    \hfill
    \begin{subfigure}[b]{0.24\textwidth}
        \centering
        \includegraphics[width=1\textwidth]{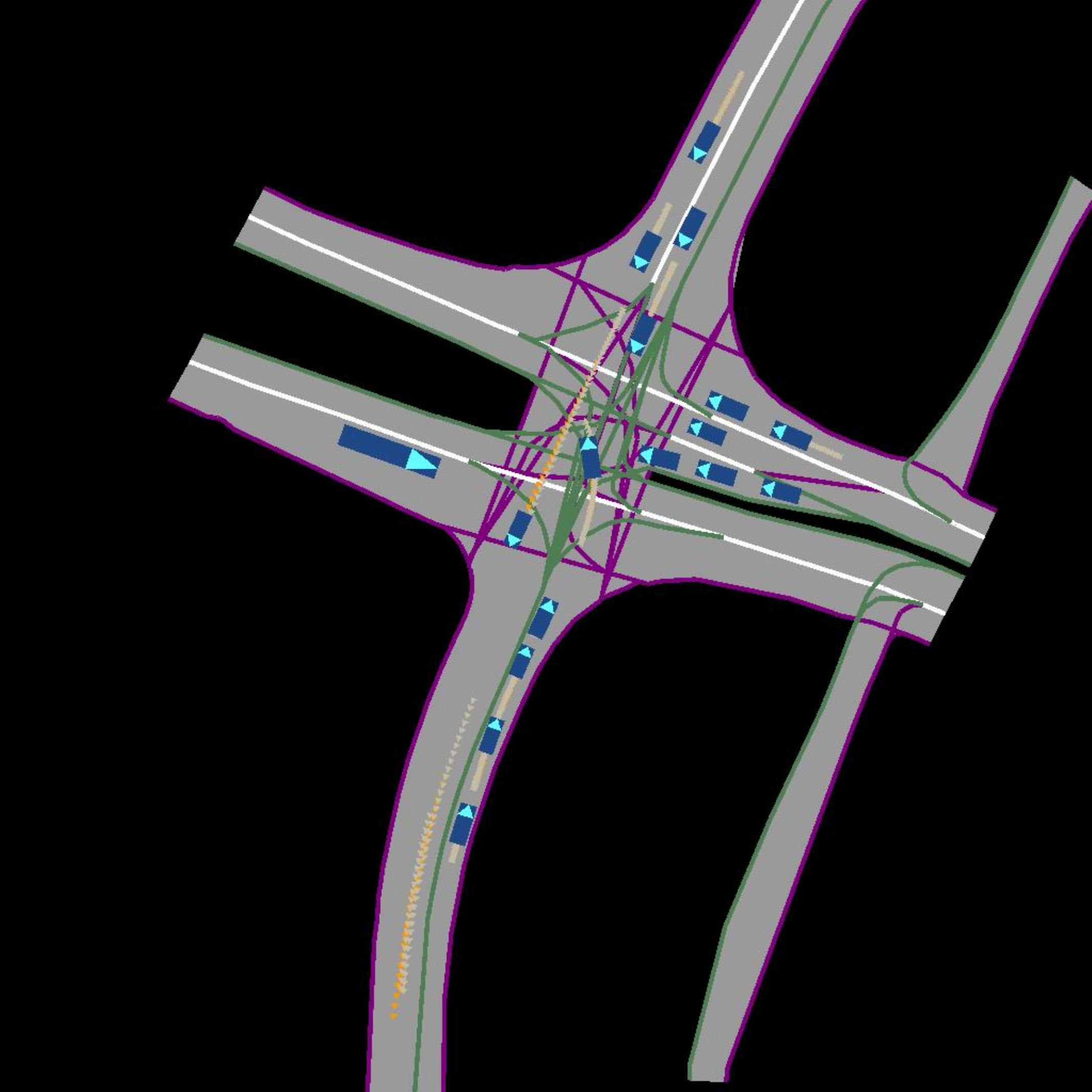}
        \caption*{\method{} (staged)}
        \label{}
    \end{subfigure}
\end{figure*}

\end{document}